\documentclass[11pt]{article}

\usepackage[utf8]{inputenc} % allow utf-8 input
\usepackage[T1]{fontenc}    % use 8-bit T1 fonts
\usepackage[colorlinks=true, linkcolor=blue, citecolor=blue]{hyperref} 
\usepackage{url}            % simple URL typesetting
\usepackage{booktabs}       % professional-quality tables
\usepackage{amsfonts}       % blackboard math symbols
\usepackage{nicefrac}       % compact symbols for 1/2, etc.
\usepackage{microtype}      % microtypography

\usepackage{amsmath, amsthm, amssymb}
\usepackage{graphicx}
\usepackage{algorithm,float}
\usepackage{appendix}
\usepackage{algorithmic}
\usepackage{color}
\usepackage{enumitem}
\usepackage{psfrag}
\usepackage{bm}
\usepackage{subfigure}
\usepackage{fullpage}

\global\long\def\E{\mathbb{E}}

\global\long\def\mA{\mathcal{A}}
\global\long\def\mS{\mathcal{S}}

\newcommand{\Kq}{K}     %kernel operator

\newcommand{\NN}{\mathbb{N}}
\newcommand{\RR}{\mathbb{R}}
\newcommand{\cA}{{\mathcal A}}
\newcommand{\cS}{{\mathcal S}}

\newcommand{\Te}{T_{\epsilon}}
\newcommand{\rank}{\textrm{rank}}
\newcommand{\rowrank}{\textrm{rank}_{\textrm{row}}}
\newcommand{\colrank}{\textrm{rank}_{\textrm{col}}}
\newcommand{\tr}{\textrm{Tr}}

\newcommand{\ceta}{{\sf c}_{{\sf me}}}
\newcommand{\C}{{\sf C}_{{\sf me}}}

\setlist[itemize]{leftmargin=*}
\setlist[enumerate]{leftmargin=*}

\theoremstyle{plain}
\newtheorem{thm}{Theorem}

\theoremstyle{plain}
\newtheorem{defn}[thm]{Definition}

\theoremstyle{plain}
\newtheorem{lem}[thm]{Lemma}

\theoremstyle{plain}

\theoremstyle{plain}
\newtheorem{cor}[thm]{Corollary}

\theoremstyle{definition}

\theoremstyle{plain}
\newtheorem{prop}[thm]{Proposition}

\newtheorem{assumption}{Assumption}

\title{Sample Efficient Reinforcement Learning\\ via Low-Rank Matrix Estimation}

\author{
  Devavrat Shah  \\
  EECS, MIT \\
  \texttt{devavrat@mit.edu} \\
  \and
  Dogyoon Song \\
  EECS, MIT \\
  \texttt{dgsong@mit.edu} \\
  \and
  Zhi Xu \\
  EECS, MIT \\
  \texttt{zhixu@mit.edu} \\
  \and
  Yuzhe Yang \\
  EECS, MIT \\
  \texttt{yuzhe@mit.edu} \\
}

\begin{document}

\date{}
\maketitle

% \vspace{-0.01in}
\begin{abstract}

We consider the question of learning $Q$-function in a sample efficient manner for reinforcement learning with 
continuous state and action spaces under a generative model. If $Q$-function is Lipschitz continuous, 
then the minimal sample complexity for estimating $\epsilon$-optimal $Q$-function is known to scale as 
${\Omega}(\frac{1}{\epsilon^{d_1+d_2 +2}})$ per classical non-parametric learning theory, where $d_1$ and $d_2$ 
denote the dimensions of the state and action spaces respectively. The $Q$-function, when viewed as a kernel, 
induces a Hilbert-Schmidt operator and hence possesses square-summable spectrum. This motivates us to consider 
a parametric class of $Q$-functions parameterized by its ``rank'' $r$, which contains all Lipschitz $Q$-functions 
as $r \to \infty$. 
As our key contribution, we develop a simple, iterative learning algorithm that finds $\epsilon$-optimal $Q$-function with 
sample complexity of $\widetilde{O}(\frac{1}{\epsilon^{\max(d_1, d_2)+2}})$ when the optimal $Q$-function has low rank $r$ 
and the discounting factor $\gamma$ is below a certain threshold. %, with $C(r)$ being a constant dependent on $r$. 
Thus, this provides an exponential improvement in sample complexity. To enable our result, we develop a novel Matrix Estimation algorithm 
that faithfully estimates an unknown low-rank matrix in the $\ell_\infty$ sense even in the presence of arbitrary bounded noise, 
which might be of interest in its own right. 
Empirical results on several stochastic control tasks confirm the efficacy of our ``low-rank'' algorithms.

%We study sample efficient reinforcement learning algorithm for problems with both continuous state and action spaces under a generative model. In particular, we focus on learning $\epsilon$-optimal $Q$-function under the $\ell_\infty$ metric. While classical non-parametric learning theory would suggest a $\Omega(\frac{1}{\epsilon^{d_1+d_2}})$ lower bound on the sample complexity, where $d_1$ and $d_2$ are respectively the dimensions of the state and the action spaces, it is conceivable that reinforcement learning tasks often possess certain problem ``structures'' which should be exploitable to devise efficient algorithms. In this paper, we formalize this intuition by introducing a low-rank framework.
%We provide iterative, non-parametric learning algorithms with an appropriately designed low-rank Matrix Estimation scheme at each iteration to leverage the structural information in the $Q$-function. As our key contribution, we show that the algorithm probably achieves a sample complexity of $\widetilde{O}(\frac{1}{\epsilon^{\max(d_1, d_2)+2}})$ for problems with discounting factor below a certain threshold. This significantly outperforms the classical lower bound in light of the curse of dimensionality, demonstrating the value of our framework. Empirical results on several stochastic control tasks confirm the efficiency of our ``low-rank'' algorithms.

\end{abstract}

% \tableofcontents

%\vspace{-0.06in}
\section{Introduction}
Reinforcement Learning (RL) has emerged as a promising technique for a variety of decision-making tasks, 
highlighted by impressive successes such as solving Atari games~\cite{mnih2013playing,mnih2015human} 
and Go~\cite{silver2016mastering,silver2017mastering}. However, generic RL methods suffer from 
``curse-of-dimensionality''. Specifically, the classical minimax theory \cite{stone1982optimal, tsybakov2008introduction} suggests that for $\epsilon > 0$, 
we need ${\Omega}(\frac{1}{\epsilon^{d_1+d_2+2}})$ samples to learn an $\epsilon$-optimal state-action value, i.e., $Q$-function 
% in the $L^2$-sense and also in the $L^{\infty}$ sense, 
when the (continuous) state and action spaces have dimensions $d_1$ and $d_2$ respectively 
and the $Q$-function is Lipschitz continuous over them. 
On the other hand, as exemplified by empirical successes, practical RL tasks seem to 
possess low-dimensional latent structures. % that appear to be exploited by empirically.
Indeed, feature-based methods precisely aim to explain such phenomenon 
%lower-dimensional representation and resulting ability to design data-efficient RL method 
by positing that either the transition kernel~\cite{yang2019sample,yang2019reinforcement} or the value function~\cite{tsitsiklis1997analysis,melo2008analysis,parr2008analysis,maei2010toward,zou2019finite}
is linear in low-dimensional features associated with states and actions. That is, not only the state and action 
spaces have low-dimensional representation, the value function is linear. % in such representation. 
While these may be true, the algorithm may not have the \emph{knowledge} of such feature map beforehand; 
and relying on the hope of a neural network to find it might be too much to ask.

Motivated by this, the primary goal in this work is to learn the optimal $Q$-function in a data-efficient
manner if it has a lower-dimensional representation, {\em without} the need of any additional information 
such as knowledge of features. Therefore, we ask the following key question in this paper: %``{\em what is a representation of 
% $Q$-function that is {universal} and allows for designing data-efficient learning algorithm when 
% $Q$-function is lower-dimensional with respect to such a representation?}''
\begin{quote}
    ``{\em Is there a universal representation of $Q$-function that allows for designing a data-efficient 
    learning algorithm if the $Q$-function has a low-dimensional structure?}''
\end{quote}

\subsection{Our Contribution}
As the main contribution of this work, we answer this question in
the affirmative by developing a novel spectral representation of the $Q$-function for a generic RL task, 
and provide a data-efficient method to learn a near-optimal $Q$-function when it is lower-dimensional. 

\begin{table}[t]
\setlength{\tabcolsep}{7pt}
% \small

\caption{Informal summary of sample complexity results for three different state/action space configurations: our results, a few selected from literature, and the lower bounds. For ours, see Theorem \ref{thm:generic} \& Appendix \ref{sec:rankrappendix}.}
\label{tab:one}

\begin{center}
\begin{tabular}{l c c c c}
% \begin{tabular}{l l l l l}
\toprule[1.3pt]
 Setting & \multicolumn{1}{c}{Our Results} & \multicolumn{2}{c}{Selected from Literature} & \multicolumn{1}{c}{Lower Bound}\\
 \midrule%\hline
 Cont. $\mS$ \& Cont. $\mA$   & $\tilde{O}\big(\frac{1}{\epsilon^{\max\{d_1, d_2\}+2}}\big)$  & \multicolumn{2}{c}{N/A} & ${\Omega}\big(\frac{1}{\epsilon^{d_1+d_2+2}}\big)$~\cite{tsybakov2008introduction}\\ [0.2ex]
%  \hline
 Cont. $\mS$ \& Finite $\mA$   & $\tilde{O}\big(\frac{1}{\epsilon^{d_1+2}}\big)$
 &$\tilde{O}\big(\frac{1}{\epsilon^{d_1+3}}\big)$~\cite{shah2018q} &    $\tilde{O}\big(\frac{1}{\epsilon^{d_1+2}}\big)~\cite{yang2019theoretical}$& $\tilde{\Omega}\big(\frac{1}{\epsilon^{d_1+2}}\big)$~\cite{shah2018q}\\ [0.2ex]
%  \hline
 Finite $\mS$ \& Finite $\mA$ &   $\tilde{O}\big(\frac{\max(|\mS|,|\mA|)}{\epsilon^{2}}\big)$   & $\tilde{O}\big(\frac{|\mS||\mA|}{(1-\gamma)^3\epsilon^2}\big)$~\cite{sidford2018near}   & $\tilde{O}\big(\frac{|\mS||\mA|}{(1-\gamma)^4\epsilon^2}\big)$~\cite{sidford2018variance} &$\tilde{\Omega}\big(\frac{|\mS||\mA|}{(1-\gamma)^3\epsilon^2}\big)$~\cite{azar2013minimax} \\
\bottomrule[1.3pt]
\end{tabular}
\vspace{-0.2in}
\end{center}
\end{table}

% giving it a try to see how it lookes if we take transpose
% \begin{table}[t]
% \setlength{\tabcolsep}{5.5pt}
% \small
% \begin{center}
% \caption{\small Informal summary of sample complexity results for three different configurations of state/action spaces, including ours, a few selected from literature, and the lower bounds. For ours, see Thm. \ref{thm:generic} \& Appx. \ref{sec:rankrappendix}.}
% \vspace{-4pt}
% \label{tab:one}
% %\vspace{-7pt}
% \begin{tabular}{c|c|c|c}
% \toprule[1.3pt]
%                 & Cont. $\mS$ \& Cont. $\mA$    & Cont. $\mS$ \& Finite $\mA$    & Finite $\mS$ \&  Finite $\mA$\\
%  \midrule%\hline
%     Ours        &   $\tilde{O}\big(\frac{1}{\epsilon^{\max\{d_1, d_2\}+2}}\big)$    &   $\tilde{O}\big(\frac{1}{\epsilon^{d_1+2}}\big)$ &   $\tilde{O}\big(\frac{\max(|\mS|,|\mA|)}{\epsilon^{2}}\big)$\\
%     Literature  &   N/A     &   $\tilde{O}\big(\frac{1}{\epsilon^{d_1+3}}\big)$~\cite{shah2018q}    &  $\tilde{O}\big(\frac{|\mS||\mA|}{(1-\gamma)^3\epsilon^2}\big)$~\cite{sidford2018near} \\
%                 &           &   $\tilde{O}\big(\frac{1}{\epsilon^{d_1+2}}\big)~\cite{yang2019theoretical}$  &  $\tilde{O}\big(\frac{|\mS||\mA|}{(1-\gamma)^4\epsilon^2}\big)$~\cite{sidford2018variance} \\
%     Lower bound &   $\tilde{\Omega}\big(\frac{1}{\epsilon^{d_1+d_2+2}}\big)$~\cite{shah2018q}    &  $\tilde{\Omega}\big(\frac{1}{\epsilon^{d_1+2}}\big)$~\cite{shah2018q}    &  $\tilde{\Omega}\big(\frac{|\mS||\mA|}{(1-\gamma)^3\epsilon^2}\big)$~\cite{azar2013minimax}\\
% \bottomrule[1.3pt]
% \end{tabular}
% \end{center}
% \vspace{-0.25in}
% \end{table}

% Representatio
\medskip
\noindent{\bf Representation.} Given state space $\mS = [0,1]^{d_1}$ and action space $\mA= [0,1]^{d_2}$, 
let $Q^* : \mS \times \mA \to \mathbb{R}$ be the optimal $Q$-function for the RL task of interest. 
We consider the integral operator $\Kq  = \Kq_{Q^*}$ induced by $Q^*$ as its kernel that maps 
any real-valued integrable function $h: \mS \to \mathbb{R}$ to $\Kq h: \mA \to \mathbb{R}$ with 
$\Kq h(a) = \int_{s \in \mS} Q(s, a) h(s) ds, ~\forall a \in \mA$. For Lipschitz $Q^*$, 
we show that $\Kq$ is a Hilbert-Schmidt operator admitting generalized singular value decomposition. 
This leads to the representation of $Q^*$:
\begin{equation}\label{eq:representation}
    Q^*(s, a) = \sum_{i=1}^\infty \sigma_i f_i(s) g_i(a), \quad\forall~s \in \mS, a \in \mA,
\end{equation}
with $\sum_{i=1}^\infty \sigma_i^2 < \infty$, and ``singular vectors'' $\{f_i: i \in \mathbb{N}\}$ and 
$\{g_i: i \in \mathbb{N}\}$ being orthonormal sets of functions. That is, for any $\delta > 0$, there exists 
$r(\delta)$ such that the $r(\delta)$ components in \eqref{eq:representation} provide $\delta$-approximation 
of $Q^*$. This inspires a parametric family of $Q^*$ parameterized by $r \geq 1$, i.e.,  
$Q^*(s, a) = \sum_{i=1}^r \sigma_i f_i(s) g_i(a)$, with all Lipschitz $Q^*$ captured as $r\to\infty$. 
When $r$ is small, it suggests a form of lower-dimensional structure within $Q^*$: we call such a $Q^*$ 
to have {\em rank} $r$.   

% Main result. Vs Lower bound.
% Comparison Table
\medskip
\noindent{\bf Sample-Efficient RL.} Given the above universal representation with the notion of dimensionality for $Q^*$ 
through its rank, we develop a data-efficient RL method. Specifically, for any $\epsilon > 0$, our method finds $\hat{Q}$ 
such that $\|\hat{Q} - Q^*\|_\infty \leq \epsilon$ using $\tilde{O}\big(\epsilon^{-(\max\{d_1, d_2\}  + 2)}\big)$ samples, 
with the hidden constant in $\tilde{O}(\cdot)$ dependent on $r, \max\{d_1, d_2\}$ (cf. Theorem \ref{thm:generic}). 
In contrast, the minimax lower bound for learning a generic Lipschitz $Q^*$ 
in the $L^{\infty}$ sense (also in the $L^2$-sense)
is of ${\Omega}\big(\epsilon^{-(d_1 + d_2 + 2)}\big)$ \cite{tsybakov2008introduction}. 
That is, our method removes the dependence on the smaller of the two dimensions 
by exploiting the low-rank structure in $Q^*$. 
Note that this provides an exponential improvement in sample complexity, 
e.g., with $d_1 = d_2 = d$, our method requires the number of samples scaling as $\epsilon^{-d-2}$ in contrast to $\epsilon^{-2d - 2}$ required for generic Lipschitz $Q^*$. 
For a quick comparison with some related works, see Table \ref{tab:one} and Section \ref{sec:related}.% with a few selected works. 

\medskip
\noindent{\bf Matrix Estimation (ME), A Novel Method.}  Our data-efficient RL method relies on a novel low-rank Matrix Estimation 
method we introduce. Notice that for any set of $m$ states $\{s_k\}_{k=1}^m$ and $n$ actions $\{a_\ell\}_{\ell=1}^n$, 
the induced matrix $[Q^*(s_{k}, a_{\ell}): k \in [m], \ell \in [n]]$ has rank (at most) $r$.
{Naively, when the $m$ chosen states ``cover'' $\mS$ finely ($n$ actions cover $\mA$, resp.) and suppose we also have a good estimate for the entire matrix, 
we can estimate $Q^*$ for the entire domain $\mS \times \mA$ by interpolating the estimates for the $mn$ entries.
This leads to the sample complexity of $\tilde{O}\big(\epsilon^{-(d_1 + d_2 + 2)}\big)$, matching the mini-max lower bound.}

{ To overcome the barrier in sample complexity, we suggest to utilize the low-rank structure of $Q^*$ by developing a novel matrix estimation method.}
At a high level, to obtain 
the improved sample complexity {$\tilde{O}\big(\epsilon^{-(\max\{d_1, d_2\}  + 2)}\big)$} as claimed, we wish to faithfully recover the $m\times n$ rank $r$ matrix in the $\ell_{\infty}$ sense, 
by observing only $\tilde{O}\big(\max(m, n) r\big)$ entries with each entry having bounded but arbitrary noise $\delta$. 
In literature \cite{candes2009exact, candes2010power, chen2018harnessing,davenport2016overview}, such a harsh setting 
has not been considered. In this work, we introduce an %extremely simple 
ME method that manages to reconstruct the entire matrix with entry-wise error within $O(\delta)$ (cf. Proposition \ref{prop:rankr_simplified}). 
This advance in ME should be of independent interest (see Table \ref{tab:two} for comparison).
With this novel method, we improve our estimates of $Q^*$ iteratively by interleaving one-step lookahead 
and matrix estimation steps. This, ultimately leads to an $\epsilon$-optimal $Q^*$ with desired sample size.

\begin{table}[t]
\setlength{\tabcolsep}{5pt}
\small
\caption{ Comparison of different ME methods with different guarantees. Ours is the only method that provides entry-wise guarantee while allowing for arbitrary, bounded error in each entry.}
\label{tab:two}
% \vspace{-8pt}
\begin{center}
\begin{tabular}{l c c c c}
\toprule[1.3pt]
 Method             &   Noise Model         &   Error Guarantees     &   Sampling Model  &   \# of Samples \\
\midrule
Our Method                &    bounded arbitrary   &  entrywise           &   adaptive    &   $O(n)$ \\
 \hline
 Convex Relaxation  & noiseless    &  exact    &   independent w.p. $p$    &   $O(n\log^2 n)$ \\
   % \cline{2-5}
 \cite{candes2010power, candes2010matrix, koltchinskii2011nuclear}                 &   bounded arbitrary   &   Frobenius           &   independent w.p. $p$    &   $O(n \log^2 n ) $ \\
 \hline
 Spectral Thresholding~\cite{chatterjee2015matrix} &  zero-mean     &   Frobenius   & independent w.p. $p$ &  $O(n^{1+c})$ \\
 \hline
 Factorization (noncvx)~\cite{chen2019noisy} &  zero-mean    &   entrywise   &   independent w.p. $p$ & $O(n\log^3n)$ \\
\bottomrule[1.3pt]
\end{tabular}
\vspace{-0.2in}
%  \vspace{-0.25in}
\end{center}
\end{table}

\medskip
\noindent{\bf Empirical Success.} While low-rank representation of $Q^*$ enables theoretical guarantees, the proof is
in the puddling: we find that for well-known control tasks, the underlying $Q^*$
has a low-rank structure. In particular, using our method that exploits the low-rank structure leads to a significant 
improvement in sample complexity over the method that does not. Our novel matrix estimation method, with provable
guarantees, turns out to be computationally most efficient, while offering superior performance of sample complexity. 
%most efficient and simple.

\medskip
\noindent{\bf Summary.}
% Overall, to the best of our knowledge, our result is the first to show such a provable, quantitative sample complexity improvement 
% for RL with continuous state and action spaces via low-rank structure. We believe that ``factorization'' of $Q^*$ can be beneficial more broadly in improving efficiency of RL.
Overall, to the best of our knowledge, our result is the first to show such a provable, quantitative sample complexity improvement 
for RL with continuous state and action spaces via low-rank structure. We believe that ``factorization'' of $Q^*$ can be beneficial 
more generally in improving the efficiency of RL, e.g., it could be embedded as an architectural constraint in neural network 
representation of the $Q^*$.
Moreover, our discussion in this work is not limited to $Q^*$ in RL; the main insight we develop in this paper remains 
valid and applicable more broadly for various problems in machine learning and other related fields beyond RL.
On the representation side, ``nice'' bivariate functions in many other problems should also possess a similar low-rank 
spectral representation with respect to the two variables involved. 
On the algorithmic side, we can utilize the framework introduced in this work to devise an algorithm that estimates 
such a function in a sample-efficient manner via iterative estimation of (sub-)matrices, indexed by the two variables.

%We believe that ``factorization'' in the structure of $Q^*$ to achieve efficiency in learning can be beneficial more generally, e.g. as an architectural constraint in neural network representation of the $Q^*$.

% \begin{table}[t]
% \setlength{\tabcolsep}{5pt}
% \small
% \caption{ Comparison of different ME methods with different guarantees. Ours is the only method that provides entry-wise guarantee while allowing for arbitrary, bounded error in each entry.}
% \label{tab:two}
% % \vspace{-8pt}
% \begin{center}
% \begin{tabular}{l c c c c}
% \toprule[1.3pt]
%  Method             &   Noise Model         &   Error Guarantees     &   Sampling Model  &   \# of Samples \\
% \midrule
% Our Method                &    bounded arbitrary   &  entrywise           &   adaptive    &   $O(n)$ \\
%  \hline
%  Convex Relaxation  & noiseless    &  exact    &   independent w.p. $p$    &   $O(n\log^2 n)$ \\
%   % \cline{2-5}
%  \cite{candes2010power, candes2010matrix, koltchinskii2011nuclear}                 &   bounded arbitrary   &   Frobenius           &   independent w.p. $p$    &   $O(n \log^2 n ) $ \\
%  \hline
%  Spectral Thresholding~\cite{chatterjee2015matrix} &  zero-mean     &   Frobenius   & independent w.p. $p$ &  $O(n^{1+c})$ \\
%  \hline
%  Factorization (noncvx)~\cite{chen2019noisy} &  zero-mean    &   entrywise   &   independent w.p. $p$ & $O(n\log^3n)$ \\
% \bottomrule[1.3pt]
% \end{tabular}
% \vspace{-0.2in}
% %  \vspace{-0.25in}
% \end{center}
% \end{table}

\subsection{Related Work}\label{sec:related}
A brief discussion of related  work on Reinforcement Learning and Matrix Estimation is provided. 

\medskip
\noindent{\bf Reinforcement Learning.}
Reinforcement Learning problems with both continuous state and action space received significantly less attention in literature. 
While there are practical RL algorithms to deal with continuous domains~\cite{van2012reinforcement,lazaric2008reinforcement,haarnoja2018soft,lillicrap2015continuous}, theoretical 
understanding on this class of problems, especially on sample complexity, is very limited~\cite{antos2008fitted}. 
Since we interpolate our estimates to the entire space via non-parametric regression without making 
any additional model assumptions, a comparison with the non-parametric minimax rate ${\Omega}(\frac{1}{\epsilon^{d_1+d_2+2}})$ 
for learning Lipschitz function~\cite{stone1982optimal,tsybakov2008introduction} is meaningful. 
%\red{DG: I think we should start a new paragraph here}

Our algorithm and proofs are general, which can be reduced to low-rank 
settings with a finite (discrete) space in a similar manner (Appendix~\ref{subsec:coro_finite}). 
The %sample complexity 
lower bound scales as
$\tilde{\Omega}(\frac{1}{\epsilon^{d+2}})$~\cite{shah2018q} for problems with continuous state space and finite action space  and {\small $\tilde{O}(\frac{|\mS||\mA|}{(1-\gamma)^3\epsilon^2})$}~\cite{azar2013minimax} for problems with both state and action spaces being finite.
When reduced to those domains, 
our method scales as { $\tilde{O}(\frac{1}{\epsilon^{d+2}})$} for the former and {\small $\tilde{O}(\frac{\max(|\mS|,|\mA|)}{ \epsilon^{2}})$} for the latter, respectively.
That is, the smaller of the two dimensions is ``removed'' from sample complexity by exploiting the low-rank structure in the same way as in the continuous problems.
Results in finite domains are abundant in literature and it is impossible to cover them all. 
We provide a high-level summary in Table~\ref{tab:one} to communicate how our algorithm fares with a few selected work. 
Note that the detailed setting often varies in literature and we refer readers to Appendix~\ref{sec:discussion} 
for further discussions. 
Finally, we remark that our  analysis requires the discounting factor $\gamma$ to be small, and leave it as an important future direction to extend to all $\gamma$.

We mention the recent empirical work~\cite{yang2020harnessing} that investigates low-rank $Q^*$ with matrix estimation for finite state
and action spaces. The results in \cite{yang2020harnessing} are solely empirical and it uses off-the-shelf ME methods. In that sense,
we provide a formal framework to understand why \cite{yang2020harnessing} works so well, resolving the theoretical open problem 
raised in their work, and we provide natural generalization for continuous state and action spaces that was missing, along with a novel
ME method. 

\medskip
\noindent{\bf Matrix Estimation.}
As discussed, matrix estimation concerns recovering a low-rank $m \times n$ matrix from partial, noisy observation of it. 
This problem has been extremely well studied~\cite{recht2010guaranteed, candes2009exact,candes2010power, koltchinskii2011nuclear, chatterjee2015matrix,chen2015fast,davenport2016overview,chen2018harnessing}. However, % ,keshavan2010matrix
most recovery guarantees are given in terms of Frobenius norm of the error, or mean squared error. In this work, we need reliable estimation 
for {\em each} entry, i.e., $\ell_\infty$ error bound. This is technically hard and there are only limited results~\cite{ding2020leave,chen2019noisy}. To make matters worse, the measurement noise in our setting can be arbitrary (not necessarily zero mean) though bounded. Therefore, a new method is required and that is precisely what we do in this work. See Appendix \ref{sec:discussion} for more detailed discussions on why existing matrix estimation methods do not work and
ours does, along with directions for future research. 

\subsection{Organization}
{The remainder of the paper is organized as follows. We introduce a formal representation theorem of $Q^*$ in Section~\ref{sec:formulation}. In Section~\ref{sec:alg_generic}, we propose our efficient RL algorithm using low-rank ME. The generic convergence and sample complexity results are established in Section~\ref{sec:main_result}, under a suitable assumption on the ME method. Section~\ref{sec:lowrankME} is dedicated to the development of our new ME method that fulfills the requirement. We provide empirical evidence in Section~\ref{sec:empirical}. In Section~\ref{sec:main_discussion}, we offer a short discussion on aspects of our ME method with full discussion deferred to Appendix~\ref{sec:discussion}. 
All the proofs as well as additional experimental results can be found in Appendices.}

\section{Markov Decision Process and Representation of $Q$-function}\label{sec:formulation}
\subsection{Markov Decision Process (MDP)} 
We consider the standard setup of infinite-horizon discounted MDP, 
which is described by $(\mS,\mA,\mathcal{P},R,\gamma).$ $\mS$ and $\mA$ are the state and action 
spaces, respectively. $\mathcal{P}(s'|s,a)$ is the unknown transition kernel, while $R(s,a)$ determines the 
immediate reward received. Finally, $\gamma\in(0,1)$ is the discounting factor. A policy $\pi(a|s)$ specifies 
the probability of selecting action $a\in\mA$ at state $s\in\mS$. The standard value function associated with a policy $\pi$ is defined as
$
V^{\pi}(s)=\E_{\pi}[\sum_{t=0}^{\infty}\gamma^{t}R(s_{t},a_{t})~|~s_{0}=s].
$
The optimal value function, denoted by $V^*$, is the value function of the reward-maximizing policy. 
That is, 
$
V^{{*}}(s)=\sup_{\pi}V^{\pi}(s),\forall s\in\mS.
$
Correspondingly, we define the optimal $Q$-function, denoted by $Q^*$, as
$
Q^*(s,a) = R(s,a) + \gamma \mathbb{E}_{s'\sim \mathcal{P}(\cdot|s,a)}[V^*(s')].
$

\medskip
\noindent{\bf MDP Regularity.} 
Throughout this paper, we assume the existence of a generative model (i.e., a simulator)~\cite{kakade2003sample}. 
We consider MDPs with the following properties: 
\begin{enumerate}
    \item
    (Compact domain) The state space $\mS$ and the action space $\mA$ are compact subsets of a Euclidean space; Without loss of generality, let $\mS=[0,1]^{d_1}$ and 
$\mA=[0,1]^{d_2}$. 
    \item
    (Bounded reward) For every $(s,a)\in\mS\times\mA$, the reward $R(s,a)$ is bounded, i.e., 
$|R(s,a)|\leq R_{\max}$. 
    \item
    (Smoothness) The optimal $Q$-function, $Q^*$, is $L$-Lipschitz %$L > 0$, with 
with respect to the 1-product metric in $\mS \times \mA$, i.e., $|Q^*(s_1,a_1) - Q^*(s_2,a_2)|\leq L d_{\mS \times \mA}\big( (s_1, a_1), (s_2, a_2) \big)$ 
where $d_{\mS \times \mA}\big( (s_1, a_1), (s_2, a_2) \big) = ||s_1-s_2||_2 + \| a_1 - a_2 \|_2$. %for $s_1, s_2\in\mS, 
\end{enumerate}

We note that the bounded reward implies that for any policy $\pi$, $|V^{\pi}(s)| \leq V_{\max} \triangleq R_{\max} /(1 - \gamma)$ for all $s$. 
This yields $|Q^*(s,a)| \leq V_{\max}$, too. Finally, we remark that for learning MDPs with continuous state/action space under $\ell_\infty$ 
guarantee, some form of smoothness assumption, such as the Lipschitz continuity above, is natural and typical~\cite{yang2019theoretical,antos2008fitted,shah2019nonasymptotic,shah2018q,dufour2012approximation}.

\subsection{Spectral Representation of $Q$-function} 
% With the discussion above and $d = d_1 = d_2$, $Q^*: [0,1]^d \times [0,1]^d \to \mathbb{R}$ 
% is bounded. As introduced earlier, it induces an integral operator $A$ on the space of bounded 
% continuous functions $C([0,1]^d)$ endowed by the standard inner-product $\langle f, g \rangle = \int_{x \in [0,1]^d} f(x)g(x) dx$. 
% Through this lens, we obtain the following representation for $Q^*$. 
With the discussion above, $Q^*: [0,1]^{d_1} \times [0,1]^{d_2} \to \mathbb{R}$ is $L$-Lipschitz and also bounded. 
%\red{because it is a continuous function on a compact domain (even without bounded reward assumption). delete?} 
As introduced earlier, it induces an integral kernel operator $\Kq = \Kq_{Q^*}: L^2([0,1]^{d_1}) \to L^2([0,1]^{d_2})$ between the spaces of square integrable functions $L^2([0,1]^d)$ (for $d \in \{ d_1, d_2\}$) 
endowed with the standard inner product $\langle f, g \rangle = \int_{x \in [0,1]^d} f(x)g(x) dx$. 
Through this lens, we obtain the following representation for $Q^*$, which follows from noticing that $\Kq$ is a Hilbert-Schmidt operator. See Appedix \ref{sec:proof_thm1} for the proof.

\begin{thm}\label{thm:representation}
Suppose the MDP regularity conditions (1) - (3). Then there exist a nonincreasing sequence $( \sigma_i \geq \RR_+: i \in \NN )$ with $\sum_{i=1}^{\infty} \sigma_i^2 < \infty$ and  orthonormal sets $\{ f_i \in L^2([0,1]^{d_1}): i \in \NN \}$ and $\{ g_i \in L^2([0,1]^{d_2}): i \in \NN \}$ such that
\begin{equation}\label{eq:qstar}
    Q^*(s,a) = \sum_{i=1}^{\infty} \sigma_i f_i(s) g_i(a), \quad\forall (s, a) \in [0,1]^{d_1} \times [0,1]^{d_2}. 
\end{equation}
As a result, for any $\delta > 0$, there exists $r^* = r^*(\delta) \in \NN$ such that for all $r \geq r^*$, the rank-$r$ approximation error satisfies
$\int_{\mS \times \mA} \big(\sum_{i=1}^{r} \sigma_i f_i(s) g_i(a) - Q^*(s,a) \big)^2 ds~da = \sum_{i=r+1}^{\infty} \sigma_i^2 \leq \delta$.
\end{thm}

\medskip
\noindent{
\bf Low Rank $Q^*$.} 
Theorem \ref{thm:representation} motivates us to consider low-rank $Q^*$. 
For any integer $r \geq 1$, we call $Q^*$ to have rank $r$ if $\sigma_i = 0$ for all $i > r$ in \eqref{eq:qstar}. 
More generally, we say $Q^*$ has $\delta$-approximate rank $r$ if $r^*(\delta) = r$ in  Theorem \ref{thm:representation}. 
We focus on efficient RL for $Q^*$ with exact or approximate low rank $r$. 
%For ease of notation, we let $d_1=d_2=d$ in the sequel. We remark that our theorems apply equally by simply replacing $d$ with $\max\{d_1,d_2\}$.

\section{Algorithm for Reinforcement Learning with Matrix Estimation}\label{sec:alg_generic}	
%\vspace{-0.05in}
We introduce an RL algorithm using generic ME procedures as a subroutine. We require the ME method in use to satisfy 
Assumption \ref{assu:me} (see Section \ref{sec:main_result}) to provide meaningful performance guarantees. However, 
there is no known ME procedure satisfying Assumption \ref{assu:me} in literature. In Section \ref{sec:lowrankME}, 
we introduce a simple ME procedure that satisfies it when $Q^*$ is exactly or approximately low-rank.

\subsection{A Narrative Description of the Algorithm}\label{sec:alg_narrative}

The RL algorithm iteratively improves estimation of $Q^*$. Each iteration consists of four steps: 
discretization, exploration, matrix estimation and generalization. We provide a narrative 
overview of the algorithm first; see Algorithm \ref{alg:generic} in Section \ref{sec:pseudocode-generic} 
for the full algorithm description in a pseudo-code format.

\begin{figure}[H]
%  \vspace{-0.06in}
    \centering
    \includegraphics[width = 0.99\textwidth]{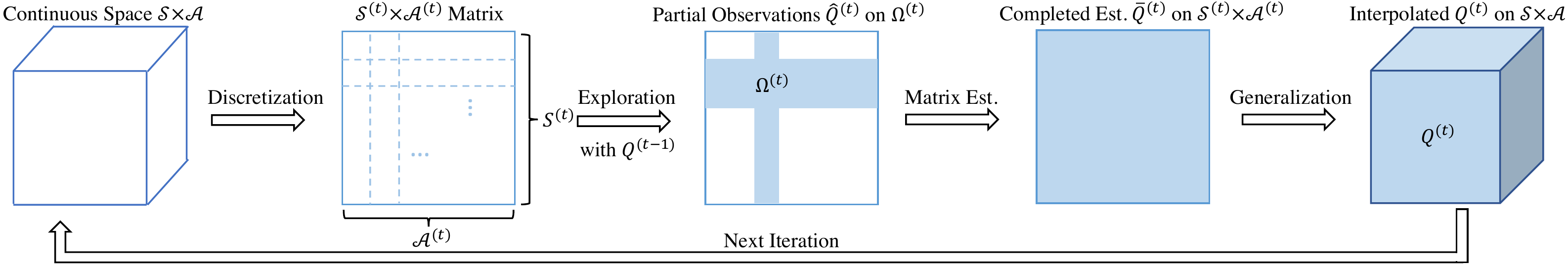}
    % \vspace{-0.07in}
    \caption{ Iterative RL using ME: exploration uses estimation ${Q}^{(t-1)}$ from
    previous iteration.}
    % \vspace{-0.12in}
    \label{fig:my_label}
\end{figure}

\medskip
\noindent{\hspace{-0.06in}
\bf Step 1. Discretization.}  
At iteration $t$, we produce $\beta^{(t)}$-nets, $\mS^{(t)} \subset \mS$ and $\mA^{(t)} \subset \mA$, for
properly chosen resolution $\beta^{(t)} \in (0,1)$ that decreases with iteration $t$. In our setup, $|\mS^{(t)}| = O\big((1/\beta^{(t)})^{d_1}\big)$, 
$|\mA^{(t)}| = O\big((1/\beta^{(t)})^{d_2}\big)$. In total, this produces $|\mS^{(t)}||\mA^{(t)}|$ many $(s,a)$ pairs 
in the discretized set $\mS^{(t)}\times\mA^{(t)}$. 

\medskip
\noindent{\bf Step 2. Exploration.} 
Using estimate ${Q}^{(t-1)}$ over the entire $\mS\times \mA$ from the previous iteration, 
we wish to produce an improved estimate of $Q^*$ over $\mS^{(t)}\times\mA^{(t)}$ through this and the next step, and 
then generalize it to $\mS \times \mA$ in Step 4. 
To produce an improved estimate over $\mS^{(t)}\times\mA^{(t)}$ in a sample-efficient manner, we first ``explore'' a carefully selected subset $\Omega^{(t)}\subset\mS^{(t)}\times\mA^{(t)}$. Specifically, for each $(s,a)\in\Omega^{(t)}$, we obtain $N^{(t)}$ 
samples of independent transitions using the generative model, which results in a set of sampled next states 
$\{ s'_i \}_{i= 1, \ldots, N^{(t)}}$. We obtain an estimate $\hat{Q}^{(t)}(s,a)$ as 
\begin{equation}
\label{eq:generic:lookahead.inside}
  \hat{Q}^{(t)}(s,a) \gets R(s,a) +\gamma\cdot \frac{1}{N^{(t)}}\sum_{i=1}^{N^{(t)}}V^{(t-1)}(s'_i),\quad\mbox{with}~V^{(t-1)}(s) = \max_{a} Q^{(t-1)}(s,a).
\end{equation}
%where $V^{(t-1)}(s) = \max_{a} Q^{(t-1)}(s,a)$. 

%We achieve that in a sample efficient manner in this and the next step. Specifically, in this step, we produce improved estimate only for  And using them, with help of an appropriate ME in the next step, we obtain improved estimate for all $(s, a) \in \mS^{(t)}\times\mA^{(t)}$. 

%We explore only the $(s,a)$ pairs in the selected subset $\Omega^{(t)}$ to obtain improved estimates for them, which serve as partial observations to complete the whole $|\mS^{(t)}|\times |\mA^{(t)}|$ matrix. Precisely, we explore    

\medskip
\noindent{\bf Step 3. Matrix Estimation.} 
Given estimates $\hat{Q}^{(t)}(s,a), \forall (s,a) \in \Omega^{(t)}$ updated in Step 2, 
% Given estimates $\hat{Q}^{(t)}(s,a), \forall (s,a) \in \Omega^{(t)}$ from Step 2, 
we wish to obtain an improved estimate of $Q^*$ for the entire $\mS^{(t)}\times\mA^{(t)}$. 
This can be viewed as a matrix estimation problem. When $Q^*$ has rank $r$ as discussed in Section \ref{sec:formulation},
the sampled matrix $[Q^*(s, a): s \in\mS^{(t)}, a \in\mA^{(t)}]$, induced by discretization, has rank at most $r$. 
Thus, we want to estimate the low-rank matrix by having access to noisy measurements for a subset of entries 
in $\Omega^{(t)}\subset\mS^{(t)}\times\mA^{(t)}$. Specifically, the noise in the measurements are not necessarily i.i.d. as they are coupled through 
$V^{(t-1)}$; thus, they are bounded but can be arbitrary. 
Ideally, we wish to estimate the matrix with the maximum entrywise error at a similar level 
%or less than 
as that in $\hat{Q}^{(t)}(s,a)$. This demands that the ME method in use is well-behaved in the $\ell_{\infty}$ sense, satisfying 
Assumption \ref{assu:me} to be stated later. While such a result is absent in literature, we shall describe ME methods fulfilling the desideratum in Section \ref{sec:lowrankME}. 
Consequently, we obtain improved estimates $\bar{Q}^{(t)}(s,a)$ for all $(s,a)\in\mS^{(t)}\times\mA^{(t)}$ after the ME step.

%This demands that the ME method in use to satisfy an $l_\infty$ guarantee, as stated later in Assumption  
%Assumption \ref{assu:me} as stated later

%Importantly, this is made feasible by viewing it as a Matrix Estimation problem for a $|\mS^{(t)}|\times |\mA^{(t)}|$ matrix, where the rows are indexed by states in $\mS^{(t)}$, the columns are indexed by actions in $\mA^{(t)}$ and the entries are the $Q^*$-function restricted to the corresponding $\mS^{(t)}\times\mA^{(t)}$. 

%Now that we have partial observations $\{\hat{Q}^{(t)}(s,a)\}_{(s,a)\in\Omega^{(t)}}$, we call a low-rank Matrix Estimation oracle to obtain a $|\mS^{(t)}|\times |\mA^{(t)}|$ matrix $\bar{Q}^{(t)}$, where each entry of $\bar{Q}^{(t)}$ represents an estimate of $Q^*(s,a)$ for the
%for the respective pair $(s,a)\in\mS^{(t)}\times\mA^{(t)}$. We note that which subset $\Omega^{(t)}$ to explore is closely related to the specific Matrix Estimation algorithm. Ideally, we would like an oracle that requires less exploration (observations) but completes well. 

\medskip
\noindent{\bf Step 4. Generalization.} 
With estimates $\bar{Q}^{(t)}(s,a), ~(s,a)\in\mS^{(t)}\times\mA^{(t)}$, 
we generalize to $\mS \times \mA$ via interpolating them. This can be achieved by any supervised 
learning algorithm. We simply utilize the $1$-nearest neighbor: for any $(s,a) \in \mS \times \mA$, 
at the end of iteration $t$ we output $Q^{(t)}(s,a) \gets \bar{Q}^{(t)}(s',a')$ 
where $(s',a')$ is closest to $(s,a)$ in $\mS^{(t)}\times \mA^{(t)}$, with ties broken arbitrarily.
%This results in estimate $Q^{(t)}:\mS\times\mA\rightarrow\mathbb{R}$.

%Overall, the idea is to improve over one iteration so that we 

\subsection{{Pseudo-Code for the Proposed Algorithm}}\label{sec:pseudocode-generic}

Below is the pseudo-code of the generic RL method described in Section \ref{sec:alg_narrative}.
	\begin{algorithm}[h]
	\small
		\caption{Main Algorithm: Low-rank Reinforcement Learning}
		\label{alg:generic}

		\hspace*{\algorithmicindent} \textbf{Input:}$\quad$ $\mS$, $\mA$, $\gamma$, $Q^{(0)}$, $T$, $\{\beta^{(t)}\}_{t=1, \ldots, T}$, $\{N^{(t)}\}_{t= 1, \ldots, T}$  \\
        \hspace*{\algorithmicindent} \textbf{Output: } $Q^{(T)}$, the $Q$-value oracle after $T$ iterations
        %\vspace{0.1in}
	\begin{algorithmic}[1]
    	\STATE {\bfseries Initialization:} For all $s \in \mS$, initialize the value oracle $Q^{(0)}(s)$. %\vspace{0.1in}
    	
% 		\STATE Initialize $\beta^{(1)}$, $V^{(0)}$\;
		\FOR{$t = 1,2, \dots, T$}{
%		Step 1: Discretization of the state / action spaces
		\STATE \texttt{/*}~~\texttt{Step 1: Discretization of $\mS$ and $\mA$}~~\texttt{*/}
		\STATE Discretize $\mS$ and $\mA$ so that $\mS^{(t)}$ is a $\beta^{(t)}$-net of $\mS$ and $\mA^{(t)}$ is a $\beta^{(t)}$-net of $\mA$. %\vspace{0.05in}
		
%		Step 2: Monte-Carlo tree search update
        \STATE \texttt{/*}~~\texttt{Step 2: Exploration of a few $(s,a)$ pairs}~~\texttt{*/}
		\STATE Select a subset of $(s,a)$ pairs, $\Omega^{(t)} \subseteq \mS^{(t)}\times \mA^{(t)}$.
		\FOR{$(s,a)\in\Omega^{(t)}$}{
		\STATE Estimate $Q^*(s,a)$ via simple lookahead based on the current value oracle $V^{(t-1)}$,
		    i.e., query the generative model to sample $N^{(t)}$ independent transitions from $(s,a)$
		    and obtain an estimate $\hat{Q}^{(t)}(s,a)$ with the sampled next states $\{ s'_i \}_{i= 1, \ldots, N^{(t)}}$:
		\vspace{-0.1in}
		\begin{equation}
		%\label{eq:generic:lookahead}
		    \hat{Q}^{(t)}(s,a) \gets R(s,a) +\gamma\cdot \frac{1}{N^{(t)}}\sum_{i=1}^{N^{(t)}}V^{(t-1)}(s'_i).
		    \vspace{-0.1in}
		\end{equation}
		}
		\ENDFOR%\vspace{0.05in}
		
%		Step 3: Matrix Completion
		\STATE \texttt{/*}~~\texttt{Step 3: Matrix completion to obtain $\bar{Q}$ from $\hat{Q}$}~~\texttt{*/}
		\STATE Estimate $\bar{Q}(s,a)$ for $(s,a) \in \mS^{(t)} \times \mA^{(t)}$ from the data $\{\hat{Q}^{(t)}(s,a)\}_{(s,a)\in\Omega^{(t)}}$,
		    utilizing the low-rank structure of $\bar{Q}(s,a)$, viz.,
		\begin{equation*}
		    \bar{Q}^{(t)} \gets \textrm{Matrix Estimation}\big( \hat{Q}^{(t)};~ \Omega^{(t)}\big).
		\end{equation*}

%		Step 4: Update of the oracle via interpolating \bar{Q}		
		\STATE \texttt{/*}~~\texttt{Step 4: Generalization via interpolating $\bar{Q}$}~~\texttt{*/}
		\STATE Update the oracles $Q^{(t)}$ and $V^{(t)}$ by calling a subroutine that interpolates  $\bar{Q}^{(t)}$ through non-parametric regression methods:
		\begin{align*}
		    Q^{(t)} \gets \textrm{Interpolation}\big( \bar{Q}^{(t)};~ \mS^{(t)}, \mA^{(t)} \big),
		\end{align*}
		and subsequently, $V^{(t)}(s)\gets \max_{a\in\mA}Q^{(t)}(s,a)$, for all  $s\in\mS$.
		}
		\ENDFOR
	
		%\vspace{10pt}
		\end{algorithmic}
	\end{algorithm}

\section{Main Result: Correctness, Convergence \& Sample Complexity}
\label{sec:main_result}

In this section, we state the result establishing correctness, convergence and finite sample analysis of our RL algorithm. We 
require a specific property, stated as Assumption \ref{assu:me}, for the Matrix Estimation (ME) method utilized 
in Step 3 of the algorithm. While there is no known ME method in the literature that satisfies it, we provide a novel ME
method with the desired property in Section \ref{sec:lowrankME}.

\subsection{{Matrix Estimation: a Key Premise}}
Recall that we describe Algorithm \ref{alg:generic} with a generic matrix estimation subroutine used in Step 3, 
without specifying what ME method is used. In fact, the success of Algorithm \ref{alg:generic} hinges on the performance 
of the ME method in use. For the convenience of exposition, we define $(\C, \ceta)$-property of an ME method for given constants 
$\C, \ceta \geq 0$, which serves as a pivotal premise for the success of Algorithm \ref{alg:generic}.

\begin{assumption}[$(\C, \ceta)$-property]
\label{assu:me}
Given finite $\mS^{(t)} \subset \mS$, $ \mA^{(t)} \subset \mA$, it is possible to construct $\Omega^{(t)} \subseteq \mS^{(t)} \times \mA^{(t)}$ 
with $|\Omega^{(t)}| \leq \C \big( | \mS^{(t)} | + |\mA^{(t)}| \big)$ for given constant $\C \geq 1$ so that 
whenever the ME method in use takes $\{ \hat{Q}^{(t)}(s,a) \}_{(s,a) \in \Omega^{(t)}}$ with $\max_{(s,a) \in \Omega^{(t)}} \big| \hat{Q}^{(t)}(s,a) - Q^*(s,a) \big| \leq \epsilon $ as an input 
and outputs $\{\bar{Q}^{(t)}(s,a) \}_{(s, a) \in \mS^{(t)} \times \mA^{(t)}}$, the following inequality holds:
\begin{equation*}
    \max_{(s,a) \in \mS^{(t)}\times \mA^{(t)}} \big| \bar{Q}^{(t)}(s,a) - Q^*(s,a) \big| \leq \ceta  \epsilon. 
\end{equation*}
\end{assumption}

We assume there exists an ME method that satisfies $(\C, \ceta)$-property and we have such a method at hand. 
Assumption \ref{assu:me} ensures the $\ell_{\infty}$ error remains under control (to be precise, $\ceta$-Lipschitz 
with respect to $\ell_{\infty}$/$\ell_{\infty}$) during the ME step, while it is stated in the language of RL 
for later uses.
Note that Assumption \ref{assu:me} does not explicitly require any structure on $Q^*$. We will require $Q^*$ to be 
low-rank or approximately low-rank to produce an ME method satisfying the assumption, as will be discussed in Section \ref{sec:lowrankME}. 

%In the language of ME, Assumption~\ref{assu:me} states that if we select enough entries to observe and that the noisy observations of them are at most $\epsilon$ away from their true values, then the completed matrix from these partial observations guarantees an entry-wise error at most $c_\eta\epsilon$ from the  true value.  
% \begin{assumption}
% \label{assu:me}
% Consider recovering a true matrix $M^*\in\mathbb{R}^{n\times m}$ of rank \red{at most} $r \ll (n\wedge m)$.
% Suppose that we can select $|\Omega| = O(n\vee m)$ noisy observations, i.e., linear number of observations, $\hat{M}_\Omega$, such that
% \begin{equation}
%     |M^*(i,j) - \hat{M}(i,j)|\leq \epsilon,\quad\forall\:(i,j)\in\Omega. 
% \end{equation}
% Then, there exists a Matrix Estimation oracle, $\textrm{ME}(\hat{M}_\Omega)$, such that 
% %with probability $1- \delta\big(poly(n),poly(1/\epsilon),r\big)$ 
% the recovered matrix based on the selected observations, i.e., $\bar{M}= \textrm{ME}(\hat{M}_\Omega)$, satisfies
% \begin{equation}
% \max_{(i,j)\in[n]\times[n]}||\bar{M} - M^*||_{\max} \leq \eta \cdot \epsilon,
% \end{equation}
% where $\eta\in\mathbb{R}_+$ is an absolute constant.
% \end{assumption}

% \red{(DG:
% The above guarantee implicitly requires that
% $Q^*(s,a) = \sum_{i=1}^rf_i(s)\cdot g_i(a), \:\forall\:s\in\mS, a\in\mA$, for some small rank $r$ and a incoherence basis $f_i, g_i$. 
% Throughout the paper, we treat $r$ as a small constant and is hidden in all the big $O$ notation. We will see them with concrete examples later.
% )}

\subsection{Correctness, Rate of Convergence \& Sample Complexity of Algorithm \ref{alg:generic}}
Now, we state the desired properties of the RL algorithm introduced
in Section \ref{sec:alg_generic}. To that end, let the algorithm start with initialization $Q^{(0)}(s,a)=0, ~\forall (s,a) \in \mS \times \mA$ 
and hence $V^{(0)}(s) = 0, ~\forall s \in \mS$. That is, $|Q^{(0)}(s,a) - Q^*(s,a) |\leq V_{\max}$, $\forall (s,a) \in \mS \times \mA$. For the sake of notational brevity, we let $d_1=d_2=d$ in the sequel. We remark that our theorems apply equally by simply replacing $d$ with $\max\{d_1,d_2\}$.
% \begin{thm}
% \label{thm:generic}
% Consider the RL algorithm described in Section \ref{sec:alg_generic} with ME satisfying Assumption \ref{assu:me}. 
% Given $\delta\in(0,1)$, there exists algorithmic choice of $\beta^{(t)}, \Omega^{(t)}, N^{(t)}$ 
% for $1\leq t \leq T$, so that with probability at least $1-\delta$, we have
% \begin{align}\label{eqn:geom_rate}
% \sup_{ (s,a) \in \mS \times \mA } \big| Q^{(t)}(s,a) - Q^*(s,a) \big| 
%             \leq (2 \gamma \ceta)^t V_{\max}, \:\forall 1\leq t \leq T. 
% \end{align}
% Further, let $\gamma < \frac{1}{2 \ceta}$. Then, with $T=\Theta\big(\log\frac{1}{\epsilon}\big)$ and 
% $\widetilde{O}\Big( \frac{1}{\epsilon^{d+2}}\cdot\log\frac{1}{\delta}\Big)$
% number of samples, we have 
% \begin{align}\label{eqn:samples}
% \mathbb{P}\Big(\sup_{(s,a) \in \mS \times \mA}\big| Q^{(T)}(s,a) - Q^*(s,a) \big|\leq\epsilon\Big) & \geq 1-\delta. 
% \end{align}
% \end{thm}
\begin{thm}
\label{thm:generic}
Consider the RL algorithm described in Section \ref{sec:alg_generic} with ME satisfying Assumption \ref{assu:me}. 
Given $\delta\in(0,1)$, there exists algorithmic choice of $\beta^{(t)}, \Omega^{(t)}, N^{(t)}$ 
for $1\leq t \leq T$, so that
\begin{align*}%\label{eqn:geom_rate}
\mathbb{P}\bigg(\sup_{ (s,a) \in \mS \times \mA } \big| Q^{(t)}(s,a) - Q^*(s,a) \big| 
            \leq (2 \gamma \ceta)^t V_{\max}, \quad\forall 1\leq t \leq T\bigg) & \geq 1-\delta. 
\end{align*}
Further, let $\gamma < \frac{1}{2 \ceta}$. Then, with $T=\Theta\big(\log\frac{1}{\epsilon}\big)$ and 
$\widetilde{O}\big( \frac{1}{\epsilon^{d+2}}\cdot\log\frac{1}{\delta}\big)$ number of samples, we have
\begin{align}\label{eqn:samples}
\mathbb{P}\bigg(\sup_{(s,a) \in \mS \times \mA}\Big| Q^{(T)}(s,a) - Q^*(s,a) \Big|\leq\epsilon\bigg) & \geq 1-\delta. 
\end{align}
\end{thm}
In the proof of Theorem \ref{thm:generic} presented in Appendix \ref{sec:proofthmgeneric}, we choose parameters
$\beta^{(t)} = \frac{V_{\max}}{8L} ( 2 \gamma \ceta)^t$, $|\Omega^{(t)}| = \C( |\mS^{(t)}| + |\mA^{(t)}| )$ and 
$N^{(t)} = \frac{ 8 }{ ( 2 \gamma \ceta )^{2(t-1)} } \log\big(\frac{2 | \Omega^{(t)} | T}{\delta}\big)$ for $1\leq t \leq T$. 
While this choice establishes the claims in Theorem \ref{thm:generic}, it is possible to 
achieve $\sup_{ (s,a) \in \mS \times \mA } \big| Q^{(t)}(s,a) - Q^*(s,a) \big| \leq \alpha^t V_{\max}$ for any 
$\alpha > \gamma \ceta$ by making a more sophisticated choice. 
Subsequently, the conclusion for sample complexity in Eq. \eqref{eqn:samples}, can be extended for any $\gamma <\frac{1}{\ceta}$.
% \red{$\ceta$ is not necessarily bigger than 1.}
Thus, the constant $\ceta$ in Assumption \ref{assu:me} determines the range of MDPs for which such gains can be achieved.
In our analysis of the proposed ME method, $\ceta \geq 1$ and indeed, we can achieve $\ceta = 1$ by trivially selecting 
$\Omega^{(t)} = \mS^{(t)} \times \mA^{(t)}$, which however, does not lead to any gain in efficiency. 
The key challenge is to find the right balance between small $\ceta$ with small $|\Omega^{(t)}|$ or $\C$. 
We address this next. 
% Thus, the constant $\ceta$ in Assumption \ref{assu:me} for the ME procedure is the ``error amplifier'' which reduces the 
% range of MDPs for which such gains can be achieved by $1/\ceta$.
% By definition, $\ceta \geq 1$ and a trivial approach for
% achieving $\ceta = 1$ is by selecting $\Omega^{(t)} = \mS^{(t)} \times \mA^{(t)}$, which however, does not lead to 
% efficiency gain in sample complexity. The key challenge is to find the right balance between small $\ceta$ 
% with as small $|\Omega^{(t)}|$ or $\C$. We address this next. 

%Theorem \ref{thm:generic} formally establishes the efficiency of our algorithm. The question left now is only whether we can design a Matrix Estimation subroutine satisfying Assumption \ref{assu:me} under proper conditions. To this end, we first develop some intuition through the special rank-$1$ case in Section \ref{sec:rank1}; we will generalize our approach in Section \ref{sec:rankr} to the generic rank-$r$ scenario and cases where $Q^*$ is only approximately rank-$r$. In Section \ref{sec:rankr}, we also discuss why existing ME approaches fail and potential directions for our algorithm to work on a wider range of $\gamma$.

\section{Matrix Estimation Methods Satisfying Assumption \ref{assu:me}}\label{sec:lowrankME}
%In this section, 
We introduce a matrix estimation method satisfying Assumption \ref{assu:me} which is required for the success of our RL algorithm as in Theorem \ref{thm:generic}. %We describe such a ME method for settings with low-rank or approximately low-rank $Q^*$. 
For the ease of illustration, we start with describing it for the rank-$1$ setting (Section \ref{sec:rank1}), then generalize it for $Q^*$ with 
generic rank $r \geq 1$ (Section \ref{sec:rankr}) and finally for the approximate rank-$r$ setting with full generality (Section \ref{sec:rankr_approx}). 

\subsection{Matrix Estimation for $Q^*$ with Rank $1$}
\label{sec:rank1}

Consider $Q^*$ with rank $1$. That is, there exist $f: \mS \to \RR$ and $g: \mA \to \RR$ so that $Q^*(s,a) = f(s) g(a)$ for all $(s, a) \in \mS \times \mA$. For the ease of exposition, %without loss of generality (since, we can add large positive constant to reward function), 
we assume $R(s,a)\in[R_{\min}, R_{\max}]$ with $ R_{\min} > 0$ for all $(s,a) \in \mS \times \mA$ in this warm-up only. Subsequently, $Q^*(s,a) \geq V_{\min} \triangleq \frac{R_{\min}}{1 - \gamma}, ~\forall(s,a)$.

\medskip
\noindent{\bf Matrix Estimation Algorithm.}
For $t \geq 1$, consider a discretization of state, action spaces, 
$\mS^{(t)} \subset \mS$, $\mA^{(t)} \subset \mA$. 
Let $Q^*(\mS^{(t)},\mA^{(t)})$ be the $|\mS^{(t)}|\times |\mA^{(t)}|$ matrix induced by restricting $Q^*$ to 
$\mS^{(t)} \times \mA^{(t)}$. Since $Q^*$ is rank $1$, it follows that $Q^*(\mS^{(t)},\mA^{(t)}) = F G^T$
where $F = [ f(s): s \in \mS^{(t)}] \in \mathbb{R}^{|\mS^{(t)}|}$ and $G = [g(a): a \in \mA^{(t)}]  \in \mathbb{R}^{|\mA^{(t)}|}$. 
Therefore, we can estimate $Q^*(\mS^{(t)},\mA^{(t)})$ by estimating $F, G$. 

Now we describe the selection of $\Omega^{(t)}$ such that $|\Omega^{(t)}| = |\mS^{(t)}| + |\mA^{(t)}| -1$. 
To that end, we first choose an {\em anchor} element $s^{\sharp} \in \mS^{(t)}$ and $a^\sharp \in \mA^{(t)}$. 
%such that $\hat{Q}^{(t)}(s^\sharp,a^\sharp) \neq 0$. 
Then, let 
$\Omega^{(t)} = \{ (s,a) \in \mS^{(t)} \times \mA^{(t)}: s = s^{\sharp} \text{ or } a = a^{\sharp} \}$.
With access to $\{\hat{Q}^{(t)}(s,a): ~(s,a) \in \Omega^{(t)}\}$, our ME method produces estimates for all $(s, a) \in \mS^{(t)} \times \mA^{(t)}$ as 
$ \bar{Q}^{(t)}(s,a) = \frac{\hat{Q}^{(t)}(s,a^\sharp) \hat{Q}^{(t)}(s^\sharp,a)}{\hat{Q}^{(t)}(s^\sharp,a^\sharp)}$. 

% As per Assumption \ref{assu:me}, for a given $\epsilon > 0$, we have access to estimate $\hat{Q}^{(t)}$
% so that $\max_{(s,a) \in \Omega^{(t)}} \big| \hat{Q}^{(t)}(s,a) - Q^*(s,a) \big| \leq \epsilon$.
% Using this, we 

\medskip
\noindent{\bf Satisfaction of Assumption \ref{assu:me}.} 
For the algorithm described above, we state the following proposition 
which verifies that Assumption \ref{assu:me} is satisfied with $\C = 1$ and $\ceta = 7 \frac{R_{\max}}{R_{\min}}$. 
\begin{prop}\label{prop:rank1}
For $\epsilon \leq \frac{1}{2} V_{\min}$, suppose that $\max_{(s,a) \in \Omega^{(t)}} \big| \hat{Q}^{(t)}(s,a) - Q^*(s,a) \big| \leq \epsilon$. Then the estimate produced by the above
ME algorithm satisfies 
\begin{align*}
    \max_{(s,a) \in \mS^{(t)}\times \mA^{(t)}} \big| \bar{Q}^{(t)}(s,a) - Q^*(s,a) \big| & \leq 7 \frac{R_{\max}}{R_{\min}} \epsilon.
\end{align*}
\end{prop}

Proposition \ref{prop:rank1} implies that when $Q^*$ is of rank $1$, our simple ME method described above satisfies 
$\big(1, 7 \frac{R_{\max}}{R_{\min}} \big)$-property for $\epsilon \leq \frac{1}{2}V_{\min}$. 
We remark that for any $c\in(0,1)$, one can show that the method fulfills $\big(1,\ceta \big)$-property with $\ceta = \frac{3+c}{1-c}\frac{R_{\max}}{R_{\min}}$ for all $\epsilon\leq cV_{\min}$.
By replacing Assumption \ref{assu:me} in Theorem \ref{thm:generic} with Proposition \ref{prop:rank1}, we obtain {convergence and sample complexity guarantees for the rank-$1$ setup} (cf. Theorem \ref{thm:rank1} stated in Appendix \ref{sec:proofrank1}). 
We refer interested readers to Appendix \ref{sec:proofrank1} for more details, including the proof of Proposition \ref{prop:rank1}.

\subsection{Matrix Estimation for $Q^*$ with Rank $r$}
\label{sec:rankr}

Based on the intuition developed in Section \ref{sec:rank1}, we consider a more general rank-$r$ setup. 
For notational convenience, given $Q: \mS \times \mA \to \RR$ and $\mS' \subset \mS, \mA' \subset \mA$, 
we let $Q(\mS', \mA')$ denote the $|\mS'| \times |\mA'|$ matrix $\big[ Q(s,a): ~(s,a) \in \mS' \times \mA' \big]$, 
whose entries are indexed by $(s,a) \in \mS' \times \mA'$. 
%When $\mS' = \{ s \}$, we write $Q(s, \mA') = Q(\{s\}, \mA' )$ with abuse of notation. 
%Given $M \in \RR^{m \times n}$, we let $M^{\dagger} \in \RR^{n \times m}$ denote the 
%Moore-Penrose pseudoinverse of $M$.

The central idea is the same as before: although $Q^*(\mS^{(t)},\mA^{(t)})\in\mathbb{R}^{m\times n}$ is an array of $mn$ real numbers, it has only $r (m+n-r)$ degrees of freedom with $r$-dimensional row and column spaces, when $\textrm{rank}(Q^*(\mS^{(t)},\mA^{(t)})) = r \leq \min\{ m, n \}$; as a result, one can successfully restore $Q^*(\mS^{(t)},\mA^{(t)})$ by exploring only $r$ entire rows and columns. There is, however, a small caveat that the $r$ rows and $r$ columns should be carefully chosen so that they are not degenerate, i.e., the $r$ rows span the entire row space of $Q^*(\mS^{(t)},\mA^{(t)})$ (the $r$ columns span the entire column space of $Q^*(\mS^{(t)},\mA^{(t)})$, respectively). Towards this end, we first define the notion of anchor states and actions.

\begin{defn}\label{defn:anchor_set}
(Anchor states and actions) 
A set of states $\mS^{\sharp} = \{s^\sharp_i\}_{i=1}^{R_s} \subset \mS$ and actions $\mA^{\sharp} = \{a^\sharp_i\}_{i=1}^{R_a} \subset\mA$ for some $R_s, R_a$ are called anchor states and actions for $Q^*$ if $\textrm{rank}~ Q^*(\mS^{\sharp}, \mA^{\sharp} ) = r$.
\end{defn}

That is, there are $r$ states in the set $\mS^\sharp$ such that $Q^*(s,\mA^\sharp), s\in \mS^\sharp$ are linearly 
independent. In other words, $\mS^\sharp$ contains states with sufficiently diverse performance on actions $\mA^\sharp$. 
Likewise, a similar interpretation holds for $\mA^\sharp$ if we look at the columns of $ Q^*(\mS^{\sharp}, \mA^{\sharp} )$. 
%Overall, anchor states/actions are intuitively sets of sufficiently diverse states/actions whose performance capture different aspects of the task reflected on $Q^*$. 

% (another version) That is, for any $s \in \mS$, $Q^*(s, \mA)$ can be expressed as a linear combination of the rows of $Q^*(\mS^{\sharp}, \mA)$ and likewise, $Q^*(\mS, a)$ can be expressed as a linear combination of the columns of $Q^*(\mS, \mA^{\sharp})$. Intuitively, $\mS^\sharp$ is a set of `representative' states that contain information about $Q^*(\mS,\mA)$; a similar interpretation is possible for $\mA^\sharp$. In other words, anchor states/actions are sets of sufficiently diverse states/actions whose performance capture different aspects of the task reflected on $Q^*$. 

% for any $s \in \mS$, $Q^*(s, \mA)$ can be expressed as a linear combination of the rows of $Q^*(\mS^{\sharp}, \mA)$ and likewise, $Q^*(\mS, a)$ can be expressed as a linear combination of the columns of $Q^*(\mS, \mA^{\sharp})$. Intuitively, $\mS^\sharp$ is a set of `representative' states that contain information about $Q^*(\mS,\mA)$; a similar interpretation is possible for $\mA^\sharp$. 

% In other words, anchor states/actions are sets of sufficiently diverse states/actions whose performance capture different aspects of the task reflected on $Q^*$. 

%\red{DG: I think the paragraph below should go to introduction or somewhere else?}
Indeed, $\mS^{\sharp}$ and $\mA^{\sharp}$ will be applied to construct our exploration sets 
and we want them to have small size. Finding only a few diverse states and actions is arguably easy 
in practical tasks --- in fact, for several stochastic control tasks experimented in Section~\ref{sec:empirical}, 
we simply pick a few states and actions that are far from each other in their respective metric spaces. We remark that 
assuming some ``anchor'' elements (i.e., elements having some special, relevant properties) is common 
in feature-based reinforcement learning~\cite{yang2019sample,duan2019state} or matrix factorization 
such as topic modeling~\cite{arora2012learning}.

\medskip
\noindent{\bf Matrix Estimation Algorithm.} 
We select anchor states $\mS^{\sharp} \subset \mS$, anchor actions $\mA^{\sharp} \subset \mA$ and fix
them throughout all iterations $1\leq t \leq T$. As before, we select appropriate $\beta^{(t)}$-nets $\mS^{(t)}$ and 
$\mA^{(t)}$ and augment them with the anchor states and actions: $\bar{\mS}^{(t)}\gets \mS^{(t)}\cup \mS^{\sharp}$ and 
$\bar{\mA}^{(t)}\gets \mA^{(t)}\cup \mA^{\sharp}$. For iteration $1\leq t\leq T$, we 
let $\Omega^{(t)} = \{ (s,a) \in \bar{\mS}^{(t)} \times \bar{\mA}^{(t)}: s \in \mS^{\sharp} \text{ or } a \in \mA^{\sharp} \}$ be the exploration set.

%We are now ready to describe our exploration set and ME method 
%which amount to exploring the entire rows corresponding to anchor states and the columns corresponding to anchor actions. At initialization (i.e., $t=0$), we select a set of anchor states $\mS^{\sharp} \subset \mS$ and a set of anchor actions $\mA^{\sharp} \subset \mA$, and fix $\mS^{\sharp}, \mA^{\sharp}$ through all $T$ iterations. Then, at each $t=1,\dots, T$:
%\begin{itemize}
   % \item Choice of anchor states/actions: pick a set of anchor states $\mS^{\sharp} \subset \mS$ and a set of anchor actions $\mA^{\sharp} \subset \mA$. It is possible to do this only once at the beginning and fix $\mS^{\sharp}, \mA^{\sharp}$ through all $T$ iterations.
%    \item Discretization: After obtaining the $\beta^{(t)}$-nets $\mS^{(t)}$ and $\mA^{(t)}$, augment them with the anchor states and actions: $\bar{\mS}^{(t)}\gets \mS^{(t)}\cup \mS^{\sharp}$ and $\bar{\mA}^{(t)}\gets \mA^{(t)}\cup \mA^{\sharp}$.
%    \item
%    Choice of exploration set $\Omega^{(t)}$:
    %At each $t =1, \ldots, T$, 
    %choose an anchor state $s^{\dagger} \in \mS^{(t)}$ and an anchor action $a^{\dagger} \in \mA^{(t)}$. 
    %Then 
 %   Let $\Omega^{(t)} = \{ (s,a) \in \bar{\mS}^{(t)} \times \bar{\mA}^{(t)}: s \in \mS^{\sharp} \text{ or } a \in \mA^{\sharp} \}$.
%    \item
Given $\hat{Q}^{(t)}(s, a)$ for $(s, a) \in \Omega^{(t)}$, our ME method produces
estimates for all $(s, a) \in \mS^{(t)} \times \mA^{(t)}$ as
    \begin{equation}\label{eqn:ME_rankr.0}
        \bar{Q}^{(t)}(s,a) = \hat{Q}^{(t)} (s, \mA^{\sharp} ) \big[ \hat{Q}^{(t)}(\mS^{\sharp}, \mA^{\sharp} ) \big]^{\dagger} \hat{Q}^{(t)}(\mS^{\sharp}, a )
    \end{equation}
where $\big[ \hat{Q}^{(t)}(\mS^{\sharp}, \mA^{\sharp} ) \big]^{\dagger}$ denotes the Moore-Penrose pseudoinverse of $\hat{Q}^{(t)}(\mS^{\sharp}, \mA^{\sharp} )$. 
With the choice of $R_s = R_a =  r$ (or a constant multiple of $r$), the size of $\Omega^{(t)}$ is at most $ r \big( | \bar{\mS}^{(t)}| + |\bar{\mA}^{(t)}| - r \big) \ll | \bar{\mS}^{(t)} | | \bar{\mA}^{(t)}|$.
%the entire (discretized) state-action space at iteration $t$, i.e., $|\Omega^{(t)}| = r \big( |\mS^{(t)}| + |\mA^{(t)}| +r \big) \ll | \mS^{(t)} | |\mA^{(t)}|$. 

\medskip
\noindent{\bf Satisfaction of Assumption \ref{assu:me}.} 
% {\bf Guarantees.} 
For given matrix $X \in \mathbb{R}^{m \times n}$, we denote by $\sigma_i(X)$ its $i$-th largest singular value, i.e., $ \sigma_1(X) \geq \sigma_2(X)\geq\dots\geq\sigma_{\min(m, n)}(X) \geq 0$. We state the following guarantee, which verifies that the matrix estimation algorithm described above satisfies Assumption \ref{assu:me}. 
%with $\C = 1$ and $\ceta = c(r; \mS^{\sharp}, \mA^{\sharp} )$ to be stated.
\begin{prop}[Simplified version of Proposition \ref{prop:rankr}]\label{prop:rankr_simplified}
    Let $\Omega^{(t)}$ and $\bar{Q}^{(t)}$ be as described above and let $|\mS^{\sharp}| = |\mA^{\sharp}| = r$.
    For any $\epsilon \leq \frac{1}{2r} \sigma_r\big( Q^*(\mS^{\sharp}, \mA^{\sharp} ) \big)$, if $\max_{(s,a) \in \Omega^{(t)}} \big| \hat{Q}^{(t)}(s,a) - Q^*(s,a) \big| \leq \epsilon$, then
    \begin{align*}
        \max_{(s,a) \in \mS^{(t)}\times \mA^{(t)}} \big| \bar{Q}^{(t)}(s,a) - Q^*(s,a) \big|		\leq c(r; \mS^{\sharp}, \mA^{\sharp} ) \epsilon
    \end{align*}
    where $c(r; \mS^{\sharp}, \mA^{\sharp} ) = \Big( 6\sqrt{2} \big( \frac{r}{\sigma_r ( Q^*(\mS^{\sharp}, \mA^{\sharp} ))} \big) + 2(1+\sqrt{5}) \big( \frac{r}{\sigma_r ( Q^*(\mS^{\sharp}, \mA^{\sharp} ))} \big)^2 \Big) V_{\max}$.
\end{prop}

Proposition \ref{prop:rankr_simplified} implies when $Q^*$ has rank $r$, our ME method
%per \eqref{eqn:ME_rankr.0} 
satisfies $\big(r, c(r; \mS^{\sharp}, \mA^{\sharp} ) \big)$-property 
for $\epsilon \leq \frac{1}{2r} \sigma_r\big( Q^*(\mS^{\sharp}, \mA^{\sharp} ) \big)$. 
Hence, we obtain Theorem \ref{thm:rankr} (stated in Appendix \ref{sec:proofthmrankr}) as a corollary of Theorem \ref{thm:generic}. That is, we obtain the desired convergence result and sample complexity  $\tilde{O}( \frac{1}{\epsilon^{d+2}} \cdot \log \frac{1}{\delta} )$ to achieve $\epsilon$ error with the output $Q^{(T)}$.
We remark that our algorithm and analysis also apply to low-rank $Q^*$ defined over discrete spaces (as also mentioned in Table \ref{tab:one} of the Introduction). 
We summarize results for (1) continuous $\mS$ and finite $\mA$; (2) finite $\mS$ and $\mA$ as corollaries in Appendix~\ref{subsec:coro_finite}.
We refer interested readers to Appendix \ref{sec:rankrappendix} for more details, including the proof of Proposition \ref{prop:rankr_simplified}.

\subsection{Matrix Estimation for $Q^*$ with Approximate Rank $r$} 
\label{sec:rankr_approx}

In Section \ref{sec:rankr}, we considered the setup where the underlying $Q^*$ has rank $r$. 
However, it may not be feasible to hope for exact low-rank structure in practice. Hence, it is desirable 
to seek methods that are reasonably robust to approximation error. We show that our ME method has such an 
appealing property. 
% Here, we provide a concise summary with full details deferred to Appendix~\ref{sec:apprx_r}.

Given $r>0$ as a parameter, let $Q^*_r$ denote the best rank-$r$ approximation of $Q^*$ in the $L^2$-sense 
so that $Q^*_r(s,a) = \sum_{i=1}^r \sigma_i f_i(s) g_i(a)$; cf. Theorem \ref{thm:representation} and its general 
version in Appendix~\ref{sec:proof_thm1}. Denote by $\zeta_r \triangleq \sup_{(s,a) \in \mS \times \mA} | Q^*_r(s,a) - Q^*(s,a) |$ 
the model bias due to the approximation. 
We introduce the notion of $r$-anchor states/actions that generalizes the notion of anchor states and actions in Definition \ref{defn:anchor_set}.

\begin{defn}\label{defn:r_anchor}
($r$-Anchor States and Actions) 
A set of states $\mS^{\sharp} = \{s^\sharp_i\}_{i=1}^{R_s} \subset \mS$ and actions $\mA^{\sharp} = \{a^\sharp_i\}_{i=1}^{R_a} \subset\mA$ 
for some $R_s, R_a$ are called $r$-anchor states and $r$-anchor actions for $Q^*$ if $\textrm{rank}~ Q^*_r(\mS^{\sharp}, \mA^{\sharp} ) = r$ 
for a positive integer $r$.
\end{defn}
It is easy to see that if $\mS^{\sharp}$ and $\mA^{\sharp}$ are $r$-anchor states/actions for $Q^*$, then they are $r'$-anchor states/actions 
for $Q^*$ for all $r' \leq r$.

\medskip
\noindent{\bf Matrix Estimation Algorithm.}
The algorithm remains the same as the exact rank-$r$ case, except that we select $\mS^{\sharp} \subset \mS$ and $\mA^{\sharp} \subset \mA$ 
to be $r$-anchor states and actions. 

\medskip
\noindent{\bf Theoretical Guarantee for Approximate Rank-$r$ Setup.} 
Previously, we imposed some regularity assumptions on $Q^*$, but the truncated function $Q^*_r$ is not guaranteed to inherit the regularity properties. Here, we additionally assume that (i) $\| Q^*_r \|_{\infty} \leq V_{\max}$ and (ii) $Q^*_r$ is $L$-Lipschitz, for the convenience of exposition. 

At a high level, our analysis is simple: for a given parameter $r > 0$, we treat $Q^*_r$ as the true function and repeat our analysis for  the rank-$r$ setup. Of course, there will be an additional bias, $Q^*_r(s,a) - Q^*(s,a)$, incurred by this substitution which requires careful tracking at each iteration. We formalize this argument in Proposition \ref{prop:apprx_r} and Theorem \ref{thm:apprx_r}. 

\begin{prop}\label{prop:apprx_r}
    Let $\Omega^{(t)}$ and $\bar{Q}^{(t)}$ be  as described above. Given a positive integer $r>0$, %let $\zeta_r := \sup_{(s,a) \in \mS \times \mA}\big| Q^*_r(s,a) - Q^*(s,a) \big|$ and 
    let $\mS^{\sharp}$ and $\mA^{\sharp}$ be some $r$-anchor states and actions for $Q^*$. 
    
    For any $\epsilon \leq \frac{1}{2\sqrt{|\mS^{\sharp}| |\mA^{\sharp}|}} \sigma_r\big( Q^*_r(\mS^{\sharp}, \mA^{\sharp} ) \big)$, if $\max_{(s,a) \in \Omega^{(t)}} \big| \hat{Q}^{(t)}(s,a) - Q^*_r(s,a) \big| \leq \epsilon$, then
    \[
        \max_{(s,a) \in \mS^{(t)}\times \mA^{(t)}} \big| \bar{Q}^{(t)}(s,a) - Q^*_r(s,a) \big| 
        	\leq \phi_c(r; \mS^{\sharp}, \mA^{\sharp})\epsilon,
    \]
    where
    \begin{equation}\label{eqn:phi_c}
        \phi_c(r; \mS^{\sharp}, \mA^{\sharp}) :=
        \Bigg( 6\sqrt{2} \bigg( \frac{\sqrt{|\mS^{\sharp}| |\mA^{\sharp}|}}{\sigma_r \big( Q^*_r(\mS^{\sharp}, \mA^{\sharp} )\big)} \bigg) 
		+ 2(1+\sqrt{5}) \bigg( \frac{\sqrt{|\mS^{\sharp}| |\mA^{\sharp}|}}{\sigma_r \big( Q^*_r(\mS^{\sharp}, \mA^{\sharp} )\big)} \bigg)^2 \Bigg) V_{\max}.
    \end{equation}
\end{prop}

With Proposition \ref{prop:apprx_r} at hand, we can obtain the following theorem as a Corollary of Theorem \ref{thm:generic} for the approximate rank-$r$ setup. 
% It is indeed an immediate corollary of Theorem \ref{thm:rankr} with $Q^*$ replaced by $Q^*_r$ and $c(r; \mS^{\sharp}, \mA^{\sharp})$ with $\phi_c(r; \mS^{\sharp}, \mA^{\sharp})$.
The theorem guarantees that when the model bias $\| Q^*_r - Q^* \|_{\infty}$ is sufficiently small, we obtain convergence and sample complexity results similar to the rank-$r$ setting with an additive error induced by the model bias.
%\red{Suppose that $|\mS^\sharp|$, $|\mA^\sharp|=O(r)$.}
%Denote by $\zeta_r \triangleq \sup_{(s,a) \in \mS %\times \mA} | Q^*_r(s,a) - Q^*(s,a) |$ the approximation error.
%For notation, let  $\zeta_r := \sup_{(s,a) \in \mS \times \mA} \big| Q^*_r(s,a) - Q^*(s,a) \big|$. 

\begin{thm}
\label{thm:apprx_r}

    Consider the approximate rank-$r$ setting in this section. Suppose that we run Algorithm \ref{alg:generic} with the ME method 
    described above.  %Let $T\geq 1$ and $\delta\in(0,1)$. 
    %Let $\zeta_r := \sup_{(s,a) \in \mS \times \mA} \big| Q^*_r(s,a) - Q^*(s,a) \big|$.
    Given a positive integer $r$,
    if $\gamma \leq \frac{1}{2 \phi_c(r; \mS^{\sharp}, \mA^{\sharp} )}$ and $\zeta_r \leq \min\Big\{\frac{\sigma_r\big( Q^*_r(\mS^{\sharp}, \mA^{\sharp} ) \big)}{2\sqrt{|\mS^{\sharp}||\mA^{\sharp}|} + (1 + \frac{1}{V_{\max}}) \sigma_r\big( Q^*_r(\mS^{\sharp}, \mA^{\sharp} ) \big) }, \frac{3}{2}V_{\max}\Big\}$, 
    then the following two statements are true.
    \begin{enumerate}
    \item
    For any $\delta > 0$, 
    with probability at least $1 - \delta$, 
    the following inequality holds for all $t=1, \ldots, T$ 
    \begin{align*}
        \sup_{ (s,a) \in \mS \times \mA } \big| Q^{(t)}(s,a) - Q^*(s,a) \big|
            &\leq
            \big(2 \phi_c(r; \mS^{\sharp}, \mA^{\sharp} ) \gamma \big)^t V_{\max}\\
            &\quad+ ( 1 + \phi_c(r; \mS^{\sharp}, \mA^{\sharp} )\gamma)\zeta_r \sum_{i=1}^{t} \big( \phi_c(r; \mS^{\sharp}, \mA^{\sharp} ) \gamma \big)^{i-1}
    \end{align*}
    by choosing algorithmic parameters $\beta^{(t)}, N^{(t)}$ appropriately.
    % $\beta^{(t)} = \frac{V_{\max}}{8L} \big( \frac{14 R_{\max}}{R_{\min}} \gamma \big)^t$, $N^{(t)} = 8 \big( \frac{14 R_{\max}}{R_{\min}} \gamma \big)^{-2(t-1)} \log\big(\frac{2 | \Omega^{(t)} | T}{\delta}\big)$ for $t = 1, \ldots, T$.
    
    \item
Further, given  $\epsilon > 0$, it suffices to set $T = \Theta( \log \frac{1}{\epsilon} )$ and use $\tilde{O}( \frac{1}{\epsilon^{d+2}} \cdot \log \frac{1}{\delta} )$ number of samples to achieve 
    \begin{equation*}
        \mathbb{P}\bigg( \sup_{(s,a) \in \mS \times \mA} \big| Q^{(T)}(s,a) - Q^*(s,a) \big| \leq \epsilon + \frac{ 1 + \gamma\phi_c(r; \mS^{\sharp}, \mA^{\sharp})}{ 1 - \gamma\phi_c(r; \mS^{\sharp}, \mA^{\sharp})}\zeta_r \bigg) 
            \geq 1 - \delta.
    \end{equation*}
    % \begin{equation*}
    %     \mathbb{P}\Big(\sup_{(s,a) \in \mS \times \mA} \big| Q^{(T)}(s,a) - Q^*(s,a) \big| \leq \epsilon + \frac{ 1 + \gamma\phi_c(r; \mS^{\sharp}, \mA^{\sharp})}{ 1 - \gamma\phi_c(r; \mS^{\sharp}, \mA^{\sharp})}\zeta_r\Big)\geq 1-\delta,
    % \end{equation*}
    % to achieve $\sup_{(s,a) \in \mS \times \mA} \big| Q^{(T)}(s,a) - Q^*(s,a) \big| \leq \epsilon + \frac{ 1 + \phi_c(r; \mS^{\sharp}, \mA^{\sharp})}{ 1 - \phi_c(r; \mS^{\sharp}, \mA^{\sharp})}\zeta_r$ with probability at least $1 - \delta$.
    \end{enumerate}
\end{thm}

% \red{a short discussion on the dependence of sample complexity on $|\mS^{\sharp}|, |\mA^{\sharp}|$ (and $\gamma$)}

Theorem \ref{thm:apprx_r} establishes the robustness of our method. When the approximation error $\zeta_r$ is not too large, with high probability, we obtain estimate of $Q^*$ that is within $\ell_\infty$ error $\epsilon + \frac{ 1 + \gamma\phi_c(r; \mS^{\sharp}, \mA^{\sharp})}{ 1 - \gamma\phi_c(r; \mS^{\sharp}, \mA^{\sharp})}\zeta_r$. Again, the algorithm only efficiently utilizes  $\tilde{O}( \frac{1}{\epsilon^{d+2}} \cdot \log \frac{1}{\delta} )$ number of samples\footnote{Here, the hidden constant in $\tilde{O}(\cdot)$ depends on $|\mS^{\sharp}|, |\mA^{\sharp}|$, {but it is not dependent on $\epsilon$.}}. Overall, the results on the approximate rank-$r$ setting justifies the soundness of our approach, from both theoretical and practical perspectives.

\medskip
\noindent{\bf {Reduction to Exact Rank-$r$ Case}.}
First of all, note that the exact rank-$r$ setting is a special case of approximate rank-$r$ case with $\zeta_s = 0$ for all $s \geq r$. 
Therefore, Theorem \ref{thm:apprx_r} applies to rank-$r$ setup discussed in Section \ref{sec:rankr}. With the choice of $|\mS^{\sharp}| = |\mA^{\sharp}| = r$, 
$\phi_c(r; \mS^{\sharp}, \mA^{\sharp})$ in \eqref{eqn:phi_c} reduces to $c(r; \mS^{\sharp}, \mA^{\sharp} )$ defined in Proposition \ref{prop:rankr_simplified}.
As a result, we can apply Theorem \ref{thm:apprx_r} to arrive at the same conclusion with Theorem \ref{thm:rankr} for exact rank-$r$ case stated 
and proved in Appendix \ref{sec:proofthmrankr}.
Taking that into account, Theorem \ref{thm:apprx_r} guarantees that when the model bias $\zeta_r$ is sufficiently small, we obtain 
convergence and sample complexity results similar to the exact rank-$r$ setting with an additive error, 
$\frac{ 1 + \gamma \phi_c(r; \mS^{\sharp}, \mA^{\sharp})}{ 1 - \gamma \phi_c(r; \mS^{\sharp}, \mA^{\sharp})}\zeta_r$, induced by the approximation bias.

\section{Empirical Evaluation}
\label{sec:empirical}

Besides theory, we empirically validate the effectiveness of our method on 5 continuous control tasks. The detailed setup can be found in Appendix~\ref{appendix:control setup}. In short, we first discretize the spaces into very fine grid and run standard value iteration to obtain a proxy of $Q^*$. The proxy has a very small approximate rank in all tasks; we hence use $r=10$ for our experiments. As mentioned, we simply select $r$ states and $r$ actions that are far from each other in their respective spaces as our anchor states and actions. For example, if the space is 2-dimensional, we uniformly divide it into $r$ squares and sample one from each square. Because of unavoidable discretization error, we also provide results on mean error, which might be a more reasonable measure in practice. While our proof requires small $\gamma$, we find the method to be generally applicable with large $\gamma$ in real tasks. Therefore, we use $\gamma = 0.9$ in all the tasks. Additional results on this aspect as well as results on all 5 tasks are provided in Appendix~\ref{appendix:control results}. 

\medskip
\noindent{\bf Improved Sample Complexity with ME.} First, we confirm that the sample complexity of our algorithm
improves with the use of ME. Our baseline is the same algorithm without the ME step, i.e., we explore and 
update all $(s,a)\in\mS^{(t)}\times \mA^{(t)}$, which is equivalent to performing a simulated value 
iteration on the discretized set. We illustrate the sample complexity for achieving different levels 
of $\ell_\infty$ error (Figure \ref{fig:ip_.9_linf_complexity}) and mean error (Figure \ref{fig:ip_.9_mean_complexity}). 
It is clear from the plots that our algorithm uses significantly less samples to achieve error at a similar level, compared to the baseline. This evidences that exploiting structure leads to improved efficiency. The same conclusion holds for the other tasks.
% \begin{figure}[H]
% % \centering
% \begin{subfigure}{0.2499\linewidth}
%     \includegraphics[width=\textwidth]{figures/ip_gamma.9_l_inf_complexity.pdf}
%     %\vspace{0.1in}
%     \caption{\small Sample Complexity}
%     %\vspace{-0.1in}
%     \label{fig:ip_.9_linf_complexity}
% \end{subfigure}
% % \hfill
% \hspace{-1ex}
% \begin{subfigure}{0.2499\linewidth}
%     \includegraphics[width=\textwidth]{figures/ip_gamma.9_mean_complexity.pdf}
%     \caption{\small Sample Complexity}
%     \label{fig:ip_.9_mean_complexity}
% \end{subfigure}
% \hspace{-1ex}
% \begin{subfigure}{0.2499\linewidth}
%     \includegraphics[width=\textwidth]{figures/ip_gamma.9_l_inf.pdf}
%     \caption{\small $\ell_{\infty}$ Errors}
%     \label{fig:ip_.9_linf_loss}
% \end{subfigure}
% \hspace{-1ex}
% \begin{subfigure}{0.2499\linewidth}
%     \includegraphics[width=\textwidth]{figures/ip_gamma.9_mean.pdf}
%     \caption{\small Mean Errors}
%     \label{fig:ip_.9_mean_loss}
% \end{subfigure}
% \vspace{-0.2cm}
% \caption{\small Empirical results on the Inverted Pendulum control task (results averaged across 5 runs).}
% \label{fig:ip_.9_main}
% \vspace{-0.5cm}
% \end{figure}
\begin{figure}[H]
% \vspace{-0.16in}
\centering
\subfigure[Sample Complexity]{
    \label{fig:ip_.9_linf_complexity}
    \includegraphics[width=0.24\textwidth]{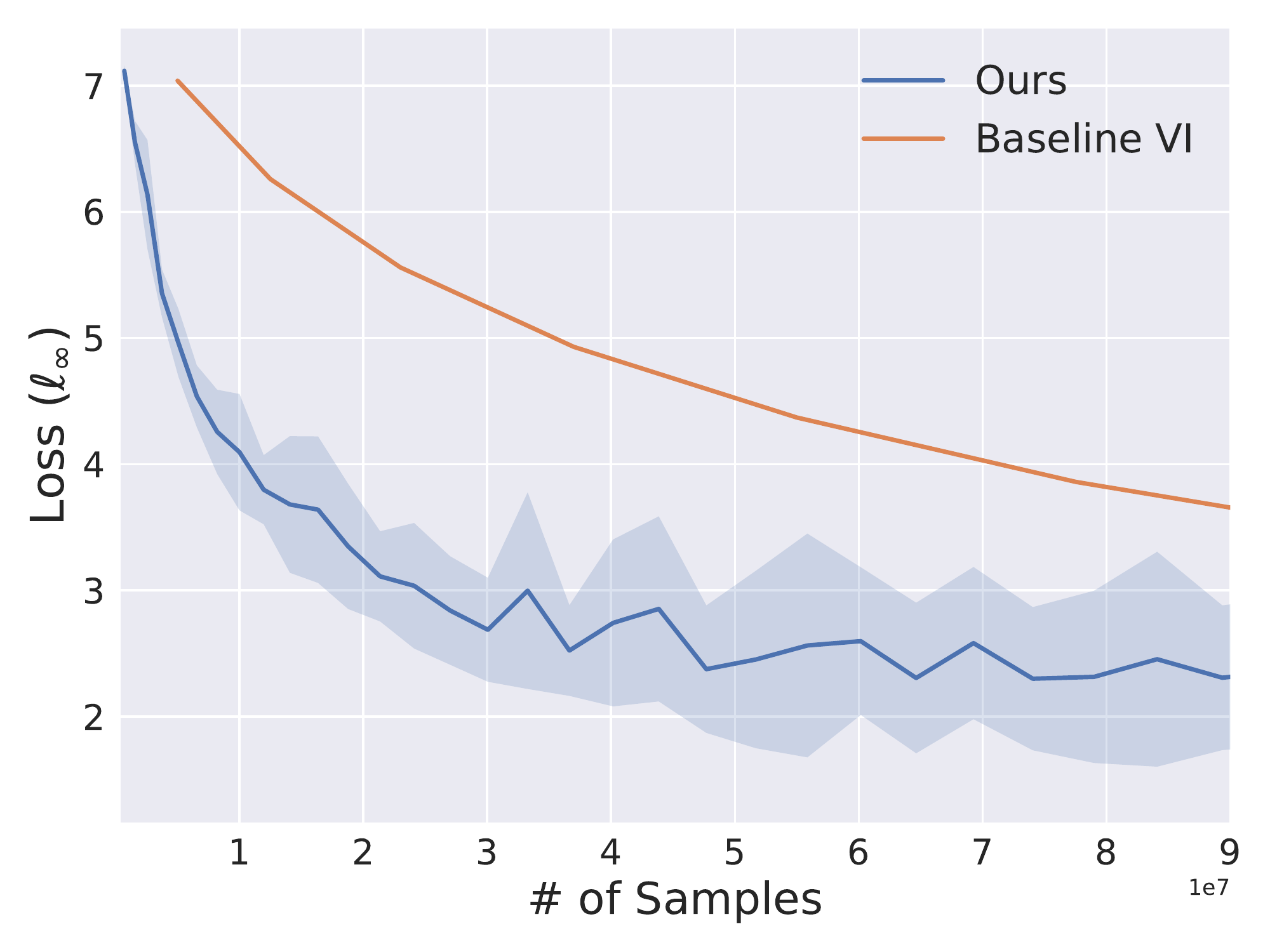}
}
\hspace{-2.26ex}
\subfigure[Sample Complexity]{
    \label{fig:ip_.9_mean_complexity}
    \includegraphics[width=0.24\textwidth]{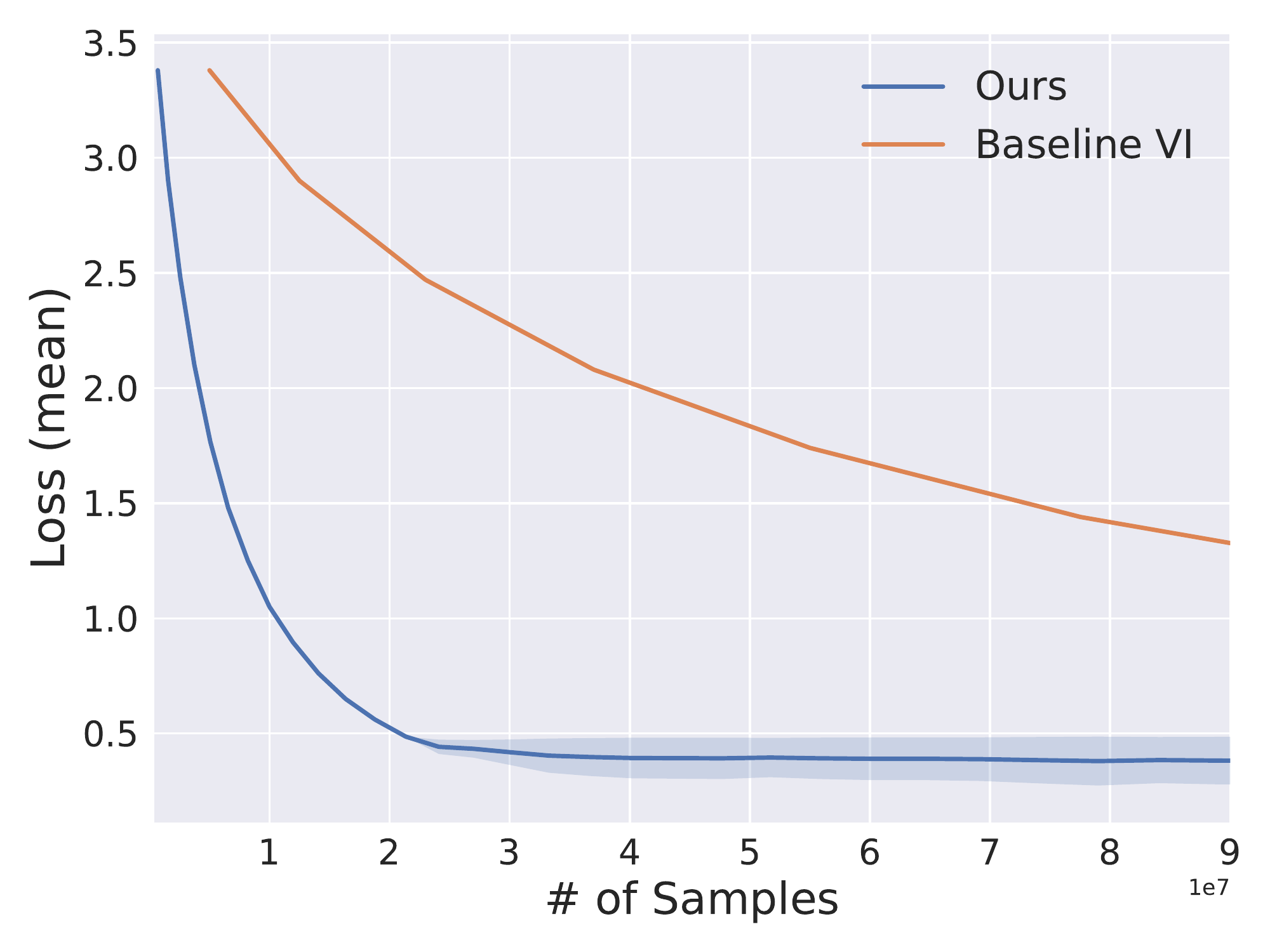}
}
\hspace{-2.26ex}
\subfigure[$\ell_{\infty}$ Errors]{
    \label{fig:ip_.9_linf_loss}
    \includegraphics[width=0.24\textwidth]{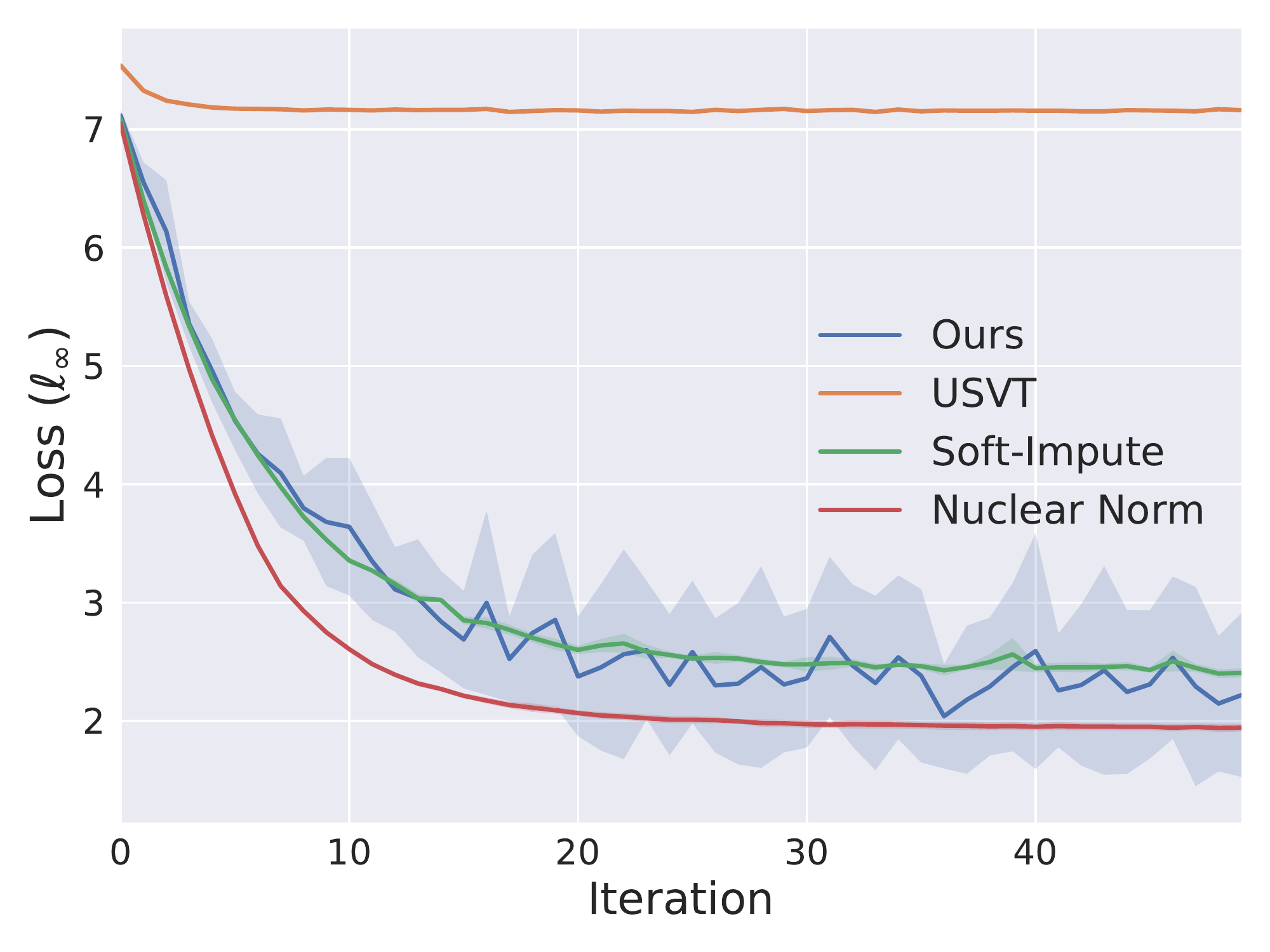}
}
\hspace{-2.26ex}
% \hfill
\subfigure[Mean Errors]{
    \label{fig:ip_.9_mean_loss}
    \includegraphics[width=0.24\textwidth]{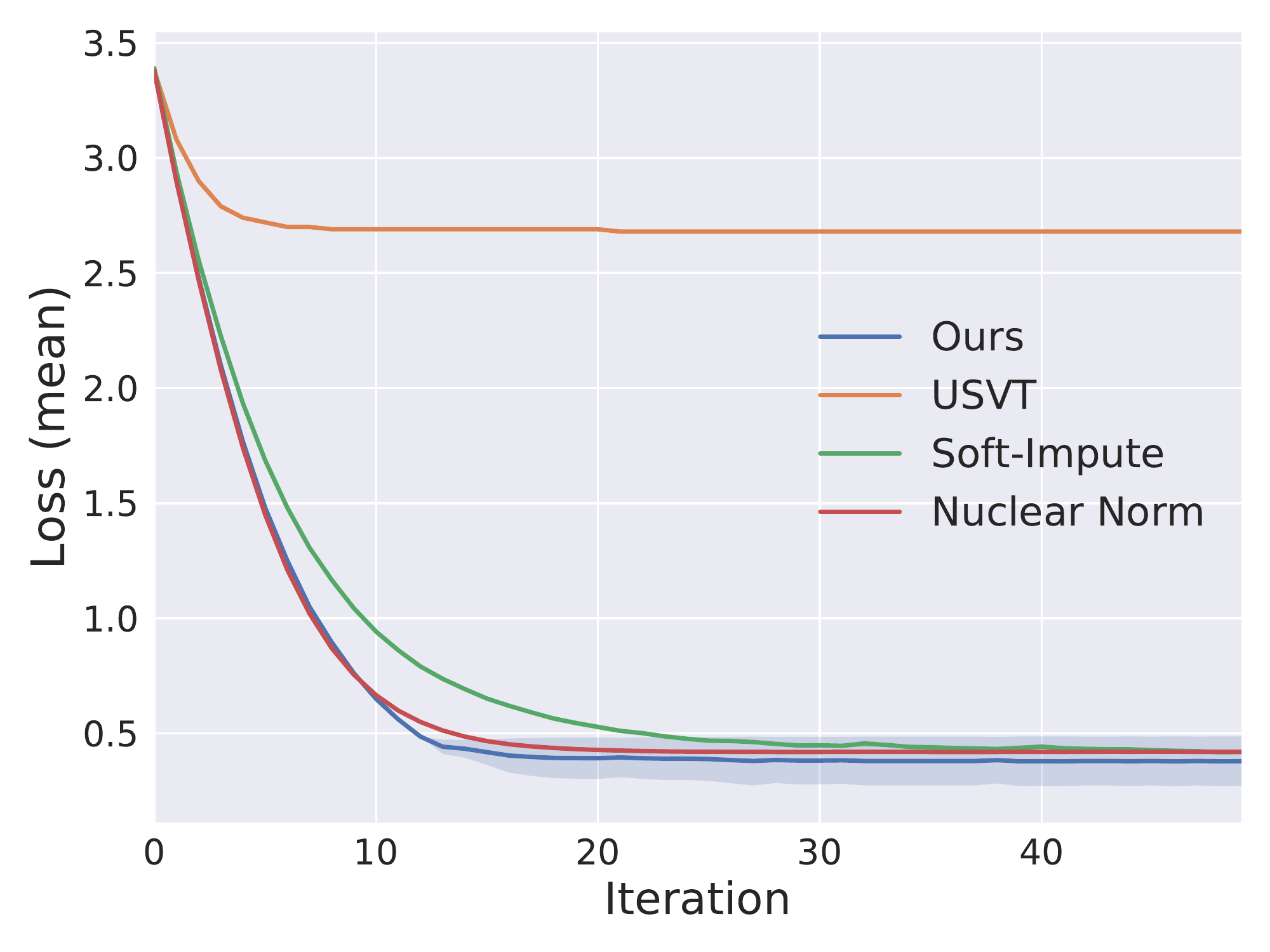}
}
% \hspace{-2.3ex}
% \subfigure[Sample complexity]{
%     \label{fig:ip_.9_mean_complexity}
%     \includegraphics[width=0.2475\textwidth]{figures/ip_gamma.9_mean_complexity.pdf}
% }
\vspace{-0.46cm}
\caption{Empirical results on the Inverted Pendulum control task (results averaged across 5 runs).
%We show both the $\ell_{\infty}$ error and the absolute mean error compared to the converged optimal $Q$ value matrix, as well as the sample complexity of our method compared with the baseline value iteration under our discretization scheme. 
}
\label{fig:ip_.9_main}
 \vspace{-0.15in}
\end{figure}

% {\bf Error Guarantees.} Next, we compare our ME method with others to validate its performance.
% While theoretically insufficient for RL applications, several established ME methods~\cite{candes2009exact,chatterjee2015matrix,mazumder2010spectral} often work well in practice. 
% We compare our method with them by feeding the same number of samples of size {\small $O ( \max\{\mS^{(t)}, \mA^{(t)}\})$}. As shown in Figure~\ref{fig:ip_.9_linf_loss} \& \ref{fig:ip_.9_mean_loss}, our method is competitive with those methods, both in terms of $\ell_\infty$ and mean errors. 
% Also, we note that our simple method is much more computationally efficient,
% compared to other optimization-based methods. 
% Overall, these results emphasize the practical value of our method beyond its theoretical soundness. 
% Lastly, we remark other ME methods also show promise in our experiments; 
% it is certainly a valuable open question to harmonize the established ME methods with low-rank RL.  

\medskip
\noindent{\bf Error Guarantees.} Next, we compare our ME method with others to validate its performance.
While theoretically insufficient for RL applications, some established ME methods~\cite{candes2009exact,chatterjee2015matrix,mazumder2010spectral} work well in practice. 
We compare the methods by feeding the same number of samples of size {\small $O ( \max\{\mS^{(t)}, \mA^{(t)}\})$}. As in Figure~\ref{fig:ip_.9_linf_loss} \& \ref{fig:ip_.9_mean_loss}, our method is competitive, both in $\ell_\infty$ \& mean errors. 
Also, we note that our simple method is computationally much more efficient,
compared to other methods based on optimization, etc. 
Overall, these results emphasize the practical value of our method beyond its theoretical soundness. 
Lastly, we remark other ME methods also show promise in our experiments; 
it is certainly a valuable open question to harmonize the established ME methods with low-rank RL.  
\begin{table}[H]
 \vspace{-0.07in}
\setlength{\tabcolsep}{7pt}
\caption{Performance metric for different stochastic control tasks using different ME methods. {A.D. stands for angular  deviation, T.G. stands for time-to-goal}; for both metrics, the smaller the better.}
\label{tab:metric}
 \vspace{-13pt}
% \small
\begin{center}
\resizebox{\textwidth}{!}{
\begin{tabular}{l c c c c c}
\toprule[1.3pt]
Method      &   Optimal   & USVT~\cite{chatterjee2015matrix}    & Soft-Impute~\cite{mazumder2010spectral}  & Nuclear Norm~\cite{candes2009exact}  &   Ours  \\
\midrule[1.3pt]
Inverted Pendulum (A.D.) & 1.6 {\scriptsize $\pm$ .0}   & 22.5 {\scriptsize $\pm$ 2.5}  & 5.3 {\scriptsize $\pm$ .6}    & 3.1 {\scriptsize $\pm$ .3}   & 3.4 {\scriptsize $\pm$ .7}   \\  [0.3ex]%\midrule  [0.2ex]
Mountain Car (T.G.)      & 75.0 {\scriptsize $\pm$ .3}  & 358.8 {\scriptsize $\pm$ 5.0} & 168.4 {\scriptsize $\pm$ 8.1} & 92.4 {\scriptsize $\pm$ 2.8} & 91.8 {\scriptsize $\pm$ 7.2} \\  [0.3ex]%\midrule
Double Integrator (T.G.) & 199.5 {\scriptsize $\pm$ .1} & 200.0 {\scriptsize $\pm$ .4}  & 199.9 {\scriptsize $\pm$ .3}  & 199.6 {\scriptsize $\pm$ .2} & 199.7 {\scriptsize $\pm$ .4} \\  [0.3ex]%\midrule
Cart-Pole (A.D.)         & 10.1 {\scriptsize $\pm$ .0}  & 19.2 {\scriptsize $\pm$ 1.0}  & 10.4 {\scriptsize $\pm$ .1}   & 10.2 {\scriptsize $\pm$ .1}  & 10.2 {\scriptsize $\pm$ .2}  \\  [0.3ex]%\midrule
Acrobot (A.D.)           & 2.4 {\scriptsize $\pm$ .0}   & 28.8 {\scriptsize $\pm$ 4.3}  & 9.1 {\scriptsize $\pm$ 1.2}   & 5.1 {\scriptsize $\pm$ .8}   & 6.2 {\scriptsize $\pm$ 1.0}  \\
\bottomrule[1.3pt]
\end{tabular}}
\vspace{-0.2in}
\end{center}
\end{table}

\medskip
\noindent{\bf Resulting Policy.} As a final proof of concept, we observe that the eventual performance of the 
policy obtained from the output $Q^{(T)}$ is very close to the policy obtained from $Q^*$ 
(cf. plots in Appendix~\ref{appendix:control results}). We summarize the results for standard 
performance metrics used in Table~\ref{tab:metric} . Obviously, our efficient method exhibits 
very competitive performance.

%In this work, 

\section{Discussion}
\label{sec:main_discussion}
{To facilitate a better understanding, we provide a short discussion on aspects of our ME method (see Appendix~\ref{app:subsec_me_discussion} for the full version).} For the success of our analysis, it is imperative for the ME subroutine to satisfy Assumption \ref{assu:me}. 
Despite the huge success of low-rank matrix completion, currently available analysis for 
the existing methods only provides a handle on the estimation error in a few limited class of norms including Frobenius norm. 
In particular, there are no satisfactory results so far that provide a control on the $\ell_{\infty}$ error of matrix estimation, 
to the best of our knowledge. As a result, we were not able to use existing ME methods and their analysis in this work.
However, we believe that what fails in the existing ME methods is their analysis, rather than the algorithms themselves.

Instead, we develop an alternative matrix estimation subroutine, which is simple, yet sufficiently powerful for our RL task, 
thereby enabling us to achieve the ultimate conclusion of improved sample complexity. One might doubt the efficacy of 
our proposed ME method. That concern is partly true, but indeed, there are two key factors that make our method work for 
the problem of our interest. First, we assume the existence of ``anchor'' states and actions, which contain all necessary information 
for the global recovery of $Q^*$. From a theoretical point of view, this assumption is related to the eigengap condition and 
the incoherence condition between eigenspace and the sampling operator, which are commonly assumed in existing ME literature. %\red{From a  practical perspective, this means the existence of ``diverse’' states and actions, which is the case in many real-world applications, as also demonstrated by the experiments}.
Second, we are not only passively fed with data, but can actively decide which data to collect. As a byproduct of active sampling, 
we get rid of the spurious log term that appears as a result of random sampling in sample complexity analysis for existing ME methods.

Our empirical results evidence that the proposed ME method is successful in the extremely sample deficient setting where 
$|\Omega^{(t)}| \asymp \max\{|\mS^{(t)}|, |\mA^{(t)}|\}$. However, it seems other existing ME methods based on convex programs 
also work similarly well, which cannot be explained with the current analysis. Therefore, it would be an exciting open question 
to harmonize existing ME methods and the low-rank RL task we consider in this work. This question might be tackled either by 
devising new proof techniques to obtain stronger error guarantees for existing ME methods or by improving our decoupled error analysis 
for RL iteration developed in this paper. We believe both directions are promising and it would be a valuable contribution 
to make progress in either direction.

\section{Conclusion}
We provide an efficient RL framework for continuous state and action spaces via proposing 
a new low-rank perspective. % via a low-rank viewpoint. 
With a novel ME method in the RL context, we demonstrate that our low-rank approach is both theoretically and practically appealing in designing sample efficient methods. 

There are several open questions such as devising better ME methods for the purpose of RL. {More broadly, this work introduces a low-rank framework for efficiently estimating functions, via appropriate matrix estimation methods at each iteration. Significant improvement in sample complexity is achieved in the RL setup, and the same recipe is likely to apply to a much broader extent in other areas in machine learning.} 
%Through our newly designed methods and experiments, practically and theoretically appealing. 
Overall, we believe that this work can serve as a starting point for many fruitful future research along this low-rank perspective.

\bibliographystyle{plain}
\bibliography{lrrl}

\newpage
\appendix
\appendixpage
\section{Proof of Theorem \ref{thm:representation}}
\label{sec:proof_thm1}
\newcommand{\X}{\mathbb{X}}
\newcommand{\mB}{\mathcal B}
\newcommand{\mH}{\mathcal{H}}
\newcommand{\bA}{A_{\mH}}

The proof of Theorem \ref{thm:representation} follows from the classical results in functional analysis. 
Interested reader may find lecture notes \cite{bhatia2009notes} and a classical textbook on the topic \cite{conway2019course} as excellent references. 
In this section, we present and prove a more general version of Theorem \ref{thm:representation} that is applicable to any compact metric spaces equipped with finite measures.

Let $\mS$ and $\mA$ be compact metric spaces, equipped with finite measures $\mu$, $\nu$, respectively.
We consider the space of square integrable functions 
\[
	L^2(\mS, \mu) = \bigg\{ f: \mS \to \RR \text{ such that } \| f \|_{L^2(\mS, \mu)} \equiv \Big( \int_{\mS} | f(s) |^2 d\mu(s) \Big)^{\frac{1}{2}} < \infty  \bigg\}
\]
and $L^2(\mA, \nu)$ defined similarly. $L^2(\mS, \mu)$ and $L^2(\mA, \nu)$ are known to be Hilbert spaces and in particular, they are separable because $\mS$ and $\mA$ are compact metric spaces. Therefore, they have countable bases. 

Recall that given any vector space $V$ over $\RR$, its dual space $V^*$ is defined as the set of all linear maps $\phi: V \to \RR$. It is known that the dual of $L^2(\mS, \mu)$ is isometrically isomorphic to $L^2(\mS, \mu)$, e.g., by the isomorphism $f \mapsto f^*$ where $f^*(f') = \langle f', f \rangle = \int_{\mS} f(s) f'(s) d\mu(s)$ (Appendix B, \cite{conway2019course}).

Given two Hilbert spaces, $\mH_1, \mH_2$, we let $\mH_1 \otimes \mH_2$ denote the tensor product of the two Hilbert spaces. The inner product in $\mH_1 \otimes \mH_2$ is defined on the basis elements so that $\langle \phi_1 \otimes \phi_2, \psi_1 \otimes \psi_2 \rangle_{\mH_1 \otimes \mH_2} = \langle \phi_1, \psi_1 \rangle_{\mH_1} \langle \phi_2, \psi_2 \rangle_{\mH_2}$ for all $\phi_1, \psi_1 \in \mH_1$ and $\phi_2, \psi_2 \in \mH_2$. Also, for every element $ \phi_1 \otimes \phi_2 \in \mH_1 \otimes \mH_2$, one can associate the rank-1 operator from $\mH_1^* \to \mH_2$ that maps a given $x^* \in \mH_1^*$ to $x^*(\phi_1) \phi_2$.

Our main theorem in this section is the following spectral theorem (singular value theorem) for $Q^*$. It is indeed a classical result from operator theory on Hilbert spaces. However, most results in existing literature cover the theory for self-adjoint operators and symmetric kernels. Although it is already implied by the classical results in a similar manner as eigenvalue decomposition extends to singular value decomposition, here we state our theorem and its proof for readers' convenience and future references.

\begin{thm}\label{thm:representation_general}
Let $(\mS, d_{\mS}, \mu)$ and $(\mA, d_{\mA}, \nu)$ be compact metric spaces equipped with finite measures. 
Let $Q^* \in L^2(\mS \times \mA, \mu \times \nu)$. 
If $Q^*$ is $L$-Lipschitz with respect to the product metric, then there exist a nonincreasing sequence $( \sigma_i \geq \RR_+: i \in \NN )$ with $\sum_{i=1}^{\infty} \sigma_i^2 < \infty$ and orthonormal bases $\{ f_i \in L^2(\mS, \mu): i \in \NN \}$ and $\{ g_i \in L^2(\mA, \nu): i \in \NN \}$ such that
\begin{equation}\label{eqn:Q_star} 
    Q^* = \sum_{i=1}^{\infty} \sigma_i f_i \otimes g_i.
\end{equation}
Subsequently, for any $\delta > 0$, there exists $r^*(\delta) \in \NN$ such that $\big\| \sum_{i=1}^{r} \sigma_i f_i \otimes g_i - Q^* \big\|_{L^2(\mS \times \mA, \mu \times \nu)}^2 \leq \delta)$ for all $r \geq r^*(\delta)$.
% $A$ is a Hilbert-Schmidt operator. There exists orthonormal sets 
% $\{f_i \in C_2([0,1]^d): i \in \mathbb{N}\}, \{g_i \in C_2([0,1]^d): i \in \mathbb{N} \}$ and 
% $\{\sigma_i \geq 0: i \in \mathbb{N}\}$ with $\sum_{i=1}^\infty \sigma_i^2 < \infty$ so that 
% for any $h \in C_2([0,1]^d)$, $Ah \in C_2([0,1]^d)$ has form
% $    Ah(\cdot)  = \sum_{i=1}^\infty \sigma_i \langle h, g_i \rangle f_i(\cdot)$. 
% Further, if $\sum_i \sigma_i \|f_i\|_\infty < \infty$,
% \begin{align}\label{eq:qstar}
%     Q^*(s, a) & = \sum_{i=1}^\infty \sigma_i f_i(s) g_i(a), ~~\forall s, a \in [0,1]^d.
% \end{align} Subsequently, for $\delta > 0$, there $\exists \:r(\delta)$ such that $ \int_{s} \int_a \big(Q^*(s, a) -\sum_{i=1}^{r(\delta)} \sigma_i f_i(s) g_i(a)\big)^2dsda  \leq \delta^2.$
% \begin{align}\label{eq:qstar.approx}
%     \int_{s} \int_a \Big(Q^*(s, a) -\sum_{i=1}^{r(\delta)} \sigma_i f_i(s) g_i(a)\Big)^2dsda & \leq \delta^2.
% \end{align}
\end{thm}

Note that we obtain the equality \eqref{eqn:Q_star} in the $L^2$ sense. However, since $Q^*$ is assumed Lipschitz continuous on a compact domain, this actually gives us a pointwise equality, i.e., $Q^*(s,a) = \sum_{i=1}^{\infty} \sigma_i f_i(s) g_i(a)$ for all $(s,a) \in \mS \times \mA$.

\begin{proof}
We define an integral kernel operator $\Kq = \Kq_{Q^*}: L^2( \mS, \mu ) \to L^2( \mA, \nu )$ induced by the kernel $Q^* \in L^2( \mS \times \mA, \mu \times \nu )$ so that
\[
	\Kq f(\cdot) = \int_{\mS} Q^*(s,\cdot) f(s) d\mu(s).
\]
Observe that $Q^*$ is a continuous function defined on a compact domain and hence bounded, viz., there exists $V_{\max} < \infty$ 
such that $|Q^*(s,a)| \leq V_{\max}$ for all $(s,a) \in \mS \times \mA$.

We present our proof in four parts. First, we verify that $\Kq$ is a compact operator from $L^2(\mS, \mu)$ to $L^2(\mA, \nu)$. Next, we argue $\Kq$ admits a generalized singular value decomposition with square summable singular values, based on the spectral theory of compact operators. Then we transfer the results for $\Kq \in L^2(\mS, \mu)^* \otimes L^2(\mA, \nu)$ to argue the spectral decomposition of $Q^* \in L^2(\mS, \mu) \otimes L^2(\mA, \nu)$. Lastly, we conclude the proof by discussing rank-$r$ approximation of $Q^*$.

\begin{enumerate}
\item
$\Kq$ is a compact operator from $L^2(\mS, \mu)$ to $L^2(\mA, \nu)$.

First, we argue that $\Kq$ is a bounded linear operator with $\| \Kq \| \leq V_{\max}^2 \mu(\mS) \nu(\mA)$.
Recall that $Q^*: \mS \times \mA \to \RR$ is Lipschitz continuous on a compact domain, hence, bounded, i.e., there exists $V_{\max} < \infty$ such that $| Q^*(s,a) | \leq V_{\max}$ for all $(s,a) \in \mS \times \mA$.
For any $f \in L^2( \mS, \mu)$,
\begin{align*}
	\| \Kq f \|_{L^2(\mA, \nu)}^2	&= \int_{\mA} \Kq f(a)^2 d\nu(a)\\
		&= \int_{\mA} \bigg( \int_{\mS} Q^*(s,a) f(s) d\mu(s) \bigg)^2 d\nu(a)\\
		&\leq \int_{\mA} \| Q^*(\cdot, a) \|_{L^2(\mS, \mu)}^2 \| f \|_{L^2(\mS, \mu)}^2 d\nu(a)		&&\because \text{Cauchy-Schwarz}\\
		&\leq V_{\max}^2 \mu(\mS) \nu(\mA) \| f \|_{L^2(\mS, \mu)}^2.					&&\because \| Q^*(\cdot, a) \|_{L^2(\mS)}^2 \leq V_{\max}^2 \mu(\mS)
\end{align*}

Next, we show that $\Kq: L^2(\mS, \mu) \to L^2(\mA, \nu)$ is indeed a compact operator. It suffices to show that for any bounded sequence $( f_n )_{n \geq 1}$ in $L^2(\mS, \mu)$, the sequence $( \Kq f_n )_{n \geq 1}$ contains a convergent subsequence. For this, we use (generalized) Arzel\`a-Ascoli theorem, which states that
if $( \Kq f_n )_{n \geq 1}$ is uniformly bounded and uniformly equicontinuous, then it contains a convergent subsequence. To that end, first
note that  $\| \Kq f_n \| \leq \|\Kq\| \|f_n\|$ and therefore, if $\|f_n\| \leq B$ for all $n \geq 1$, then $\| \Kq f_n \| \leq \| \Kq \|B$ for all $n \geq 1$. That is, the sequence $( \Kq f_n )_{n \geq 1}$ is uniformly bounded. Next, we can also verify that $( \Kq f_n )_{n \geq 1}$ is equicontinuous because for all $n \geq 1$,
\begin{align*}
	\big| \Kq f_n(a_1) - \Kq f_n(a_2) \big| &\leq \bigg| \int_{\mS} \big\{ Q^*(s,a_1) - Q^*(s,a_2) \big\} f_n(s) d\mu(s) \bigg|\\
		&\leq \| Q^*(s,a_1) - Q^*(s,a_2) \|_{L^2(\mS)} \| f_n \|_{L^2(\mS, \mu)}\\
		&\leq L \mu(\mS)^{\frac{1}{2}} d_{\mA}( a_1, a_2 ) \| f_n \|_{L^2(\mS, \mu)}\\
		&\leq BL\mu(\mS)^{\frac{1}{2}} d_{\mA}( a_1, a_2 ).
\end{align*}
In the second to last inequality, we used the fact that $Q^*$ is $L$-Lipschitz to show
\begin{align*}
	\| Q^*(s,a_1) - Q^*(s,a_2) \|_{L^2(\mS, \mu)} \| f_n \|_{L^2(\mS, \mu)}
		&= \bigg( \int_{\mS}  \big( Q^*(s,a_1) - Q^*(s,a_2) \big)^2 d\mu(S) \bigg)^{\frac{1}{2}}\\
		&\leq \bigg( \int_{\mS}  L^2 d_{\mA}( a_1, a_2 )^2 d\mu(S) \bigg)^{\frac{1}{2}}\\
		&= L \mu(\mS)^{\frac{1}{2}} d_{\mA}( a_1, a_2 ).
\end{align*}

\item
Spectral decomposition of $\Kq$.

\begin{itemize}
\item
First of all, we show that there exist orthonormal bases $\{ f_i \in L^2(\mS, \mu): i \in \NN \}$, $\{ g_i \in L^2(\mA, \nu): i \in \NN \}$ and singular values $\{ \sigma_i \geq 0: i \in \NN \}$ such that 
\begin{equation}\label{eqn:svd_A}
	\Kq = \sum_{i=1}^{\infty} \sigma_i f_i^* \otimes g_i.
\end{equation}

To see this, we consider the adjoint operator of $\Kq$, namely, $\Kq^*: L^2(\mA, \nu) \to L^2( \mS, \mu)$. %(Sec 6, Lecture 14 \cite{bhatia2009notes}). 
Since $\Kq: L^2(\mS, \mu) \to L^2(\mA, \nu)$ is compact, $\Kq^*$ is also compact. %(Sec 11, Lecture 20 \cite{bhatia2009notes}). 
Note that $\Kq^*\Kq$ is compact and self-adjoint. By the spectral theorem for compact self-adjoint operators, there exist $\{ \tau_i \in \RR: i \in \NN \}$ and an orthonormal basis $\{ f_i \in L^2( \mS, \mu ): i \in \NN \}$ such that $\Kq^*\Kq f_i = \tau_i f_i$ for all $i \in \NN$. We can observe that $\tau_i \geq 0$ for all $i$ because $\tau_i = \tau_i \langle f_i, f_i \rangle = \langle \Kq^*\Kq f_i, f_i \rangle = \| \Kq f_i \|_{L^2(\mS, \mu)}^2 \geq 0$. We let $I:= \{ i \in \NN: \tau_i > 0 \}$. 

Next, we observe that $\ker( \Kq^* \Kq ) = \ker( \Kq )$. Showing $\ker( \Kq^* \Kq ) \supseteq \ker( \Kq )$ is trivial. To show the other direction, let's suppose that $f \in \ker( \Kq^* \Kq )$. Then $\| \Kq f \|_{L^2(\mA, \nu)}^2 = \langle \Kq f, \Kq f, \rangle = \langle \Kq ^*\Kq f, f \rangle = 0$, which requires $\Kq f = 0$ and thus $f \in \ker(\Kq)$.

For $i \in I$, we let $g_i = \frac{1}{\sqrt{\tau_i}} \Kq f_i$. Then $\langle g_i, g_j \rangle = \frac{1}{\sqrt{\tau_i \tau_j}} \langle \Kq f_i, \Kq f_j \rangle = \frac{1}{\sqrt{\tau_i \tau_j}} \langle \Kq^*\Kq f_i, f_j \rangle = \delta_{ij} $, and hence, $\{g_i: i \in I\}$ consists of orthonormal vectors. We can augment $\{g_i: i \in I\}$  by adding appropriate vectors to make $\{g_i: i \in \NN\}$ an orthonormal basis of $L^2(\mA, \nu)$.

Every vector $\phi \in L^2(\mS, \mu)$ can be expanded as $\phi = \sum_{i=1}^{\infty} \langle \phi, f_i \rangle f_i$. Then we see that $\Kq \phi = \sum_{i=1}^{\infty} \langle \phi, f_i \rangle \Kq f_i = \sum_{i=1}^{\infty} \sqrt{\tau_i} \langle \phi, f_i \rangle g_i$. By letting $\sigma_i = \sqrt{\tau_i}$, we obtain \eqref{eqn:svd_A}.

\item
In addition, we show that $\sum_{i=1}^{\infty} \sigma_i^2 = \| Q^* \|_{L^2(\mS \times \mA, \mu \times \nu)}^2 < \infty$. The Hilbert-Schmidt norm of operator $\Kq$ is defined as 
$\| \Kq \|_{HS} = \tr(\Kq^*\Kq) = \sum_{i=1}^{\infty} \| \Kq f_i \|_{L^2(\mA, \nu)}^2  < \infty$. Note that $\| \Kq \|_{HS} = \sum_{i=1}^{\infty} \sigma_i^2$.

First, we observe that for each $i \in \NN$,
\begin{align*}
	\langle \Kq f_i, \Kq f_i \rangle_{L^2(\mA, \nu)}
		&= \int_{\mA} \bigg( \int_{\mS} Q^*(s,a) f_i(s) d\mu(s) \bigg)^2 d\nu(a)\\
		&= \int_{\mA} \big\langle Q^*(\cdot, a), f_i \big\rangle_{L^2(\mS, \mu)}^2 d\nu(a).
\end{align*}
We define a function $G(a) := \big\langle Q^*(\cdot, a), f_i \big\rangle_{L^2(\mS, \mu)}^2$. Recall that $Q^* \in L^2( \mS \times \mA, \mu \times \nu)$ and observe that $G$ is a nonnegative measurable function. Then we can use Tonelli's theorem to see that
\begin{align*}
	 \tr(K^*K) &= \sum_{i=1}^{\infty} \langle \Kq f_i, \Kq f_i \rangle_{L^2(\mA, \nu)}
	 	= \sum_{i=1}^{\infty}  \int_{\mA} \big\langle Q^*(\cdot, a), f_i \big\rangle_{L^2(\mS, \mu)}^2 d\nu(a)\\
		&= \int_{\mA} \sum_{i=1}^{\infty} \big\langle Q^*(\cdot, a), f_i \big\rangle_{L^2(\mS, \mu)}^2 d\nu(a)	\qquad~\because\text{Tonelli's theorem}\\
		&= \int_{\mA} \| Q^*(\cdot, a) \|_{L^2(\mS, \mu)}^2 d\nu(a).			\qquad\qquad\quad \because\text{the orthonormality of } \{f_i\}
\end{align*}
We have $\int_{\mA} \| Q^*(\cdot, a) \|_{L^2(\mS, \mu)}^2 d\nu(a) = \int_{\mA} \big( \int_{\mS} | Q^*(s, a) |^2 d\mu(s) \big) d\nu(a) = \| Q^* \|_{L^2(\mS \times \mA, \mu \times \nu)}^2$ by Fubini's theorem and therefore, $\sum_{i=1}^{\infty} \sigma_i^2 = \| Q^* \|_{L^2(\mS \times \mA, \mu \times \nu)}^2$.

\end{itemize}

\item
Spectral decomposition of $Q^*$.

Now we show that $Q^* = \sum_{i=1}^{\infty} \sigma_i f_i \otimes g_i$ for the same singular values $\{ \sigma_i \geq 0: i \in \NN \}$ and orthonormal bases $\{ f_i \in L^2(\mS, \mu): i \in \NN \}$, $\{ g_i \in L^2(\mA, \nu): i \in \NN \}$ as in \eqref{eqn:svd_A}. 

For that purpose, we assume that
\begin{equation}\label{eqn:svd_Q}
	Q^* = \sum_{i=1}^{\infty} \sigma_i f_i \otimes g_i + \varepsilon
\end{equation}
for some $\varepsilon \in L^2( \mS \times \mA, \mu \times \nu )$. For all $\phi \in L^2( \mS, \mu )$ and $\psi \in L^2( \mA, \nu )$, we have 
\begin{align*}
	\langle \psi, K \phi \rangle_{L^2(\mA, \nu)}
		&=\int_{\mA} \psi(a) \bigg( \int_{\mS} Q^*(s,a) \phi(s) d\mu(s) \bigg) d\nu(a)\\
		&= \int_{\mA} \psi(a) \bigg( \int_{\mS} \Big( \sum_{i=1}^{\infty} \sigma_i f_i(s) g_i(a) + \varepsilon(s,a) \Big) \phi(s) d\mu(s) \bigg) d\nu(a)\\
		&= \int_{\mA} \psi(a) \bigg\langle \sum_{i=1}^{\infty} \sigma_i f_i,  ~\phi \bigg\rangle_{L^2(\mS, \mu)} g_i(a) d\nu(a)
			+ \int_{\mA} \psi(a) \bigg( \int_{\mS} \varepsilon(s,a) \phi(s) d\mu(s) \bigg) d\nu(a).
\end{align*}
When $\phi = f_i$ and $\psi = g_j$, we have $\langle g_j, \Kq f_i \rangle_{L^2(\mA, \nu)} = \sigma_i \delta_{ij}$. By Fubini's theorem,
\begin{align}
	 \sigma_i \delta_{ij}
	 	&= \sigma_i \langle g_j, g_i \rangle + \int_{\mS \times \mA} \varepsilon(s,a) f_i(s) g_j(a) d (\mu \times \nu)(s \times a)  \nonumber\\
		&= \sigma_i \delta_{ij} + \langle \varepsilon, f_i \otimes g_j \rangle_{L^2(\mS \times \mA, \mu \times \nu)}.               \label{eqn:svd_Q.1}
\end{align}
In order to satisfy \eqref{eqn:svd_Q.1}, we must have $\langle \varepsilon, f_i \otimes g_j \rangle_{L^2(\mS \times \mA, \mu \times \nu)} = 0$ for all $(i,j) \in \NN^2$.

It is known that $L^2(\mS \times \mA, \mu \times \nu)$ is isomorphic to $L^2( \mS, \mu ) \otimes L^2( \mA, \nu)$ and $\{ f_i \otimes g_j: (i,j) \in \NN^2 \}$ constitutes an orthonormal basis of $L^2( \mS, \mu ) \otimes L^2( \mA, \nu)$. Therefore, $\varepsilon = 0$ and $Q^* = \sum_{i=1}^{\infty} \sigma_i f_i \otimes g_i$.

\item
Best rank-$r$ approximation of $Q^*$.

Without loss of generality, we may assume $\sigma_1 \geq \sigma_2 \geq \dots \geq 0$, i.e., the singular values are sorted in descending order. For any finite $r \in \NN$, let $Q^*_r = \sum_{i=1}^r \sigma_i f_i \otimes g_i$. 

Then, 
\begin{align*}
	\big\| Q^* - Q^*_r \big\|_{L^2(\mS \times \mA, \mu \times \nu)}^2
		&= \bigg\| \sum_{i=r+1}^{\infty} \sigma_i f_i \otimes g_i \bigg\|_{L^2(\mS \times \mA, \mu \times \nu)}^2\\
		&= \sum_{i, j=r+1}^{\infty} \sigma_i \sigma_j \big\langle f_i \otimes g_i, f_j \otimes g_j \big\rangle_{L^2(\mS \times \mA, \mu \times \nu)}\\
		%&= \sum_{i, j=r+1}^{\infty} \sigma_i \sigma_j \big\langle f_i, f_j \big\rangle_{L^2(\mS, \mu)}	\big\langle g_i, g_j \big\rangle_{L^2(\mA, \nu)}
		&= \sum_{i=r+1}^{\infty} \sigma_i^2
\end{align*}
where we have used the orthonormality of $\{ f_i \}$ and $\{ g_i \}$.

We conclude the proof with two final remarks:
\begin{itemize}
	\item
	Among all rank-$r$ functions of the form $\sum_{i=1}^r \lambda_i \phi_i \otimes \psi_i$ for some $\phi_i \in L^2(\mS, \mu)$, $\psi_i \in L^2(\mA, \nu)$, $Q^*_r$ is the ``best'' rank-$r$ approximation of $Q^*$ in the $L^2(\mS \times \mA, \mu \times \nu)$ sense.

\item
	Since $\sum_{i=1}^\infty \sigma_i^2 < \infty$, for any $\delta > 0$, there exists $r = r(\delta)$ so that $\sum_{i=r+1}^\infty \sigma_i^2 < \delta$. That is, we can approximate $Q^*_r$ arbitrarily well with a sufficiently large, yet still finite, rank $r$.
\end{itemize}

\end{enumerate}
This completes the proof of Theorem \ref{thm:representation_general}.
\end{proof}

\section{Proof of Theorem \ref{thm:generic}}\label{sec:proofthmgeneric}

\subsection{Helper Lemma: Error Bound for Lookahead Subroutine}
This section is devoted to the proof of Theorem \ref{thm:generic}. To this end, we first need to understand the error guarantees for the 
lookahead (exploration) subroutine based on the current oracle $V^{(t-1)}$, cf. Eq.~(\ref{eq:generic:lookahead.inside}) and Line 8 of Algorithm \ref{alg:generic}.
This is summarized in the following lemma. 
% With some abuse of notation, suppose that we have a current value oracle $\hat{V}$ and that an estimate $\hat{Q}(s,a)$ 
% {\bf Notation.}
% \begin{itemize}
%     \item $Q^{(t)}$ and $V^{(t)}$, the estimates after the full iteration $t$.
%     \item $Q^{(t)}_{ME}$ and $V^{(t)}_{ME}$, the estimates after the ME step at iteration $t$.
%     \item $Q^{(t)}_{LA}$ and $V^{(t)}_{LA}$, the estimates obtained from the tree search oracle at iteration $t$.
% \end{itemize}
% \subsection{One-step Lookahead Tree Search}
% Consider the one-step sparse sampling oracle. Let $C$ be the number of children sampled.

\begin{lem}
\label{lem:one_step_lookahead}
    Suppose that we have access to a value oracle $V: \mS \to \mathbb{R}$ such that
    \begin{equation*}
        \sup_{s \in \mS} \big| V(s) - V^*(s) \big| \leq B.
    \end{equation*}
    Given $(s,a) \in \mS \times \mA$, let $s'_1, \ldots, s'_N$ be the next states of $(s,a)$ independently drawn from the generative model and let
    $\hat{Q}(s,a) = R(s,a) + \gamma\cdot\frac{1}{N}\sum_{i=1}^{N} V(s'_i)$.
    Then for any $\delta > 0$, 
        \begin{equation*}
            |\hat{Q}(s,a) - Q^*(s,a)| \leq \gamma \left( B + \sqrt{ \frac{2 V_{\max}^2}{N} \log\bigg(\frac{2}{\delta}\bigg)} \right)
        \end{equation*}
     with probability at least $1-\delta$.
\end{lem}

\begin{proof}
    Note that $Q^*(s,a) = R(s,a) + \gamma \E_{s'\sim P_{s,a}}[V^*(s')]$ by definition of $Q^*$ and $V^*$ (cf. Bellman equation). 
    It follows that
    \begin{align}
        |\hat{Q}(s,a) - Q^*(s,a)| 
            & = \gamma\bigg|\frac{1}{N}\sum_{i=1}^N V (s_i') - \E_{s'\sim P_{s,a}}[V^*(s')]  \bigg| \nonumber \\
		    &\leq \gamma  \left| \frac{1}{N} \sum_{i=1}^N V (s_i')- \frac{1}{N} \sum_{i=1}^N V^* (s_i') \right|
		        + \gamma \left| \frac{1}{N} \sum_{i=1}^N {V}^* (s_i') - \mathbb{E}_{s' \sim P_{sa}}\left[{V}^*(s')\right] \right| \nonumber \\
		    &= \frac{\gamma}{N} \sum_{i=1}^N \big| V (s_i') - V^* (s_i') \big|
		        + \gamma \left| \frac{1}{N} \sum_{i=1}^N {V}^* (s_i') - \mathbb{E}_{s' \sim P_{sa}}\left[{V}^*(s')\right] \right|.
		\label{eq:lemma_tree_proof_1}
    \end{align}
By assumption, the first term in Eq.~(\ref{eq:lemma_tree_proof_1}) is bounded by $\gamma B$. Meanwhile, since $| V^*(s') | \leq V_{\max}$, 
we can apply Hoeffding's inequality to control the second term. Specifically, for any $t > 0$,
\begin{equation*}
    \Pr\left( \frac{1}{N} \sum_{i=1}^N {V}^* (s_i') - \mathbb{E}_{s' \sim P_{sa}}\left[{V}^*(s')\right] > t \right)
        \leq \exp \bigg( - \frac{ N t^2 }{2 V_{\max}^2} \bigg).
\end{equation*}
Solving $ \delta = 2 \exp \Big( - \frac{N t^2}{2 V_{\max}^2} \Big)$ for $t$ yields $t = \sqrt{ \frac{2 V_{\max}^2}{N} \log\big(\frac{2}{\delta}\big)}$ 
and this completes the proof.
\end{proof}

% \subsection{Matrix Estimation Guarantee}
% We make the following assumption regarding recovering a true matrix $M^*$ of rank $r$, from a subset $\Omega$ of noisy observations.

% \begin{assumption}
% Consider recovering a true $r$-rank matrix $M^*\in\mathbb{R}^n \times \mathbb{R}^n$.
% Suppose that we are given $|\Omega| = \kappa(n, r)$ noisy observations , $\hat{M}_\Omega$, such that
% \begin{equation}
%     |M^*(i,j) - \hat{M}(i,j)|\leq \epsilon,\quad\forall\:(i,j)\in\Omega. 
% \end{equation}
% Then, there exists a Matrix Estimation oracle, $ME(\hat{M}_\Omega)$ such that with probability $1- \delta\big(poly(n),poly(1/\epsilon),r\big)$ the recovered matrix $\bar{M}= ME(\hat{M}_\Omega)$ satisfies
% \begin{equation}
% \max_{(i,j)\in[n]\times[n]}||\bar{M} - M^*||_\infty \leq \eta \cdot \epsilon.
% \end{equation}
% \end{assumption}

\subsection{Proof of Theorem \ref{thm:generic}}

\begin{proof}[Proof of Theorem \ref{thm:generic}]
We prove the first statement by mathematical induction.
For $t=0$, $Q^{(0)}(s,a) \equiv 0$ and thus $| Q^{(0)}(s,a) - Q^*(s,a) | \leq V_{\max}$ for all $(s,a)$.
Next, we want to show that for $t=1, \ldots, T$,
\begin{equation}\label{eqn:generic_induction}
    \sup_{(s,a) \in \mS \times \mA} \big| Q^{(t)}(s,a) - Q^*(s,a) \big| \leq \rho \sup_{(s,a) \in \mS \times \mA} \big| Q^{(t-1)}(s,a) - Q^*(s,a) \big|.
\end{equation}
Fix $t$ and suppose that $\sup_{(s,a) \in \mS \times \mA}\big| Q^{(t-1)}(s,a) - Q^*(s,a) \big| \leq B^{(t-1)}$. Note that this implies  
$\sup_{s \in \mS} \big| V^{(t-1)}(s) - V^*(s) \big| \leq B^{(t-1)}$ because $Q^{(t-1)}, Q^*$ are continuous and $\mA$ is compact \footnote{For each $s \in \mS$, there exist $a^{(t-1)}(s), a^*(s) \in \mA$ such that $V^{(t-1)}(s) = Q^{(t-1)}(s,a^{(t-1)}(s))$ and $V^*(s) = Q^*(s,a^*(s))$. If $V^{(t-1)}(s) \geq V^*(s)$, then $V^{(t-1)}(s) - V^*(s) = Q^{(t-1)}(s,a^{(t-1)}(s)) - Q^*(s,a^*(s)) \leq Q^{(t-1)}(s,a^{(t-1)}(s)) - Q^*(s,a^{(t-1)}(s))$. If $V^{(t-1)}(s) < V^*(s)$, then $V^*(s) - V^{(t-1)}(s) = Q^*(s,a^*(s)) - Q^{(t-1)}(s,a^{(t-1)}(s)) \leq Q^*(s,a^*(s)) - Q^{(t-1)}(s,a^*(s))$. Therefore, $| V^{(t-1)}(s) - V^*(s) | \leq \max_{a \in \{a^{(t-1)}(s), a^*(s)\}}\big\{ Q^{(t-1)}(s,a) - Q^*(s,a) \big\}$.}.
To prove the inequality in Eq. \eqref{eqn:generic_induction}, we backtrack the updating steps in Algorithm \ref{alg:generic}. 

For each $s \in \mS$ and $a \in \mA$, let $\hat{s}^{(t)} \in \arg\min_{s' \in \mS^{(t)}} \| s' - s \|_2$ and $\hat{a}^{(t)} \in \arg\min_{a' \in \mA^{(t)}} \| a' - a \|_2$. Since $\mS^{(t)}$ is a $\beta^{(t)}$-net of $\mS$, $\| \hat{s}^{(t)} - s \| \leq \beta^{(t)}$. Likewise, $\| \hat{a}^{(t)} - a \| \leq \beta^{(t)}$. As $Q^{(t)}(s,a) = \bar{Q}^{(t)}(\hat{s}^{(t)}, \hat{a}^{(t)})$ and $Q^*$ is $L$-Lipschitz,
\begin{align*}
    \big| Q^{(t)}(s,a) - Q^*(s,a) \big|
        &= \big| \bar{Q}^{(t)}(\hat{s}^{(t)},\hat{a}^{(t)}) - Q^*(s,a) \big|\\
        &= \big| \bar{Q}^{(t)}(\hat{s}^{(t)},\hat{a}^{(t)}) - Q^*(\hat{s}^{(t)},\hat{a}^{(t)}) \big| + \big| Q^*(\hat{s}^{(t)},\hat{a}^{(t)}) - Q^*(s,a) \big|\\
        &\leq \big| \bar{Q}^{(t)}(\hat{s}^{(t)},\hat{a}^{(t)}) - Q^*(\hat{s}^{(t)},\hat{a}^{(t)}) \big| + 2 L \beta^{(t)}.
\end{align*}
Therefore, we obtain the following upper bound for Step 4 (interpolation):
\begin{equation}\label{eqn:generic_interpol}
    \sup_{ (s,a) \in \mS \times \mA } \big| Q^{(t)}(s,a) - Q^*(s,a) \big|
        \leq \max_{(s,a) \in \mS^{(t)} \times \mA^{(t)}} \big| \bar{Q}^{(t)}(s, a) - Q^*(s, a) \big| 
            + 2 L \beta^{(t)}.
\end{equation}

By Assumption \ref{assu:me}, we have the following upper bound for Step 3 (matrix estimation):
\begin{equation}\label{eqn:generic_me}
    \max_{(s,a) \in \mS^{(t)} \times \mA^{(t)}} \big| \bar{Q}^{(t)}(s, a) - Q^*(s, a) \big| 
        \leq \ceta \max_{(s,a) \in \Omega^{(t)} } \big| \hat{Q}^{(t)}(s, a) - Q^*(s, a) \big|.
\end{equation}

Lastly, applying Lemma \ref{lem:one_step_lookahead} and taking union bound over $(s,a) \in \Omega^{(t)}$, we can show that
\begin{equation}\label{eqn:generic_mcla}
    \max_{(s,a) \in \Omega^{(t)} } \big| \hat{Q}^{(t)}(s, a) - Q^*(s, a) \big|
        \leq \gamma \left( B^{(t-1)} + \sqrt{ \frac{2 V_{\max}^2}{N^{(t)}} \log\bigg(\frac{2 | \Omega^{(t)} | T}{\delta}\bigg)} \right)
\end{equation}
with probability at least $1-\frac{\delta}{T}$.

Combining Eqs. \eqref{eqn:generic_interpol}, \eqref{eqn:generic_me}, \eqref{eqn:generic_mcla} yields
\[
    \sup_{ (s,a) \in \mS \times \mA } \big| Q^{(t)}(s,a) - Q^*(s,a) \big|
        \leq B^{(t)}
\]
with probability at least $1-\frac{\delta}{T}$ where
\[
    B^{(t)}
        = \gamma \ceta  \left( B^{(t-1)} + \sqrt{ \frac{2 V_{\max}^2}{N^{(t)}} \log\bigg(\frac{2 | \Omega^{(t)} | T}{\delta}\bigg)} \right)
            + 2 L \beta^{(t)}.
\]
By Assumption \ref{assu:me}, this requires at most
$|\Omega^{(t)}| = \C \big( |\mS^{(t)}| + |\mA^{(t)}| \big)$. Moreover, for each $1\leq t \leq T$, if we choose $\beta^{(t)} = \frac{V_{\max}}{8L} ( 2 \gamma \ceta)^t$ 
and 
\begin{equation}\label{eqn:Nt}
    N^{(t)} = \frac{ 8 }{ ( 2 \gamma \ceta )^{2(t-1)} } \log\bigg(\frac{2 | \Omega^{(t)} | T}{\delta}\bigg),
\end{equation}
then $B^{(t-1)} \leq (2 \gamma \ceta)^{t-1} V_{\max}$ implies that $B^{(t)} \leq (2 \gamma \ceta)^t V_{\max}$ with probability at least $1 - \frac{\delta}{T}$. 

At the beginning, we observed $| Q^{(0)}(s,a) - Q^*(s,a) | \leq V_{\max}$ for all $(s,a)$, i.e., $B^{(0)} \leq V_{\max}$.
By taking the union bound over $t = 1, \ldots, T$, 
\[
    \sup_{ (s,a) \in \mS \times \mA } \big| Q^{(t)}(s,a) - Q^*(s,a) \big|
        \leq (2 \gamma \ceta)^t V_{\max}, \quad \forall t=1, \ldots, T
\]
with probability at least $1 - \delta$.

\paragraph{Sample complexity.}
If $\gamma < \frac{1}{2 \ceta}$, then $2 \gamma \ceta < 1$. Let $\Te = \Big\lceil \frac{ \log \big( \frac{V_{\max}}{\epsilon} \big)}{ \log \big( \frac{1}{2 \gamma \ceta } \big)} \Big\rceil$ and observe that $(2 \gamma \ceta) \epsilon \leq (2\gamma \ceta)^{\Te} V_{\max} \leq \epsilon$. 
For each $t, 1\leq t \leq T$, we query 
$\hat{Q}^{(t)}(s,a)$ for $(s,a) \in \Omega^{(t)}$, each of which requires exploring 
$N^{(t)}$ samples. Therefore, the total sample complexity of Algorithm \ref{alg:generic} with 
$T=\Te$ is $\sum_{t=1}^{\Te} \big| \Omega^{(t)} \big| N^{(t)}$.

By standard argument on covering number, we can see that $|\cS^{(t)}|, |\cA^{(t)}| \leq C' \big( \frac{1}{\beta^{(t)}}\big)^d = C'\big( \frac{8L}{V_{\max}} \big)^d \big( 2\gamma \ceta \big)^{-dt}$ for some absolute constant $C' > 0$. This is an increasing function of $t$ and hence, $| \Omega^{(t)}| = \C\big( |\mS^{(t)}| + |\mA^{(t)}| \big)$ and 
$N^{(t)}$ as described in Eq. \eqref{eqn:Nt} are also increasing with respect to $t$. 

Observe that $\beta^{(\Te)} =  \frac{V_{\max}}{8L} ( 2 \gamma \ceta)^{\Te} \geq \frac{2 \gamma \ceta}{8L} \epsilon$. Hence, $|\cS^{(\Te)}|, |\cA^{(\Te)}| \leq C' \big( \frac{8L}{2\gamma \ceta} \big)^d \frac{1}{\epsilon^d}$. Therefore, the overall number of samples utilized
by the algorithm are 
\begin{align}
\sum_{t=1}^{\Te} \big| \Omega^{(t)} \big| N^{(t)}
        & \leq \Te \big| \Omega^{(\Te)} \big| N^{(\Te)}   \nonumber\\
        &\leq \Te \cdot \C \big( |\mS^{(\Te)}| + |\mA^{(\Te)}| \big) \cdot \frac{8}{(2\gamma \ceta)^{2(\Te - 1)}}\log\bigg(\frac{2 \C \big( |\mS^{(\Te)}| + |\mA^{(\Te)}| \big) \Te}{\delta}\bigg)  \nonumber\\
        &\leq \Te \cdot 2\C C' \bigg( \frac{8L}{2\gamma \ceta} \bigg)^d \frac{1}{\epsilon^d} \cdot 8 \bigg(\frac{V_{\max}}{\epsilon}\bigg)^2 \log \bigg( \frac{4\C C'\Te}{\delta} \Big( \frac{8L}{2\gamma \ceta} \Big)^d \frac{1}{\epsilon^d} \bigg)    \nonumber\\
        &= 16\C C'V_{\max}^2 \bigg(\frac{8L}{2\gamma \ceta} \bigg)^d \cdot \frac{\Te}{\epsilon^{d+2}}  \cdot \log \bigg( 4\C C' \Big( \frac{8L}{2\gamma \ceta}\Big)^d \cdot \frac{\Te}{\epsilon^d} \cdot \frac{1}{\delta} \bigg).    \label{eqn:sc}
\end{align}
Since $\Te = \Big\lceil \frac{ \log \big( \frac{V_{\max}}{\epsilon} \big)}{ \log \big( \frac{1}{2 \gamma \ceta } \big)} \Big\rceil = O \big( \log \frac{1}{\epsilon} \big)$, it follows from \eqref{eqn:sc} that the overall sample complexity scales as 
$O \bigg( \frac{1}{\epsilon^{d+2}} \log \frac{1}{\epsilon} \cdot \Big( \log \frac{1}{\epsilon} + \log \frac{1}{\delta} \Big) \bigg)$. This
completes the proof of Theorem \ref{thm:generic}.
\end{proof}

\section{Supplement to Section \ref{sec:rank1}: Rank($Q^*$) $= 1$}\label{sec:proofrank1}

We prove Propsition \ref{prop:rank1} here. We also state and prove Theorem \ref{thm:rank1} 
which incorporates implications of Proposition \ref{prop:rank1} on Theorem \ref{thm:generic}.

\subsection{Proof of Proposition \ref{prop:rank1}}\label{sec:proprank1}

\begin{proof}[Proof of Proposition \ref{prop:rank1}]
First, we note that for any $(s,a) \in \mS \times \mA$,
\begin{equation*}
    Q^*(s,a) = f(s) g(a) = \frac{ f(s) g(a^{\sharp}) f(s^{\sharp}) g(a)}{ f(s^{\sharp}) g(a^{\sharp})} = \frac{Q^*(s, a^\sharp)Q^*(s^\sharp,a)}{Q^*(s^\sharp,a^\sharp)}.
\end{equation*}

We assumed that $\big| \hat{Q}^{(t)}(s,a) - Q^*(s,a) \big| \leq \epsilon$ for all $(s,a)\in\Omega^{(t)}$. 
Since $(s,a^{\sharp}), (s^{\sharp}, a), (s^{\sharp}, a^{\sharp}) \in \Omega^{(t)}$,
\begin{align*}
    \bar{Q}^{(t)}(s,a)  
        &\leq \frac{ \big( 1 + \frac{\epsilon}{ Q^*(s ,a^{\sharp})} \big) \big( 1 + \frac{\epsilon}{ Q^*(s^{\sharp} ,a)} \big) }{1 - \frac{\epsilon}{Q^*(s^{\sharp} ,a^{\sharp})}} Q^*(s,a) 
        % \leq \Big( 1 + \frac{\epsilon}{V_{\min}} \Big)^2\Big( 1 + \frac{2 \epsilon}{Q^*(s^{\sharp}, a^{\sharp})} \Big) Q^*(s,a)
         \leq \Big( 1 + \frac{\epsilon}{V_{\min}} \Big)^2\Big( 1 + \frac{2 \epsilon}{V_{\min}} \Big) Q^*(s,a).
\end{align*}
The last inequality follows from that $\frac{1}{1-x} \leq 1 + 2x$ for $0 \leq x \leq \frac{1}{2}$ and that $\epsilon \leq \frac{1}{2}V_{\min} 
\leq \min \{ Q^*(s,a^{\sharp}), Q^*(s^{\sharp}, a), Q^*(s^{\sharp} ,a^{\sharp}) \}$. Therefore,
\begin{align*}
    \bar{Q}^{(t)}(s,a) - Q^*(s,a) 
        &\leq \bigg[ 4 \Big( \frac{\epsilon}{V_{\min}} \Big) + 5 \Big( \frac{\epsilon}{V_{\min}} \Big)^2 + 2 \Big( \frac{\epsilon}{V_{\min}} \Big)^3 \bigg] Q^*(s,a)\\
        &\leq 7 Q^*(s,a) \frac{\epsilon}{V_{\min}}
        \leq 7 \frac{ V_{\max}}{V_{\min}} \epsilon.
\end{align*}
In a similar manner,
\begin{align*}
    \bar{Q}^{(t)}(s,a)  
        &\geq \frac{ \big( 1 - \frac{\epsilon}{ Q^*(s ,a^{\sharp})} \big) \big( 1 - \frac{\epsilon}{ Q^*(s^{\sharp} ,a)} \big) }{1 + \frac{\epsilon}{Q^*(s^{\sharp} ,a^{\sharp})}} Q^*(s,a)
        \geq \Big( 1 - \frac{\epsilon}{V_{\min}} \Big)^2\Big( 1 - \frac{\epsilon}{Q^*(s^{\sharp}, a^{\sharp})} \Big) Q^*(s,a)
\end{align*}
because $\frac{1}{1+x} \geq 1 - x$ for $0 \leq x \leq \frac{1}{2}$, and thus,
\begin{align*}
    \bar{Q}^{(t)}(s,a) - Q^*(s,a) 
        &\geq \bigg[ -3 \Big( \frac{\epsilon}{V_{\min}} \Big) + 3 \Big( \frac{\epsilon}{V_{\min}} \Big)^2 - \Big( \frac{\epsilon}{V_{\min}} \Big)^3 \bigg] Q^*(s,a)
        \geq -\frac{7}{4} \frac{V_{\max}}{V_{\min}} \epsilon.
\end{align*}
Therefore, for all $(s,a) \in \mS^{(t)} \times \mA^{(t)}$, $\big| \bar{Q}^{(t)}(s,a) - Q^*(s,a) \big| \leq 7 \frac{V_{\max}}{V_{\min}} \epsilon = 7 \frac{R_{\max}}{R_{\min}} \epsilon$. This completes the proof of Proposition \ref{prop:rank1}.
\end{proof}

\subsection{Theorem \ref{thm:rank1} $=$ Proposition \ref{prop:rank1} $+$ Theorem \ref{thm:generic}}

\begin{thm}
\label{thm:rank1}
Let $Q^*$ be rank 1. Consider the RL algorithm (cf. Section \ref{sec:alg_generic}) 
with the Matrix Estimation method as described in Section \ref{sec:rank1}.
If $\gamma < \frac{R_{\min}}{14 R_{\max}}$, then the following two statements are true. 

\begin{enumerate}
    \item For any $\delta > 0$, we have 
    \begin{equation*}
        \sup_{ (s,a) \in \mS \times \mA } \big| Q^{(t)}(s,a) - Q^*(s,a) \big|
                    \leq \bigg( \frac{14 R_{\max}}{R_{\min}} \gamma \bigg)^t V_{\max}, ~~\forall~ 1\leq t\leq T,
    \end{equation*}
    with probability at least $1 - \delta$ by choosing algorithmic parameters $\beta^{(t)}, N^{(t)}$ appropriately.
    \item  Further, given $\epsilon > 0$, it suffices to set $T = \Theta( \log \frac{1}{\epsilon} )$ and use $\tilde{O}( \frac{1}{\epsilon^{d+2}} \cdot \log \frac{1}{\delta} )$ number of samples to achieve 
    \[
        \mathbb{P}\bigg( \sup_{(s,a) \in \mS \times \mA} \big| Q^{(T)}(s,a) - Q^*(s,a) \big| \leq \epsilon\bigg) \geq 1 - \delta.
    \]
\end{enumerate}
\end{thm}

% \begin{thm}
% \label{thm:rank1}
% Let $Q^*$ be rank 1. Consider the RL algorithm (cf. Section \ref{sec:alg_generic}) 
% with the Matrix Estimation method described in Section \ref{sec:rank1}.
% If $\gamma < \frac{R_{\min}}{14 R_{\max}}$ and given $\delta \in (0,1)$, 
% then there exists choice of $\beta^{(t)}, N^{(t)}, 1 \leq t \leq T$ so that 
% $$ \mathbb{P}\Big(\sup_{ (s,a) \in \mS \times \mA } \big| Q^{(t)}(s,a) - Q^*(s,a) \big|
%                     \leq \big( \frac{14 R_{\max}}{R_{\min}} \gamma \big)^t V_{\max}, ~~\forall 1\leq t\leq T\Big) \geq 1-\delta.$$
% Further, for given $\epsilon \in (0,1)$, with choice of $T = \Theta( \log \frac{1}{\epsilon} )$ and using 
% $\tilde{O}( \frac{1}{\epsilon^{d+2}} \cdot \log \frac{1}{\delta} )$ samples, the estimation produced by the
% algorithm at the end of $T$ iteration satisfies
% $$\mathbb{P}\Big(\sup_{(s,a) \in \mS \times \mA} \big| Q^{(T)}(s,a) - Q^*(s,a) \big| \leq \epsilon\Big) \geq 1-\delta.$$
% \end{thm}

\begin{proof}[Proof of Theorem \ref{thm:rank1}]
The proof is basically the same as the proof of Theorem \ref{thm:generic}, with the assumption on the matrix estimation oracle (i.e., Assumption \ref{assu:me}) replaced with the explicit guarantee provided in Proposition \ref{prop:rank1}. 
The only subtlety comes from that Proposition \ref{prop:rank1} is a ``local'' guarantee that holds only for $\epsilon \leq \frac{1}{2} V_{\min}$ whereas Assumption \ref{assu:me} is a global condition that holds for any $\epsilon$.
This requires us to ensure $\max_{(s,a) \in \Omega^{(t)} } \big| \hat{Q}^{(t)}(s, a) - Q^*(s, a) \big| \leq \frac{1}{2} V_{\min}$ for all $t = 1, \ldots, T$, but the argument in the proof of Theorem \ref{thm:generic} itself remains valid.

To that end, we make exactly the same choice of algorithmic parameters $\beta^{(t)}, N^{(t)}$ as 
\begin{equation}\label{eqn:params_rank1}
    \beta^{(t)} = \frac{V_{\max}}{8L} ( 2 \gamma \ceta)^t
    \quad\text{and}\quad
    N^{(t)} = \frac{ 8 }{ ( 2 \gamma \ceta )^{2(t-1)} } \log\bigg(\frac{2 | \Omega^{(t)} | T}{\delta}\bigg)
\end{equation}
with $\ceta = \frac{7 R_{\max}}{R_{\min}}$ as suggested in Proposition \ref{prop:rank1}
and $| \Omega^{(t)} | = \C (|\mS^{(t)}| + |\mA^{(t)}|) $ with $\C = 1$.  To complete the proof, it suffices to 
show that $ \max_{(s,a) \in \Omega^{(t)} } \big| \hat{Q}^{(t)}(s, a) - Q^*(s, a) \big|  \leq \frac{1}{2} V_{\min} $ 
for all $t$. 

We establish this via mathematical induction. For $0 \leq t \leq T$, let 
$B^{(t)} := \sup_{(s,a) \in \mS \times \mA} \big| Q^{(t)}(s,a) - Q^*(s,a) \big|$. 
We can see that if $B^{(t-1)} \leq \big( \frac{14 R_{\max}}{R_{\min}} \gamma \big)^{t-1} V_{\max}$, 
then with probability at least $1 - \frac{\delta}{T}$, the following two inequalities hold:
\begin{enumerate}
    \item 
    $\max_{(s,a) \in \Omega^{(t)} } \big| \hat{Q}^{(t)}(s, a) - Q^*(s, a) \big| \leq \frac{1}{2} V_{\min}$, and
    \item
    $B^{(t)} \leq \frac{14 R_{\max}}{R_{\min}} \gamma B^{(t-1)}$.
\end{enumerate}
The first inequality follows from Lemma \ref{lem:one_step_lookahead} (see also \eqref{eqn:generic_mcla}): with probability at least $1 - \frac{\delta}{T}$,
\begin{align*}
    \max_{(s,a) \in \Omega^{(t)} } \big| \hat{Q}^{(t)}(s, a) - Q^*(s, a) \big|
        &\leq \gamma \left( B^{(t-1)} + \sqrt{ \frac{2 V_{\max}^2}{N^{(t)}} \log\bigg(\frac{2 | \Omega^{(t)} | T}{\delta}\bigg)} \right)
        \leq \frac{3}{2} \gamma B^{(t-1)}\\
        &\leq \frac{3}{2} \gamma V_{\max} \leq \frac{3}{2}\frac{R_{\min}}{14 R_{\max}}V_{\max} = \frac{3}{28} V_{\min} \\
        &\leq \frac{1}{2} V_{\min}.
\end{align*}
Also, the second inequality follows from the same argument as in the proof of Theorem \ref{thm:generic}, cf. Eqs. \eqref{eqn:generic_interpol}, \eqref{eqn:generic_me}, \eqref{eqn:generic_mcla}.

It remains to certify that $B^{(t)} \leq \big( \frac{14 R_{\max}}{R_{\min}} \gamma \big)^{t} V_{\max}$ for $t = 0, \ldots, T-1$. 
First of all, $Q^{(0)}(s,a) \equiv 0$ by assumption, and hence, $B^{(0)} \leq V_{\max}$. Thus, by the second inequality and the condition on $\gamma$, $B^{(t)} \leq \big( \frac{14 R_{\max}}{R_{\min}} \gamma \big) B^{(t-1)} \leq \cdots \leq \big( \frac{14 R_{\max}}{R_{\min}} \gamma \big)^t B^{(0)} \leq \big( \frac{14 R_{\max}}{R_{\min}} \gamma \big)^t V_{\max}$ for all $t = 0, \ldots, T-1$ and the proof is complete.
\end{proof}

\section{Supplement to Section \ref{sec:rankr}: Rank($Q^*$) $= r$}\label{sec:rankrappendix}
In this section, we state and prove a general version of Propsition \ref{prop:rankr_simplified}. Once we have the general version  
Proposition \ref{prop:rankr}, we state and prove Theorem \ref{thm:rankr} which incorporates implications of 
Proposition \ref{prop:rankr} on Theorem \ref{thm:generic}. Lastly, we discuss corollaries for finite space 
in Section \ref{subsec:coro_finite}.

\subsection{Proof of Proposition \ref{prop:rankr_simplified}}
\begin{lem}\label{lem:schur_compl}
	Let $M = \begin{bmatrix} A & B \\ C & D \end{bmatrix}$. If $\textrm{rank}~A = \textrm{rank}~M$, then $D = C A^{\dagger} B$.
\end{lem}
\begin{proof}
	Since $\rowrank ~ \begin{bmatrix} A & B \end{bmatrix} \geq \rowrank~A = \rank~ A = \rank~ M = \rowrank~ M$, there exists a matrix $P$ such that 
	$ \begin{bmatrix} C & D \end{bmatrix} = P \begin{bmatrix} A & B \end{bmatrix}$. Also, observe that $\colrank \begin{bmatrix} A & B \end{bmatrix} 
	\leq \colrank~ M = \rank~M = \rank~A = \colrank~A$. That is, the column space of $B$ is a subspace of the column space of $A$. It follows that
	$AA^{\dagger} A = A$ and $AA^{\dagger}B = B$ because the left multiplication of $AA^{\dagger}$ is the projection on the column space of $A$.
	We obtain
	\[
		\begin{bmatrix} C & D \end{bmatrix}
			= P \begin{bmatrix} A & B \end{bmatrix}
			= P \begin{bmatrix} AA^{\dagger} A & AA^{\dagger}B \end{bmatrix}
			= \begin{bmatrix} PA & PAA^{\dagger}B \end{bmatrix}.
	\]
	Therefore, $PA = C$ and $D = PAA^{\dagger}B = C A^{\dagger} B$.
\end{proof}

\begin{prop}\label{prop:rankr}
    Let $\Omega^{(t)}$ and $\bar{Q}^{(t)}$ as described above.
    For any $\epsilon \leq \frac{1}{2\sqrt{|\mS^{\sharp}| |\mA^{\sharp}|}} \sigma_r\big( Q^*(\mS^{\sharp}, \mA^{\sharp} ) \big)$, if $\max_{(s,a) \in \Omega^{(t)}} \big| \hat{Q}^{(t)}(s,a) - Q^*(s,a) \big| \leq \epsilon$, then
    \begin{align}
        &\max_{(s,a) \in \mS^{(t)}\times \mA^{(t)}} \big| \bar{Q}^{(t)}(s,a) - Q^*(s,a) \big|   \nonumber\\
        &\qquad\qquad\qquad\leq \Bigg( 6\sqrt{2} \bigg( \frac{\sqrt{|\mS^{\sharp}| |\mA^{\sharp}|}}{\sigma_r \big( Q^*(\mS^{\sharp}, \mA^{\sharp} )\big)} \bigg) + 2(1+\sqrt{5}) \bigg( \frac{\sqrt{|\mS^{\sharp}| |\mA^{\sharp}|}}{\sigma_r \big( Q^*(\mS^{\sharp}, \mA^{\sharp} )\big)} \bigg)^2 \Bigg) V_{\max} \epsilon.
            \label{eqn:prop_rankr}
    \end{align}
\end{prop}

\begin{proof}[Proof of Proposition \ref{prop:rankr}]
First, we observe that for any $(s,a) \in \mS^{(t)} \times \mA^{(t)}$
\begin{equation}\label{eqn:ME_rankr.true} 
	Q^*(s,a) = Q^*(s, \mA^{\sharp}) \big[ Q^*(\mS^{\sharp}, \mA^{\sharp}) \big]^{\dagger} Q^*(\mS^{\sharp}, a ).
\end{equation}
This can be verified by applying Lemma \ref{lem:schur_compl} to $M = \begin{bmatrix} A & B \\ C & D \end{bmatrix} \in \RR^{|\bar{\mS}^{(t)}| \times |\bar{\mA}^{(t)}|}$ where
\begin{align*}
	A &= Q^*( \mS^{\sharp}, \mA^{\sharp} ),	&B &= Q^*( \mS^{\sharp}, \mA^{(t)} ), \\
	C &= Q^*( \mS^{(t)}, \mA^{\sharp} ),		&D &= Q^*( \mS^{(t)}, \mA^{(t)} ).
\end{align*}
Here, $\rank~ A = r = \rank~ M$ by definition of anchor states/actions and the fact that $Q^*$ has rank $r$. Hence, 
\[
	Q^*( \mS^{(t)}, \mA^{(t)} ) 
		= D = C A^{\dagger} B
		= Q^*( \mS^{(t)}, \mA^{\sharp} ) \big[ Q^*( \mS^{\sharp}, \mA^{\sharp} ) \big]^{\dagger}  Q^*( \mS^{\sharp}, \mA^{(t)} ).
\]

Next, we fix $(s,a) \in \mS^{(t)} \times \mA^{(t)}$ and consider the error $\bar{Q}^{(t)}(s,a) - Q^*(s,a)$.
According to the definition of $\bar{Q}^{(t)}(s,a)$ (cf. \eqref{eqn:ME_rankr.0}) and \eqref{eqn:ME_rankr.true},
\begin{align*}
	\bar{Q}^{(t)}(s,a) - Q^*(s,a)
		&= \hat{Q}^{(t)} (s, \mA^{\sharp} ) \big[ \hat{Q}^{(t)}(\mS^{\sharp}, \mA^{\sharp} ) \big]^{\dagger} \hat{Q}^{(t)}(\mS^{\sharp}, a ) 
				- Q^*(s, \mA^{\sharp}) \big[ Q^*(\mS^{\sharp}, \mA^{\sharp}) \big]^{\dagger} Q^*(\mS^{\sharp}, a ) \\
		&\leq \hat{Q}^{(t)} (s, \mA^{\sharp} ) \big[ \hat{Q}^{(t)}(\mS^{\sharp}, \mA^{\sharp} ) \big]^{\dagger} \hat{Q}^{(t)}(\mS^{\sharp}, a ) 
			- Q^* (s, \mA^{\sharp} ) \big[ \hat{Q}^{(t)}(\mS^{\sharp}, \mA^{\sharp} ) \big]^{\dagger} Q^*(\mS^{\sharp}, a )\\
		&\qquad + Q^* (s, \mA^{\sharp} ) \big[ \hat{Q}^{(t)}(\mS^{\sharp}, \mA^{\sharp} ) \big]^{\dagger} Q^*(\mS^{\sharp}, a )
				- Q^*(s, \mA^{\sharp}) \big[ Q^*(\mS^{\sharp}, \mA^{\sharp}) \big]^{\dagger} Q^*(\mS^{\sharp}, a ) \\
		&= \tr \Big( \big[ \hat{Q}^{(t)}(\mS^{\sharp}, \mA^{\sharp} ) \big]^{\dagger} \cdot \Big[ \hat{Q}^{(t)}(\mS^{\sharp}, a )  \hat{Q}^{(t)} (s, \mA^{\sharp} ) -  Q^*(\mS^{\sharp}, a ) Q^* (s, \mA^{\sharp} )  \Big]  \Big)\\
		&\qquad + \tr \Big( \Big\{ \big[ \hat{Q}^{(t)}(\mS^{\sharp}, \mA^{\sharp} ) \big]^{\dagger} - \big[ Q^*(\mS^{\sharp}, \mA^{\sharp}) \big]^{\dagger} \Big\} \cdot Q^*(\mS^{\sharp}, a ) Q^*(s, \mA^{\sharp}) \Big).
\end{align*}
Since $\big| \tr(AB) \big| \leq \sqrt{\rank~ B}\|A\|_{op} \|B\|_F$, we obtain
\begin{align}
	\big| \bar{Q}^{(t)}(s,a) - Q^*(s,a) \big|
		&\leq \sqrt{2} \Big\| \big[ \hat{Q}^{(t)}(\mS^{\sharp}, \mA^{\sharp} ) \big]^{\dagger} \Big\|_{op} \Big\| \hat{Q}^{(t)}(\mS^{\sharp}, a )  \hat{Q}^{(t)} (s, \mA^{\sharp} ) -  Q^*(\mS^{\sharp}, a ) Q^* (s, \mA^{\sharp} )  \Big\|_F		\label{eqn:err_rankr_term.1}\\
			&\qquad + \Big\| \big[ \hat{Q}^{(t)}(\mS^{\sharp}, \mA^{\sharp} ) \big]^{\dagger} - \big[ Q^*(\mS^{\sharp}, \mA^{\sharp}) \big]^{\dagger} \Big\|_{op} \Big\| Q^*(\mS^{\sharp}, a ) Q^*(s, \mA^{\sharp}) \Big\|_F.	\label{eqn:err_rankr_term.2}
\end{align}
In the remainder of the proof, we establish upper bounds for \eqref{eqn:err_rankr_term.1} and \eqref{eqn:err_rankr_term.2} separately.

\begin{itemize}
\item
Upper bound for \eqref{eqn:err_rankr_term.1}.
Note that $\hat{Q}^{(t)}(\mS^{\sharp}, \mA^{\sharp} ) = Q^*(\mS^{\sharp}, \mA^{\sharp}) + E$ for some $E \in \RR^{|\mS^{\sharp}| \times | \mA^{\sharp}|}$ such that $\| E \|_{\max} \leq \epsilon$ by assumption. Therefore, $\| E \|_{op} \leq \sqrt{|\mS^{\sharp}| |\mA^{\sharp}|} \| E \|_{\max} \leq \epsilon \sqrt{|\mS^{\sharp}| |\mA^{\sharp}|}$. Since $\sigma_r\big( \hat{Q}^{(t)}(\mS^{\sharp}, \mA^{\sharp} ) \big) \geq \sigma_r \big( Q^*(\mS^{\sharp}, \mA^{\sharp}) \big) - \| E \|_{op}$ due to Weyl's inequality, we have
\[
	\Big\| \big[ \hat{Q}^{(t)}(\mS^{\sharp}, \mA^{\sharp} ) \big]^{\dagger} \Big\|_{op}
		= \frac{1}{ \sigma_r \big( \hat{Q}^{(t)}(\mS^{\sharp}, \mA^{\sharp} )  \big)}
		\leq \frac{1}{ \sigma_r \big( Q^*(\mS^{\sharp}, \mA^{\sharp} ) \big) - \epsilon \sqrt{|\mS^{\sharp}| |\mA^{\sharp}|} }
		\leq \frac{2}{ \sigma_r \big( Q^*(\mS^{\sharp}, \mA^{\sharp} ) \big) }
\]
provided that $\epsilon \leq \frac{1}{2\sqrt{|\mS^{\sharp}| |\mA^{\sharp}|}} \sigma_r\big( Q^*(\mS^{\sharp}, \mA^{\sharp} ) \big)$.

It is easy to see $\Big\| \hat{Q}^{(t)}(\mS^{\sharp}, a )  \hat{Q}^{(t)} (s, \mA^{\sharp} ) -  Q^*(\mS^{\sharp}, a ) Q^* (s, \mA^{\sharp} )  \Big\|_F \leq \big( 2 V_{\max} \epsilon + \epsilon^2 \big) \sqrt{|\mS^{\sharp}| |\mA^{\sharp}|}$ because $\big| \hat{Q}^{(t)}(s, a )  \hat{Q}^{(t)} (s, a ) -  Q^*(s, a ) Q^* (s, a ) \big| \leq 2 V_{\max} \epsilon + \epsilon^2$ for all $(s,a) \in \mS^{(t)} \times \mA^{(t)}$.

\item
Upper bound for \eqref{eqn:err_rankr_term.2}.
We derive an upper bound on $\big\| \big[ \hat{Q}^{(t)}(\mS^{\sharp}, \mA^{\sharp} ) \big]^{\dagger} - \big[ Q^*(\mS^{\sharp}, \mA^{\sharp}) \big]^{\dagger} \big\|_{op}$ using a classical result on the perturbation of pseudoinverses. By Theorem 3.3 of \cite{stewart1977perturbation}, for any $A$ and $B$ with $B = A + \Delta$,
\[	\| B^{\dagger} - A^{\dagger} \|_{op} \leq \frac{1 + \sqrt{5}}{2} \max\{ \| A^{\dagger} \|_{op}^2, \| B^{\dagger} \|_{op}^2  \} \| \Delta \|_{op} 	\]
Therefore,
\[
	\big\| \big[ \hat{Q}^{(t)}(\mS^{\sharp}, \mA^{\sharp} ) \big]^{\dagger} - \big[ Q^*(\mS^{\sharp}, \mA^{\sharp}) \big]^{\dagger} \big\|_{op}
		\leq \frac{1 + \sqrt{5}}{2} \cdot \frac{4}{ \sigma_r \big( Q^*(\mS^{\sharp}, \mA^{\sharp} ) \big)^2} \cdot \epsilon \sqrt{|\mS^{\sharp}| |\mA^{\sharp}|}.
\]

Also, it is easy to see that $\big\| Q^*(\mS^{\sharp}, a ) Q^*(s, \mA^{\sharp}) \big\|_F \leq V_{\max} \sqrt{|\mS^{\sharp}| |\mA^{\sharp}|}$.
\end{itemize}

All in all, inserting the upper bounds back into \eqref{eqn:err_rankr_term.1} and \eqref{eqn:err_rankr_term.2}, we have
\begin{align*}
	\big| \bar{Q}^{(t)}(s,a) - Q^*(s,a) \big|
		&\leq \frac{2\sqrt{2}}{ \sigma_r \big( Q^*(\mS^{\sharp}, \mA^{\sharp} ) \big) } ( 2 V_{\max} \epsilon + \epsilon^2 ) \sqrt{|\mS^{\sharp}| |\mA^{\sharp}|}\\
		&\qquad
			+ \frac{2(1 + \sqrt{5})}{\sigma_r \big( Q^*(\mS^{\sharp}, \mA^{\sharp} )\big)^2} V_{\max} |\mS^{\sharp}| |\mA^{\sharp}| \epsilon\\
		&\leq \Bigg( 6\sqrt{2} \bigg( \frac{\sqrt{|\mS^{\sharp}| |\mA^{\sharp}|}}{\sigma_r \big( Q^*(\mS^{\sharp}, \mA^{\sharp} )\big)} \bigg) + 2(1+\sqrt{5}) \bigg( \frac{\sqrt{|\mS^{\sharp}| |\mA^{\sharp}|}}{\sigma_r \big( Q^*(\mS^{\sharp}, \mA^{\sharp} )\big)} \bigg)^2 \Bigg) V_{\max} \epsilon.
\end{align*}
\end{proof}

\subsection{Theorem \ref{thm:rankr} $=$ Proposition \ref{prop:rankr} $+$ Theorem \ref{thm:generic}}\label{sec:proofthmrankr}

We state Theorem \ref{thm:rankr} that follows as Corollary of Theorem \ref{thm:generic} using Proposition 
\ref{prop:rankr_simplified} (or Proposition \ref{prop:rankr}). Recall that assuming $|\mS^{\sharp}| = |\mA^{\sharp}| = r$, we defined the following quantity
\[
    c(r; \mS^{\sharp}, \mA^{\sharp} ) = \bigg( 6\sqrt{2} \Big( \frac{r}{\sigma_r ( Q^*(\mS^{\sharp}, \mA^{\sharp} ))} \Big) + 2(1+\sqrt{5}) \Big( \frac{r}{\sigma_r ( Q^*(\mS^{\sharp}, \mA^{\sharp} ))} \Big)^2 \bigg) V_{\max}
\]
in Proposition \ref{prop:rankr_simplified}. This is a special case of $\ceta$ for $|\mS^{\sharp}| = |\mA^{\sharp}| = r$ 
that appears in Proposition \ref{prop:rankr} as the multiplier on the right-hand side of \eqref{eqn:prop_rankr}. 
This quantity appears in the following theorem statement to determine the range of $\gamma$ and the convergence rate.

As a matter of fact, our algorithm does not require $|\mS^{\sharp}| = |\mA^{\sharp}| = r$. We present a general 
theorem for approximate rank-$r$ setup (Theorem \ref{thm:apprx_r}) in Appendix \ref{sec:apprx_r} in full generality 
without assuming $|\mS^{\sharp}| = |\mA^{\sharp}| = r$. One can derive a general version of Theorem \ref{thm:rankr} 
for $\mS^{\sharp}, \mA^{\sharp}$ beyond $|\mS^{\sharp}| = |\mA^{\sharp}| = r$ from Theorem \ref{thm:apprx_r} 
by letting $\zeta_r = 0$, where $\zeta_r$ is the approximation error between the rank-$r$ approximation of $Q^*$ and the actual $Q^*$. That is, if $Q^*$ is of rank $r$, $\zeta_r = 0$.
Parsing our general results briefly, we remark that as long as $|\mS^{\sharp}| = |\mA^{\sharp}| = O(r)$, we achieve the same scaling of sample complexity in terms of the problem dimensions.

\begin{thm}
\label{thm:rankr}
Let $Q^*$ have rank $r$. Consider  the RL algorithm (cf. Section \ref{sec:alg_generic}) with the Matrix Estimation method
as described in Section \ref{sec:rankr}. If $\gamma \leq \frac{1}{2 c(r; \mS^{\sharp}, \mA^{\sharp} )}$, 
then the following statements hold.
    \begin{enumerate}
    \item    For any $\delta > 0$, we have 
    \begin{equation*}%\label{eqn:geom_rate}
        \sup_{ (s,a) \in \mS \times \mA } \big| Q^{(t)}(s,a) - Q^*(s,a) \big|
            \leq \big(  2 c(r; \mS^{\sharp}, \mA^{\sharp} ) \gamma  \big)^t V_{\max},
             \quad\textrm{for all } t = 1, \ldots, T
    \end{equation*}
    with probability at least $1 - \delta$ by choosing algorithmic parameters $\beta^{(t)}, N^{(t)}$ appropriately.

    \item
    Further, given $\epsilon > 0$, it suffices to set $T = \Theta( \log \frac{1}{\epsilon} )$ and use $\tilde{O}( \frac{1}{\epsilon^{d+2}} \cdot \log \frac{1}{\delta} )$ number of samples to achieve 
    \[
        \mathbb{P}\bigg( \sup_{(s,a) \in \mS \times \mA} \big| Q^{(T)}(s,a) - Q^*(s,a) \big| \leq \epsilon \bigg)
            \geq 1 - \delta.
    \]
    \end{enumerate}
\end{thm}

\begin{proof}[Proof of Theorem \ref{thm:rankr}]
We take the same approach as in the proof of Theorem \ref{thm:rank1}; the proof is essentially the same as the proof of Theorem \ref{thm:generic}, with the assumption on the matrix estimation oracle (i.e., Assumption \ref{assu:me}) replaced with the explicit guarantee provided in Proposition \ref{prop:rankr_simplified}. 

As before, the only subtlety comes from that Proposition \ref{prop:rankr_simplified} is a ``local'' guarantee that holds only for $\epsilon \leq \frac{1}{2r} \sigma_r\big( Q^*(\mS^{\sharp}, \mA^{\sharp} ) \big)$ whereas Assumption \ref{assu:me} is a global condition that holds for any $\epsilon$.
This requires us to ensure $\max_{(s,a) \in \Omega^{(t)} } \big| \hat{Q}^{(t)}(s, a) - Q^*(s, a) \big| \leq \frac{1}{2r} \sigma_r\big( Q^*(\mS^{\sharp}, \mA^{\sharp} ) \big)$ for all $t = 1, \ldots, T$, but the argument in the proof of Theorem \ref{thm:generic} itself remains valid.

We make exactly the same choice of algorithmic parameters $\beta^{(t)}, N^{(t)}$ as in the proof of Theorem \ref{thm:generic}:
\begin{equation*}%\label{eqn:params_rankr}
    \beta^{(t)} = \frac{V_{\max}}{8L} ( 2 \gamma \ceta)^t
    \quad\text{and}\quad
    N^{(t)} = \frac{ 8 }{ ( 2 \gamma \ceta )^{2(t-1)} } \log\bigg(\frac{2 | \Omega^{(t)} | T}{\delta}\bigg)
\end{equation*}
with an adaptation $\ceta = c(r; \mS^{\sharp}, \mA^{\sharp} )$ as suggested in Proposition \ref{prop:rankr_simplified}.
To complete the proof, it suffices to show that $ \max_{(s,a) \in \Omega^{(t)} } \big| \hat{Q}^{(t)}(s, a) - Q^*(s, a) \big|  \leq \frac{1}{2r} \sigma_r\big( Q^*(\mS^{\sharp}, \mA^{\sharp} ) \big)$ for all $t = 1, \ldots, T$.

In the rest of the proof, we prove the above claim by mathematical induction. 
For $t = 0, \ldots, T$, we let $B^{(t)} := \sup_{(s,a) \in \mS \times \mA} \big| Q^{(t)}(s,a) - Q^*(s,a) \big|$. We can see that if $B^{(t-1)} \leq \big( 2 c(r; \mS^{\sharp}, \mA^{\sharp} ) \gamma \big)^{t-1} V_{\max}$, then with probability at least 
$1 - \frac{\delta}{T}$, the following two inequalities hold:
\begin{enumerate}
    \item 
    $\max_{(s,a) \in \Omega^{(t)} } \big| \hat{Q}^{(t)}(s, a) - Q^*(s, a) \big| \leq \frac{1}{2r} \sigma_r\big( Q^*(\mS^{\sharp}, \mA^{\sharp} ) \big)$, and
    \item
    $B^{(t)} \leq 2 c(r; \mS^{\sharp}, \mA^{\sharp} ) \gamma B^{(t-1)}$.
\end{enumerate}

\begin{itemize}
    \item 
    To prove the first claim, we apply Lemma \ref{lem:one_step_lookahead} (see also \eqref{eqn:generic_mcla}) and observe that 
    with probability at least $1 - \frac{\delta}{T}$,
    \begin{align*}
        \max_{(s,a) \in \Omega^{(t)} } \big| \hat{Q}^{(t)}(s, a) - Q^*(s, a) \big|
            &\leq \gamma \left( B^{(t-1)} + \sqrt{ \frac{2 V_{\max}^2}{N^{(t)}} \log\bigg(\frac{2 | \Omega^{(t)} | T}{\delta}\bigg)} \right)\\
            &\leq \frac{3}{2} \gamma B^{(t-1)}\leq \frac{3}{2} \gamma V_{\max}\\
            &\leq \frac{3}{4 c(r; \mS^{\sharp}, \mA^{\sharp} )} V_{\max}.
    \end{align*}
    Here, the last two inequalities follow from the assumptions that $\gamma \leq \frac{1}{2 c(r; \mS^{\sharp},\mA^{\sharp} )}$ and that $B^{(t-1)} \leq \big( 2 c(r; \mS^{\sharp}, \mA^{\sharp} ) \gamma \big)^{t-1} V_{\max} \leq V_{\max}$. Then it suffices to show that 
    \begin{equation}\label{eqn:induction_1}
        \frac{3}{4 c(r; \mS^{\sharp}, \mA^{\sharp} )} V_{\max}
            \leq \frac{1}{2r} \sigma_r\big( Q^*(\mS^{\sharp}, \mA^{\sharp} ) \big).
    \end{equation}
    Recall that $c(r; \mS^{\sharp}, \mA^{\sharp} ) = \Big( 6\sqrt{2} \big( \frac{r}{\sigma_r ( Q^*(\mS^{\sharp}, \mA^{\sharp} ))} \big) + 2(1+\sqrt{5}) \big( \frac{r}{\sigma_r ( Q^*(\mS^{\sharp}, \mA^{\sharp} ))} \big)^2 \Big) V_{\max}$. Also, we observe that $\sigma_r ( Q^*(\mS^{\sharp}, \mA^{\sharp} )) > 0$ by definition of the anchor states/actions. Thus, \eqref{eqn:induction_1} is satisfied if the following inequality is true:
    \[
        \bigg( \frac{r}{\sigma_r\big( Q^*(\mS^{\sharp}, \mA^{\sharp}) \big)} \bigg) 
        \Bigg\{ 2(1+\sqrt{5})  \bigg( \frac{r}{\sigma_r\big( Q^*(\mS^{\sharp}, \mA^{\sharp}) \big)} \bigg) + 6\sqrt{2} - \frac{3}{2} \Bigg\}
            \geq 0.
    \]
    This quadratic inequality is satisfied if and only if $\frac{r}{\sigma_r( Q^*(\mS^{\sharp}, \mA^{\sharp}) )} \geq 0$ or $\frac{r}{\sigma_r( Q^*(\mS^{\sharp}, \mA^{\sharp}) )} \leq \frac{3 - 12\sqrt{2}}{4(1+\sqrt{5})}$. Since $r \geq 1$ and $\sigma_r( Q^*(\mS^{\sharp}, \mA^{\sharp}) ) > 0$, we always have $\frac{r}{\sigma_r( Q^*(\mS^{\sharp}, \mA^{\sharp}) )} \geq 0$ and therefore, the 
    inequality in \eqref{eqn:induction_1} is always true.

    \item
    The inequality in the second claim follows from the same argument as in the proof of Theorem \ref{thm:generic}, cf. Eqs. \eqref{eqn:generic_interpol}, \eqref{eqn:generic_me}, \eqref{eqn:generic_mcla}.
\end{itemize}

It remains to certify that $B^{(t)} \leq \big( 2 c(r; \mS^{\sharp}, \mA^{\sharp} ) \gamma \big)^{t} V_{\max}$ for $t = 0, \ldots, T-1$. 
First of all, $Q^{(0)}(s,a) \equiv 0$ by assumption, and hence, $B^{(0)} \leq V_{\max}$. Thus, by the second inequality and the condition on $\gamma$, $B^{(t)} \leq 2 c(r; \mS^{\sharp}, \mA^{\sharp} ) \gamma B^{(t-1)} \leq \ldots \leq \big( 2 c(r; \mS^{\sharp}, \mA^{\sharp} ) \gamma \big)^t B^{(0)} \leq \big( 2 c(r; \mS^{\sharp}, \mA^{\sharp} ) \gamma \big)^t V_{\max}$ for all $t = 1, \ldots, T-1$ and the proof is complete.
\end{proof}

\subsection{Corollaries of Theorem \ref{thm:rankr}}
\label{subsec:coro_finite}

Recall that our algorithm do not demand any special properties of $\mS, \mA$ except the existence of $\beta^{(t)}$-net, 
which is the case whenever $\mS, \mA$ are compact. Also, our analysis is general in the sense that it only requires 
$\mS, \mA$ to be compact with finite measures, and $Q^*$ to be $L$-Lipschitz. Therefore, it is not hard to see that 
our algorithm and analysis are applicable to the case where state or action space is finite, or both. 
We summarize results below as corollaries of Theorem \ref{thm:rankr} without proofs.

Before presenting the results, we recall the following quantity defined in Proposition \ref{prop:rankr_simplified}:
\[
    c(r; \mS^{\sharp}, \mA^{\sharp} ) = \bigg( 6\sqrt{2} \Big( \frac{r}{\sigma_r ( Q^*(\mS^{\sharp}, \mA^{\sharp} ))} \Big) + 2(1+\sqrt{5}) \Big( \frac{r}{\sigma_r ( Q^*(\mS^{\sharp}, \mA^{\sharp} ))} \Big)^2 \bigg) V_{\max},
\]
which is a special case of $\ceta$ for $|\mS^{\sharp}| = |\mA^{\sharp}| = r$ that appears in 
Proposition \ref{prop:rankr} as the multiplier on the right-hand side of \eqref{eqn:prop_rankr}. 
This quantity determines the range of $\gamma$ and the convergence rate.

\medskip
\noindent{\bf Continuous $\mS\subset\mathbb{R}^d$ and Finite $\mA$.} In this case, the algorithm only needs to discretize the state space at each iteration. In other words, $\mA^{(t)}=\mA$, for all $t=1,\dots, T$ and  $\Omega^{(t)} = \{ (s,a) \in \bar{\mS}^{(t)} \times {\mA}: s \in \mS^{\sharp} \text{ or } a \in \mA^{\sharp} \}$. Finally, the generalization step only needs to interpolate the state space $\mS$. 
%Suppose that in this case (i.e., continuous $\mS$ and finite $\mA$), the optimal $Q^*$ is rank-$r$: $Q^*=\sum_{i=1}^r\sigma_if_i(s)g_i(a)$, and 
Let $|\mS^{\sharp}| = |\mA^{\sharp}| = r$. Then, 
we have the following guarantees as an immediate corollary of Theorem \ref{thm:rankr}:

\begin{cor}
\label{cor:rankr}
    
    Consider the rank-$r$ setting with continuous $\mS$ and finite $\mA$.
    Suppose that we run  the RL algorithm (cf. Section \ref{sec:alg_generic}) with the Matrix Estimation method
described in Section \ref{sec:rankr}.
    If $\gamma \leq \frac{1}{2 c(r; \mS^{\sharp}, \mA^{\sharp} )}$, then the following holds.
    \begin{enumerate}
    \item
    For any $\delta > 0$, we have 
    \begin{equation*}%\label{eqn:geom_rate}
        \sup_{ (s,a) \in \mS \times \mA } \big| Q^{(t)}(s,a) - Q^*(s,a) \big|
            \leq \big(  2 c(r; \mS^{\sharp}, \mA^{\sharp} ) \gamma  \big)^t V_{\max},
             \quad\textrm{for all } t = 1, \ldots, T
    \end{equation*}
    with probability at least $1 - \delta$ by choosing algorithmic parameters $\beta^{(t)}, N^{(t)}$ appropriately.
    % $\beta^{(t)} = \frac{V_{\max}}{8L} \big( \frac{14 R_{\max}}{R_{\min}} \gamma \big)^t$, $N^{(t)} = 8 \big( \frac{14 R_{\max}}{R_{\min}} \gamma \big)^{-2(t-1)} \log\big(\frac{2 | \Omega^{(t)} | T}{\delta}\big)$ for $t = 1, \ldots, T$.
    
    \item
Further, given  $\epsilon > 0$,
%For any $\epsilon > 0$ and $\delta > 0$, 
it suffices to set $T = \Theta( \log \frac{1}{\epsilon} )$ and use $\tilde{O}( \frac{1}{\epsilon^{d+2}} \cdot \log \frac{1}{\delta} )$ number of samples to achieve 
    \[
        \mathbb{P}\bigg(\sup_{(s,a) \in \mS \times \mA} \big| Q^{(T)}(s,a) - Q^*(s,a) \big| \leq \epsilon \bigg)
            \geq 1 - \delta.
    \]
    \end{enumerate}
\end{cor}

\medskip\noindent
{\bf Finite $\mS$ and Finite $\mA$.} Since the spaces are discrete, we have an optimal $Q^*$ being a $|\mS|\times|\mA|$ matrix. For this special case, the algorithm simply skips the discretization (i.e., $\beta^{(t)}=0$) and generalization steps at each iteration. In other words, $\mS^{(t)}=\mS$ and $\mA^{(t)}=\mA$, for all $t=1,\dots, T$, and  $\Omega^{(t)} = \{ (s,a) \in {\mS} \times {\mA}: s \in \mS^{\sharp} \text{ or } a \in \mA^{\sharp} \}$. Suppose that the optimal matrix $Q^*(\mS,\mA)$ is rank-$r$ and let $|\mS^{\sharp}| = |\mA^{\sharp}| = r$. We then have the following guarantees:

\begin{cor}
\label{cor:rankr_finite}
    
    Consider finite $\mS$ and finite $\mA$ with the optimal matrix $Q^*(\mS,\mA)$ being rank-$r$.  Suppose that we run  the RL algorithm (cf. Section \ref{sec:alg_generic}) with the Matrix Estimation method
described in Section \ref{sec:rankr}.
    If $\gamma \leq \frac{1}{2 c(r; \mS^{\sharp}, \mA^{\sharp} )}$, then the following holds.
    \begin{enumerate}
    \item
    For any $\delta > 0$, we have 
    \begin{equation*}%\label{eqn:geom_rate}
        \sup_{ (s,a) \in \mS \times \mA } \big| Q^{(t)}(s,a) - Q^*(s,a) \big|
            \leq \big(  2 c(r; \mS^{\sharp}, \mA^{\sharp} ) \gamma  \big)^t V_{\max},
             \quad\textrm{for all } t = 1, \ldots, T
    \end{equation*}
    with probability at least $1 - \delta$ by choosing algorithmic parameters $\beta^{(t)}, N^{(t)}$ appropriately.
    % $\beta^{(t)} = \frac{V_{\max}}{8L} \big( \frac{14 R_{\max}}{R_{\min}} \gamma \big)^t$, $N^{(t)} = 8 \big( \frac{14 R_{\max}}{R_{\min}} \gamma \big)^{-2(t-1)} \log\big(\frac{2 | \Omega^{(t)} | T}{\delta}\big)$ for $t = 1, \ldots, T$.
    
    \item
Further, given  $\epsilon > 0$, it suffices to set $T = \Theta( \log \frac{1}{\epsilon} )$ and use $\tilde{O}( \frac{\max(|\mS|,|\mA|)}{\epsilon^{2}} \cdot \log \frac{1}{\delta} )$ number of samples to achieve 
    \[
        \mathbb{P}\bigg(\sup_{(s,a) \in \mS \times \mA} \big| Q^{(T)}(s,a) - Q^*(s,a) \big| \leq \epsilon \bigg)
            \geq 1 - \delta.
    \]
    \end{enumerate}
\end{cor}

\section{Proof of Theorem \ref{thm:apprx_r}}
\label{sec:apprx_r}

The proof of Proposition \ref{prop:apprx_r} is omitted due to its similarity to the proof of Proposition \ref{prop:rankr} with minor modifications.
\begin{proof}[Proof of Theorem \ref{thm:apprx_r}]
In this proof, we repeat the proof of Theorem \ref{thm:rankr} with necessary modifications. 
As before, we must ensure $\sup_{(s,a) \in \Omega^{(t)} } \big| \hat{Q}^{(t)}(s, a) - Q^*_r(s, a) \big| \leq \frac{1}{2\sqrt{|\mS^{\sharp}||\mA^{\sharp}|}} \sigma_r\big( Q^*_r(\mS^{\sharp}, \mA^{\sharp} ) \big)$ for all $t = 1, \ldots, T$ to use Proposition \ref{prop:apprx_r}. For that purpose, we make exactly the same choice of algorithmic parameters $\beta^{(t)}, N^{(t)}$ as in the proof of Theorem \ref{thm:generic}, i.e., we let
\begin{equation*}%\label{eqn:params_rankr}
    \beta^{(t)} = \frac{V_{\max}}{8L} ( 2 \gamma  \ceta)^t
    \quad\text{and}\quad
    N^{(t)} = \frac{ 8 }{ ( 2 \gamma  \ceta )^{2(t-1)} } \log\bigg(\frac{2 | \Omega^{(t)} | T}{\delta}\bigg)
\end{equation*}
with an adaptation $ \ceta = \phi_c(r; \mS^{\sharp}, \mA^{\sharp} )$ as suggested in Proposition \ref{prop:apprx_r}, cf. \eqref{eqn:phi_c}.
To complete the proof, it suffices to show that $ \sup_{(s,a) \in \Omega^{(t)} } \big| \hat{Q}^{(t)}(s, a) - Q^*_r(s, a) \big|  \leq \frac{1}{2\sqrt{|\mS^{\sharp}||\mA^{\sharp}|}} \sigma_r\big( Q^*_r(\mS^{\sharp}, \mA^{\sharp} ) \big)$ for all $t = 1, \ldots, T$.

In the rest of the proof, we prove the above claim by mathematical induction. For notational brevity, we use the shorthand notation $\sigma_r := \sigma_r\big( Q^*_r(\mS^{\sharp}, \mA^{\sharp} ) \big)$ and $c_r := \phi_c(r; \mS^{\sharp}, \mA^{\sharp})$.
For $t = 0, \ldots, T$, we let $B^{(t)} := \sup_{(s,a) \in \mS \times \mA} \big| Q^{(t)}(s,a) - Q^*_r(s,a) \big|$. We can see that if 
\[
    B^{(t-1)} \leq \big( 2 c_r \gamma \big)^{t-1} V_{\max} + (1 + c_r \gamma) \zeta_r \sum_{i=1}^{t-1} \big( c_r \gamma \big)^{i-1},
\] 
then with probability at least $1 - \frac{\delta}{T}$, the following two inequalities hold:
\begin{enumerate}
    \item 
    $\sup_{(s,a) \in \Omega^{(t)} } \big| \hat{Q}^{(t)}(s, a) - Q^*_r(s, a) \big| \leq \frac{1}{2\sqrt{|\mS^{\sharp}||\mA^{\sharp}|}} \sigma_r$, and
    \item
    $B^{(t)} \leq \big( 2 c_r \gamma \big)^{t} V_{\max} + (1 + c_r \gamma) \zeta_r \sum_{i=1}^{t} \big( c_r \gamma \big)^{i-1}$.
\end{enumerate}

\begin{itemize}
    \item 
    To prove the first claim, we apply Lemma \ref{lem:one_step_lookahead} (see also \eqref{eqn:generic_mcla}) and observe that 
    with probability at least $1 - \frac{\delta}{T}$,
    \begin{align*}
        \sup_{(s,a) \in \Omega^{(t)} } \big| \hat{Q}^{(t)}(s, a) - Q^*(s, a) \big|
            &\leq \gamma \left( B^{(t-1)} + \sqrt{ \frac{2 V_{\max}^2}{N^{(t)}} \log\bigg(\frac{2 | \Omega^{(t)} | T}{\delta}\bigg)} \right)\\
            &\leq \frac{3\gamma}{2} ( 2 c_r \gamma )^{t-1} V_{\max} + \gamma (1 + c_r \gamma) \zeta_r \sum_{i=1}^{t-1} (c_r \gamma )^{i-1}\\
            &\leq \frac{3\gamma}{2} V_{\max} + \frac{(1+c_r \gamma)\gamma}{1- c_r \gamma} \zeta_r.
            %&\leq \frac{3}{2} \gamma B^{(t-1)}
            % \leq \frac{3}{2} \Big( \gamma V_{\max} + \frac{1+c_r}{1-c_r\gamma} \zeta_r \Big)\\
            % &\leq \frac{3}{4 \phi_c(r; \mS^{\sharp}, \mA^{\sharp} )} V_{\max}.
    \end{align*}
    % Here, the last two inequalities follow from the assumptions that $\gamma \leq \frac{1}{2 \phi_c(r; \mS^{\sharp}, \mA^{\sharp} )}$ and that $B^{(t-1)} \leq  \big( 2 \phi_c(r; \mS^{\sharp}, \mA^{\sharp} ) \gamma \big)^{t-1} V_{\max} \leq V_{\max}$. 
    Since $\sup_{(s,a) \in \Omega^{(t)} } \big| \hat{Q}^{(t)}(s, a) - Q^*_r(s, a) \big| \leq \sup_{(s,a) \in \Omega^{(t)} } \big| \hat{Q}^{(t)}(s, a) - Q^*(s, a) \big| + \zeta_r$, it suffices to show that 
    \begin{equation*}
        \frac{3\gamma}{2} V_{\max} + \Big( \frac{(1+c_r \gamma)}{1- c_r \gamma}\gamma + 1 \Big) \zeta_r
            \leq \frac{1}{2\sqrt{|\mS^{\sharp}||\mA^{\sharp}|}} \sigma_r.
    \end{equation*}
    Since $\gamma \leq \frac{1}{2 c_r} \leq 1$, the above inequality is satisfied if the following inequality is true:
    \begin{equation}\label{eqn:approx_induction_1}
        \frac{3}{4c_r}\Big( V_{\max} + 2 \zeta_r \Big) + \zeta_r \leq \frac{\sigma_r}{2\sqrt{|\mS^{\sharp}||\mA^{\sharp}|}}
    \end{equation}
    
    We observe that $\sigma_r > 0$ by definition of $r$-anchor states/actions, cf. Definition \ref{defn:r_anchor}. Also, recall from \eqref{eqn:phi_c} that 
    \begin{align*}
        c_r &= \bigg( 6\sqrt{2} \frac{ \sqrt{|\mS^{\sharp}||\mA^{\sharp}|}}{\sigma_r} + 2(1+\sqrt{5}) \Big(             \frac{\sqrt{|\mS^{\sharp}||\mA^{\sharp}|}}{\sigma_r } \Big)^2 \bigg) V_{\max} \\
            &\geq 6 \frac{ \sqrt{|\mS^{\sharp}||\mA^{\sharp}|}}{\sigma_r} \bigg( 1 + \frac{ \sqrt{|\mS^{\sharp}||\mA^{\sharp}|}}{\sigma_r} \bigg) V_{\max}\\
            &> 0.
    \end{align*} 
    Therefore, \eqref{eqn:approx_induction_1} is satisfied if
    \[
        1 + 2 \frac{\zeta_r}{V_{\max}} \leq \bigg( \frac{1}{2}\frac{\sigma_r}{\sqrt{|\mS^{\sharp}||\mA^{\sharp}|}} - \zeta_r \bigg) 8 \frac{ \sqrt{|\mS^{\sharp}||\mA^{\sharp}|}}{\sigma_r} \bigg( 1 + \frac{ \sqrt{|\mS^{\sharp}||\mA^{\sharp}|}}{\sigma_r} \bigg).
    \]
    We introduce a variable $X = \frac{\sqrt{|\mS^{\sharp}||\mA^{\sharp}|}}{\sigma_r}$ (which is always positive) to rewrite this inequality as
    \begin{equation*}
        8 \zeta_r X^2 + (8\zeta_r - 4 ) X - 3 + 2\frac{\zeta_r}{V_{\max}} \leq 0.
    \end{equation*}
    
    It is easy to see that the above inequality is satisfied if
    \[
        X_- \leq X \leq X_+
    \]
    where $X_{\pm} = \frac{-(4\zeta_r - 2) \pm \sqrt{ (4\zeta_r - 2)^2 + 8\zeta_r ( 3 - 2\frac{\zeta_r}{V_{\max}} ) }}{8\zeta_r}$. As $X_- < 0$ and $X > 0$, we can conclude that \eqref{eqn:approx_induction_1} is satisfied if $X \leq X_+$, which is equivalent to the condition $0 \leq \zeta_r \leq \frac{4X+3}{8 X^2 + 8X + \frac{2}{V_{\max}}}$. 
    By assumption, $\zeta_r \leq \frac{1}{2X + 1 + \frac{1}{V_{\max}}}$ and therefore, $\frac{4X + 3 }{\zeta_r} \geq (4X+3)(2X + 1 + \frac{1}{V_{\max}}) = 8X^2 + (10 + \frac{4}{V_{\max}})X + \frac{3}{V_{\max}} + 3 \geq 8 X^2 + 8X + \frac{2}{V_{\max}}$. Thus, $0 \leq \zeta_r \leq \frac{4X+3}{8 X^2 + 8X + \frac{2}{V_{\max}}}$ is satisfied. Consequently, \eqref{eqn:approx_induction_1} is also satisfied, and the first inequality is proved.

    \item
    To prove the second claim, We revisit Eqs. \eqref{eqn:generic_interpol}, \eqref{eqn:generic_me}, \eqref{eqn:generic_mcla} in the proof of Theorem \ref{thm:generic}.
    Note that nothing has changed for Step 4 (interpolation) and we obtain the same upper bound as in \eqref{eqn:generic_interpol}:
    \begin{equation}\label{eqn:generic_interpol_approx}
        \sup_{ (s,a) \in \mS \times \mA } \big| Q^{(t)}(s,a) - Q^*(s,a) \big|
            \leq \max_{(s,a) \in \mS^{(t)} \times \mA^{(t)}} \big| \bar{Q}^{(t)}(s, a) - Q^*(s, a) \big| 
            + 2 L \beta^{(t)}.
    \end{equation}

    For Step 3 (matrix completion), it follows from Proposition \ref{prop:apprx_r}:
    \begin{equation}\label{eqn:generic_me_approx}
        \max_{(s,a) \in \mS^{(t)} \times \mA^{(t)}} \big| \bar{Q}^{(t)}(s, a) - Q^*(s, a) \big| 
            \leq c_r \Big( \max_{(s,a) \in \Omega^{(t)} } \big| \hat{Q}^{(t)}(s, a) - Q^*(s, a) \big| + \zeta_r \Big) + \zeta_r.
    \end{equation}
    The above inequality follows from the observation that
    \begin{align*}
        \max_{(s,a) \in \Omega^{(t)} } \big| \hat{Q}^{(t)}(s, a) - Q^*_r(s, a) \big|
            &\leq \max_{(s,a) \in \Omega^{(t)} } \big| \hat{Q}^{(t)}(s, a) - Q^*(s, a) \big| + \zeta_r,\\
        \max_{(s,a) \in \mS^{(t)} \times \mA^{(t)}} \big| \bar{Q}^{(t)}(s, a) - Q^*(s, a) \big|
            &\leq \max_{(s,a) \in \mS^{(t)} \times \mA^{(t)}} \big| \bar{Q}^{(t)}(s, a) - Q^*_{{r}}(s, a) \big| + \zeta_r.
    \end{align*}

    Lastly, applying Lemma \ref{lem:one_step_lookahead} and taking union bound over $(s,a) \in \Omega^{(t)}$, we can show that
    \begin{equation}\label{eqn:generic_mcla_approx}
        \max_{(s,a) \in \Omega^{(t)} } \big| \hat{Q}^{(t)}(s, a) - Q^*(s, a) \big|
            \leq \gamma \left( B^{(t-1)} + \sqrt{ \frac{2 V_{\max}^2}{N^{(t)}} \log\bigg(\frac{2 | \Omega^{(t)} | T}{\delta}\bigg)} \right)
    \end{equation}
    with probability at least $1-\frac{\delta}{T}$.
    
    Combining Eqs. \eqref{eqn:generic_interpol_approx}, \eqref{eqn:generic_me_approx}, \eqref{eqn:generic_mcla_approx} yields that the following holds with probability at least $1-\frac{\delta}{T}$
    \begin{align*}
        B^{(t)}
            &\leq c_r \Bigg\{ \gamma \left( B^{(t-1)} + \sqrt{ \frac{2 V_{\max}^2}{N^{(t)}} \log\bigg(\frac{2 | \Omega^{(t)} | T}{\delta}\bigg)} \right) + \zeta_r \Bigg\} + \zeta_r + 2L \beta^{(t)}\\
            &= c_r \gamma B^{(t-1)} + c_r \gamma (2 c_r \gamma )^{t-1} V_{\max} + ( c_r \gamma + 1) \zeta_r\\
            &\leq c_r \gamma \bigg\{ \big( 2 c_r \gamma \big)^{t-1} V_{\max} + (1 + c_r \gamma ) \zeta_r \sum_{i=1}^{t-1} \big( c_r \gamma \big)^{i-1} \bigg\} + c_r \gamma (2 c_r \gamma )^{t-1} V_{\max} + ( c_r \gamma + 1) \zeta_r\\
            &= \big(2 c_r \gamma \big)^t V_{\max} + (c_r \gamma + 1) \zeta_r \sum_{i=1}^{t} \big( c_r \gamma \big)^{i-1}.
    \end{align*}
\end{itemize}

It remains to certify that for $t = 0, \ldots, T-1$,
\[
    B^{(t)} \leq \big( 2 c_r \gamma \big)^{t} V_{\max} + (1 + c_r \gamma) \zeta_r \sum_{i=1}^{t} \big( c_r \gamma \big)^{i-1},
\] 
First of all, $Q^{(0)}(s,a) \equiv 0$ by assumption, and hence, $B^{(0)} \leq V_{\max}$. Thus, by the second inequality, this condition is satisfied for all $t = 1, \ldots, T-1$ and the proof is complete.
\end{proof}

\section{Additional Discussions on RL and ME}\label{sec:discussion}
%{\bf Discussion} 
%\red{DG: will write it down here first, and move it to the Appendix later, if necessary. will add some references later.}
\subsection{Reinforcement Learning}
Our work is motivated by the need to improve efficiency of RL algorithms for problems with continuous state and action space, where literature results are scarce. As a byproduct of our analysis, the resulting ``low-rank'' algorithm can also be reduced to settings where one of the spaces is finite or both. We offer a high-level comparison in Table~\ref{tab:one} with a few selected work from literature to help readers see how our approach fares with others from literature. This is by no means a complete illustration, given the vast literature on the finite settings. 

We remark that Table~\ref{tab:one} is not aimed at a strict comparison on sample complexity since each work focuses on different problem settings. Rather, we intend to convey a rough sense of how our efficient algorithm performs in the setting with finite spaces, and especially what we gain in sample complexity with exploiting low-rank structure. In continuous state and action, our algorithm effectively removes the dependence on the smaller dimension by leveraging the low-rank factorization. The same heuristic in fact carries over to the finite cases, where the dependence on the size of smaller space is ``removed,'' i.e., the sample complexity depends on $|\mS|$ instead of $|\mS|| \mA|$, assuming $|\mS| \geq |\mA|$. That is, exploitation of low-rank structure consistently benefits the sample complexity of our method in the same manner for all three settings.

In Table~\ref{tab:one}, we include two work per setting selected from literature (except the setting with continuous $\mS$ \& continuous $\mA$ where we were not able to find an appropriate work to compare with).
This is because there are extremely various problem settings considered in literature, which involve different technical conditions, partly due to the long history involving finite spaces. For example, between the two work selected for continuous $\mS$ and finite $\mA$, \cite{shah2018q} considers learning the $Q$-function in a single sample path, whereas \cite{yang2019theoretical} considers learning the $Q$-function with sparse neural networks when re-sampling i.i.d. transitions is possible. Learning from a single sample path is harder than the other setup, and hence, leads to a sample complexity of $\tilde{O}(\frac{1}{\epsilon^{d+3}})$

For problems with finite $\mS$ and finite $\mA$, there has been a great effort in learning an $\epsilon$-optimal policy instead of just learning an $\epsilon$-optimal value function. In this context, a line of work~\cite{sidford2018near, sidford2018variance} attempted to improve the dependence on the term $1/({1-\gamma})$ in sample complexity and recently this question is addressed in~\cite{sidford2018near} by achieving an {\small $\tilde{O}(\frac{|\mS||\mA|}{(1-\gamma)^3\epsilon^2})$}~ upper bound that matches the lower bound from \cite{azar2013minimax}. 
Regardless, traditional results on learning $\epsilon$-optimal policy/value commonly scale as the product $|\mS||\mA|$. The main message we want to convey with Table~\ref{tab:one} in the setting is that the dependence of sample complexity on the size of state/action space can be significantly improved from $|\mS||\mA|$ to $\max\{|\mS|,|\mA|\}$ by exploiting the low-rank structure of $Q$-function.

\subsection{Matrix Estimation}
\label{app:subsec_me_discussion}
We analyze the performance of our proposed algorithm in a decoupled fashion, controlling the worst-case error (for a high-probability event w.r.t. the randomness in sampling step). For the success of our analysis, it is imperative for the matrix estimation subroutine to satisfy Assumption \ref{assu:me} with the two constants $\C, \ceta$ as small as possible. The assumption ensures that the matrix estimation method in use does not amplify the $\ell_{\infty}$ error too wildly.

\medskip\noindent
{\bf Why Existing Methods Fail.}
Matrix estimation has been a popular topic of active research for the last few decades, which culminated in the low-rank matrix completion via convex relaxation of rank minimization \cite{recht2010guaranteed, candes2009exact, candes2010power}. Also, various algorithms for matrix completion/estimation -- including singular value thresholding \cite{keshavan2010matrix, chatterjee2015matrix} and nuclear-norm regularization \cite{candes2009exact, candes2010power, koltchinskii2011nuclear} -- have been proposed and analyzed with provable guarantees. Despite the huge success in both theory and practice, the available analysis for those existing methods only provides a handle on the error measured in Frobenius norm and a few other limited class of norms (Schatten norms, regularizing norm and its dual, etc.) under certain circumstances (Chapters 9-10, \cite{wainwright2019high}). In particular, there are no satisfactory results so far that provide a control on the $\ell_{\infty}$ error of matrix estimation, to the best of our knowledge. 

Recently, the convergence guarantees for the so-called Burer-Monteiro approach, which takes low-rank factor matrices as decision variables (also commonly referred to as ``nonconvex optimization’' in literature), has been actively studied in pursuit of developing a computationally more efficient alternative of convex program-based approaches \cite{ge2016matrix, chen2019noisy}. For example, \cite{chen2019noisy} provides an $\ell_{\infty}$ guarantee under certain setup. However, they assumed i.i.d. zero-mean noise and requires a proper initialization at the ground truth (for analysis).
As a result, we were not able to use existing ME methods and their analysis in this work.

We do not believe this is an algorithmic failure of the ME methods, but it is rather a limitation stemming from the disparity between the traditional analysis in ME and the needs in RL application. For example, considering the error in Frobenius norm is natural in the ME tradition for several reasons, but that analysis is not sufficient for applications where entrywise error is more important. Moreover, it seems manageable, but is not straightforward at once how the mathematical conditions for matrix recovery in ME literature will translate in the context of reinforcement learning. For example, the finite-dimensional incoherence condition between the principal subspaces and the measurement in matrix estimation could translate to a similar infinite-dimensional version of incoherence condition, but some efforts would be needed to reforge existing ME analysis to fit in RL applications seamlessly.

\medskip\noindent
{\bf Why Our ME Method Works.}
Instead, we develop an alternative matrix estimation subroutine, which is simple, yet sufficiently powerful for our RL task, thereby enabling us to achieve the ultimate conclusion for the RL problem of interest. 
The proposed method is amenable to $\ell_{\infty}$-error analysis facilitated by matrix algebra (see Proposition \ref{prop:rankr} and its proof). 
At first glance, our proposal seems extremely simple, and one might doubt its efficacy, worrying about its numerical stability, etc. because it involves the pseudoinverse of a matrix. 
That concern is partly true, but indeed, there are two key factors that make our method work for the problem of our interest. 

First, we assume the existence of ``anchor'' states and actions, which contain all necessary information for the global recovery of $Q^*$.  
From a theoretical point of view, this assumption is related to the eigengap condition and the incoherence condition between eigenspace and the sampling operator, which are commonly assumed in existing ME literature.
From a  practical perspective, this means the existence of faithful representatives that reflect the ``diversity’' of states and actions, which is the case in many real-world applications.

Second, we are not only passively fed with data, but can actively decide which data to collect. Note that our algorithm requires full measurement for the two cylinders (rectangles when represented as a matrix) corresponding to the anchor states and anchor actions without any missing values in them. This is feasible by adaptive sampling, which is not achievable by random sampling. As a byproduct, active sampling allows us to get rid of the spurious log term that appears in sample complexity of existing ME methods as a result of random sampling.

All in all, our ME method is expected to perform reasonably well in the setup considered in this work. 
We confirm this is the case with experiments (see Section \ref{sec:empirical} and Appendix \ref{appendix:control setup}).

\medskip\noindent
{\bf Open Questions for Future Work.}
We have seen that the proposed ME method is successful in the extremely sample deficient setting where $|\Omega^{(t)}| \asymp \max\{|\mS^{(t)}|, |\mA^{(t)}|\}$. 
However, it seems other existing ME methods based on convex programs also work similarly well, which cannot be expained with the current analysis. 

As a matter of fact, when the computation cost is ignored, convex-relaxation-based approach is widely accepted as the best one in terms of robustness. 
This is glimpsed by the evolution of $\ell_{\infty}$ error in our experiments; unlike the fluctuations observed in our method and soft-impute, 
the error steadily decreases for the nuclear norm minimization. Also, we believe existing ME methods can perform better as $|\Omega^{(t)}|$ becomes larger. We have observed that our simple ME method is most efficient in the sample-deficient setting where $|\Omega^{(t)}| \asymp \max\{|\mS^{(t)}|, |\mA^{(t)}|\}$, but we do not know if the same conclusion will still hold as $|\Omega^{(t)}|$ increases. 

Therefore, how to harmonize existing ME methods and the low-rank RL task we consider in this paper would be an exciting open question. 
This question might be tackled either by devising new proof techniques to obtain stronger error guarantees for existing ME methods 
or by improving our decoupled error analysis for RL iteration developed in this paper. We believe both directions are promising and 
it would be a valuable contribution to make progress in either direction.

\section{Experimental Setup for Stochastic Control Tasks}
\label{appendix:control setup}

In this section, we formalize the detailed settings for several stochastic control tasks we use. Following previous work~\cite{russ2019,yang2020harnessing}, we briefly introduce the background for each task, and then present the system dynamics as well as our simulation setting.
For consistency, we follow the dynamics setup in \cite{russ2019,yang2020harnessing}, while adding additionally a noise term $\mathcal{N}$ to one dimension of the state dynamics.

\medskip\noindent
{\bf General Setup.} 
We first discretize the state space and the the action space into a fine grid and run standard value iteration to obtain a proxy of $Q^*$. Subsequently, when measuring the $\ell_\infty$ error, we take the max (absolute) difference between our estimate $Q^{(t)}$ and the proxy of $Q^*$ over this fine grid. For the mean error, we use the average of the (absolute) difference over this grid. For anchor states and actions, we simply select $r$ states and $r$ actions that are well separated in their respective space. To do so, we divide the space uniformly into $r$ parts and then select a state/action from each part randomly.
We use $r=10$ in all experiments. In terms of the comparison with different Matrix Estimation methods, we note that as mentioned, the sampling procedure is different: traditional methods often work by independently sampling each entry with some fixed probability $p$, while our method explores a few entire rows and columns. We hence control all the ME methods to have the same number of observations (i.e., same size of the exploration set $\Omega^{(t)}$ as ours) at each iteration, but switch to independent sampling for the traditional methods.

\medskip\noindent
{\bf Inverted Pendulum.}
In this control task, we aim to balance an inverted pendulum on the equilibrium position, i.e., the upright position \cite{russ2019}.
% \red{Following descriptions are to be modified}
The angle and the angular speed tuple, $(\theta, \dot{\theta})$, describes the system dynamics, which is formulated as follows \cite{sutton2018reinforcement}:
% . Denote $\tau$ as the time interval between decisions, $u$ as the torque input on the pendulum, the dynamics can be written as~\cite{ong2015value,sutton2018reinforcement}:
\begin{align*}
    & \theta := \theta + \dot{\theta}~\tau,\\
    & \dot{\theta} := \dot{\theta} + \left( \sin{\theta} - \dot{\theta} + u \right)\tau + \mathcal{N}(\mu,\sigma^2),
\end{align*}
where $\tau$ is the time interval between decisions, $u$ denotes the input torque on the pendulum, and $\mathcal{N}$ refers to the noise term we added with mean $\mu$ and variance $\sigma$. We formulate the reward function to stabilize the pendulum on an upright pendulum:
\begin{equation*}
    r(\theta,u) = - 0.1u^2 + \exp{\left(\cos{\theta}-1 \right)}.
\end{equation*}
In the simulation, 
% the state space is $(-\pi, \pi]$ for $\theta$ and $[-10,10]$ for $\dot{\theta}$. W
we limit the input torque in $[-1,1]$ and set $\tau=0.3$, $\mu=0$, and $\sigma=0.1$.
We discretize each dimension of the state space into 50 values, and action space into 1000 values, which forms the discretization of the optimal $Q$-value matrix to be of dimension $2500\times 1000$. 
% We follow \cite{julier2004unscented} to handle the policy of continuous states by modelling their transitions using multi-linear interpolation.

\medskip\noindent
{\bf Mountain Car.}
The Mountain Car problem aims to drive an under-powered car up to a hill~\cite{sutton2018reinforcement}. We use the position and the velocity of the car, $\left(x, \dot{x}\right)$, to describe the physical dynamics of the system. Denote $\mathcal{N}$ as the noise term added, $u$ as the acceleration input on the car, we can express the system dynamics as
\begin{align*}
    & x := x + \dot{x} + \mathcal{N}(\mu,\sigma^2),\\
    & \dot{x} := \dot{x} - 0.0025\cos{(3x)} + 0.001u.
\end{align*}
We define a reward function that encourages the car to drive up to the top of the hill at $x_0=0.5$:
\begin{equation*}
r(x) = \left\{
\begin{aligned}
10, & \qquad x\ge x_0,\\
-1, & \qquad \text{else}.
\end{aligned}
\right.
\end{equation*}
We follow standard settings \cite{yang2020harnessing} to limit the input $u\in[-1,1]$. We choose $\mu=0$ and $\sigma=1e^{-3}$. Similarly, the whole state space is discretized into 2500 values, and the action space is discretized into 1000 values, which translates to a discretization of $2500\times 1000$ for the optimal $Q$-value matrix. % The evaluation metric we are concerned about is the total time it takes to reach the top of the mountain, given a randomly and uniformly generated initial state.

\medskip\noindent
{\bf Double Integrator.}
We consider the Double Integrator system~\cite{ren2008consensus}, where a unit mass brick moves along the $x$-axis on a frictionless surface. The brick is controlled with a horizontal force input $u$, which aims to regulate the brick to $\bm{x}=[0,0]^T$~\cite{russ2019}. Similarly, we use the position and the velocity $\left(x, \dot{x}\right)$ of the brick to describe the physical dynamics:
\begin{align*}
    & x := x + \dot{x}~\tau + \mathcal{N}(\mu,\sigma^2),\\
    & \dot{x} := \dot{x} + u~\tau,
\end{align*}
where $\mathcal{N}$ is the noise term added. Following~\cite{russ2019}, we define the reward function using the quadratic cost formulation, which regulates the brick to $\bm{x}=[0,0]^T$:
\begin{equation*}
r(x, \dot{x}) = - \frac{1}{2} \left( x^2 + \dot{x}^2 \right).
\end{equation*}
The input torque is limited to be $u\in[-1,1]$. We again set $\tau=0.1$, $\mu=0$, and $\sigma=0.1$. Similar to the previous tasks, we obtain a discretization of $2500\times 1000$ for the optimal $Q$-value matrix, with state space discretized into 2500 values and action space discretized into 1000 values.

\medskip\noindent
{\bf Cart-Pole.}
Despite simple tasks with smaller state dimensions, we consider the harder Cart-Pole problem with $4$-dimensional state space~\cite{barto1983neuronlike}. The problem consists a pole attached to a cart moving on a frictionless track, aiming to stabilize the pole at the upright stable position. The cart is controlled by a limited force that can be applied to both sides of the cart.
To describe the physical dynamics of the Cart-Pole system, we use a 4-element tuple $(\theta, \dot{\theta}, x, \dot{x})$, corresponding to the angle and the angular speed of the pole, and the position and the speed of the cart. The dynamics can be expressed as follows:
\begin{align*}
    & \ddot{\theta} := \frac{g\sin{\theta} - \frac{u+ml \dot{\theta}^2 \sin{\theta}}{m_c+m} \cos{\theta}}{l\left( \frac{4}{3} - \frac{m\cos^2{\theta}}{m_c+m} \right)},\\
    & \ddot{x} := \frac{u + ml\left( \dot{\theta}^2 \sin{\theta} - \ddot{\theta}\cos{\theta}\right) }{m_c+m},\\
    & \theta := \theta + \dot{\theta}~\tau,\\
    & \dot{\theta} := \dot{\theta} + \ddot{\theta}~\tau + \mathcal{N}(\mu,\sigma^2),\\
    & x := x + \dot{x}~\tau,\\
    & \dot{x} := \dot{x} + \ddot{x}~\tau,
\end{align*}
where $u\in[-10,10]$ denotes the input applied to the cart, $\mathcal{N}$ with $\mu=0$ nad $\sigma=0.1$ denotes the noise term, $m_c = 1kg$ denotes the mass of the cart, $m = 0.1kg$ denotes the mass of the pole, and $g=9.8m/s^2$ corresponds to the gravity acceleration.

We define the reward function similar to Inverted Pendulum that tries to stabilize the pole in the upright position:
\begin{equation*}
    r(\theta) = \cos^4{(15 \theta)}.
\end{equation*}
In the simulation, we discretize each dimension of the state space into 10 values, and action space into 1000 values, which forms an optimal $Q$-value function as a matrix of dimension $10000\times 1000$.

\medskip\noindent
{\bf Acrobot.}
Finally, we present the Acrobot swinging up task \cite{russ2019}.
The Acrobot is an underactuated two-link robotic arm in the vertical plane (i.e., a two-link pendulum), with only an actuator on the second joint.
The goal is to stabilize the Acrobot at the upright position.
The equations of motion for the Acrobot can be derived using the method of Lagrange~\cite{russ2019}. 
The physical dynamics of the system is described by the angle and the angular speed of both links, i.e., $(\theta_1, \dot{\theta_1}, \theta_2, \dot{\theta_2})$. Denote $\tau$ as the time interval, $u$ as the input force on the second joint, $\mathcal{N}$ as the noise term added, the dynamics of Acrobot can be derived as
\begin{align*}
    & D_1 := m_1 \left(l_{1}^2 + l_{c1}^2 \right) + m_2\left(l_1^2 + l_{2}^2 + l_{c2}^2 + 2 l_1 l_{c2}\cos{\theta_2}\right),\\
    & D_2 := m_2 \left(l_{2}^2 + l_{c2}^2 + l_1  l_{c2} \cos{\theta_2}\right),\\
    & \phi_2 := m_2 l_{c2}g \sin{\left(\theta_1 + \theta_2\right)},\\
    & \phi_1 := - m_2 l_1 l_{c2} \dot{\theta_2} \left(\dot{\theta_2} + 2\dot{\theta_1} \right) \sin{\theta_2} + \left(m_1 l_{c1} + m_2 l_1\right) g \sin{\theta_1} + \phi_2,\\
    & \ddot{\theta_2} := \frac{u + \frac{D_2}{D_1}\phi_1 - m_2 l_1 l_{c2} \dot{\theta_1}^2\sin{\theta_2} - \phi_2}{m_2 (l_2^2 + l_{c2}^2) - \frac{D_2^2}{D_1}},\\
    & \ddot{\theta_1} := -\frac{D_2 \ddot{\theta_2} + \phi_1}{D_1},\\
    & \theta_1 := \theta_1 + \dot{\theta_1}~\tau,\\
    & \dot{\theta_1} := \dot{\theta_1} + \ddot{\theta_1}~\tau + \mathcal{N}(\mu,\sigma^2),\\
    & \theta_2 := \theta_2 + \dot{\theta_2}~\tau,\\
    & \dot{\theta_2} := \dot{\theta_2} + \ddot{\theta_2}~\tau,
\end{align*}
where $l_1=l_2=1m$ are the length of two links, $l_{c1}=l_{c2}=0.5m$ denote position of the center of mass of both links, $m_{1}=m_{2}=1kg$ denote the mass of two links, and $g=9.8m/s^2$ denotes the gravity acceleration. $u$ corresponds to the input force applied, which is limited by $u\in[-10,10]$.

Similar to the Inverted Pendulum, we define the reward function that favors the Acrobot to stabilize at the upright unstable fixed point $\bm{x}=[\pi,0,0,0]^T$:
\begin{equation*}
    r(\bm{\theta},u) = \exp{\left(-\cos{\theta_1}-1\right)} + \exp{\left(-\cos{(\theta_1 + \theta_2)}-1 \right)}.
\end{equation*}
Since the state space of Acrobot is also 4-dimensional, we again discretize each dimension of the state space into 10 values, and action space into 1000 values, which forms discretization of the optimal $Q$-value matrix to be of dimension $10000\times 1000$.

\section{Additional Results on Stochastic Control Tasks}
\label{appendix:control results}
In this section, we provide additional results on all the 5 tasks. These include plots for sample complexity, error guarantees and visualization of the learned policies. 

\medskip\noindent
{\bf Summary of Empirical Results.} We remark that the conclusion remains the same as in the main paper (cf. Section \ref{sec:empirical}). Using our low-rank algorithm with the proposed ME method, the sample complexity is significantly improved as compared with the baselines. For the error guarantees, our ME method is very competitive, both in $\ell_\infty$ and mean error. We again note that our simple method is much more efficient in terms of computational complexity, compared to other ME methods based on optimizations.
Finally, the visualization of policies demonstrates that the learned policy, obtained from the output $Q^{(T)}$ is often very close to the policy obtained from $Q^*$, and this leads to the desired behavior in terms of performance metrics, as summarized in Table~\ref{tab:metric} of the main paper.
Overall, these consistent results across various stochastic control tasks confirm the efficacy of our generic low-rank algorithm.

\subsection{Inverted Pendulum}
%\vspace{-0.05in}
{\bf Sample Complexity and Error Guarantees.} Repeated from the main paper for completeness.
\begin{figure}[H]
\vspace{-0.16in}
\centering
\subfigure[Sample Complexity]{
   % \label{fig:ip_.9_linf_complexity}
    \includegraphics[width=0.24\textwidth]{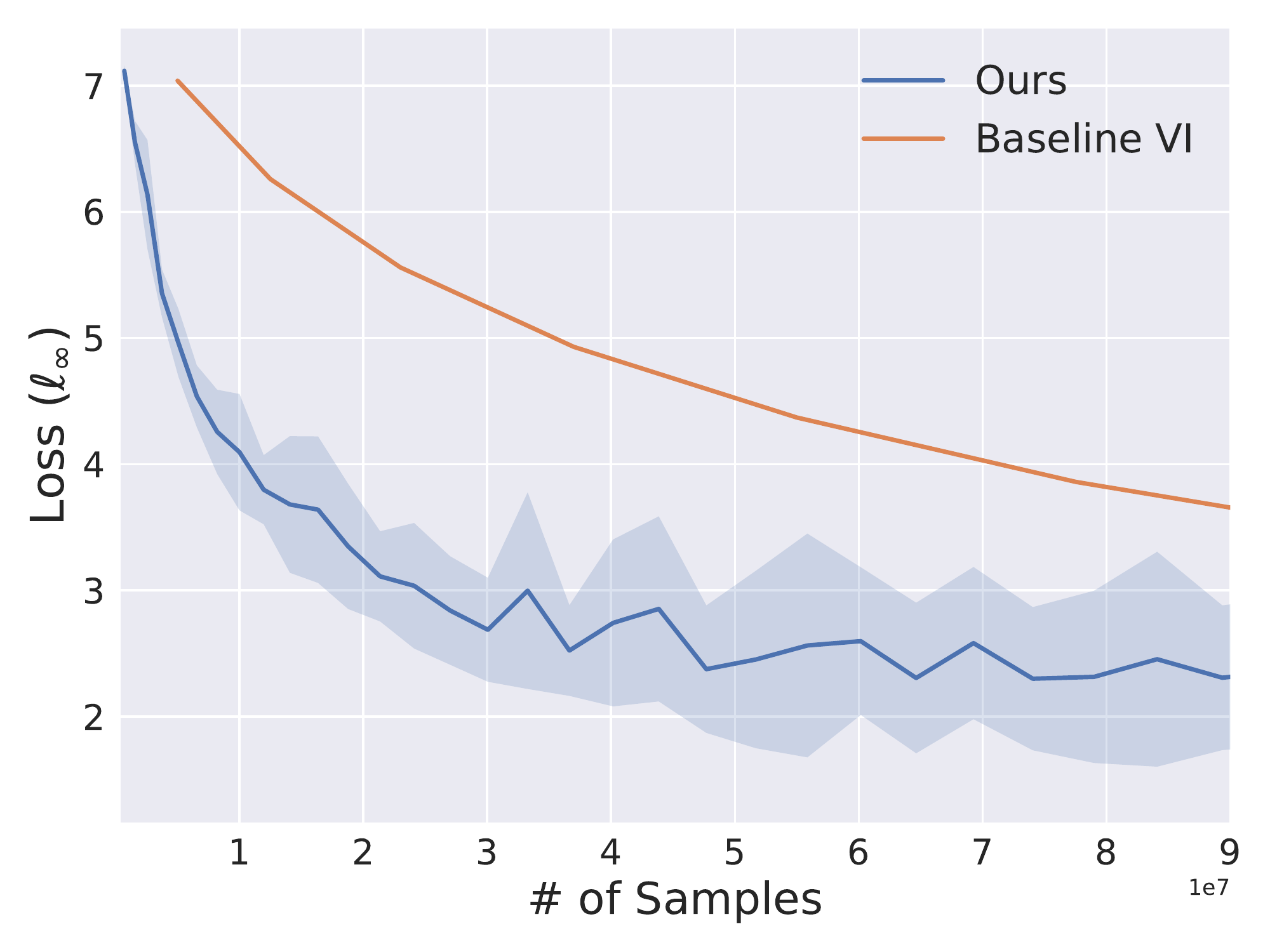}
}
\hspace{-2.3ex}
\subfigure[Sample Complexity]{
   % \label{fig:ip_.9_mean_complexity}
    \includegraphics[width=0.24\textwidth]{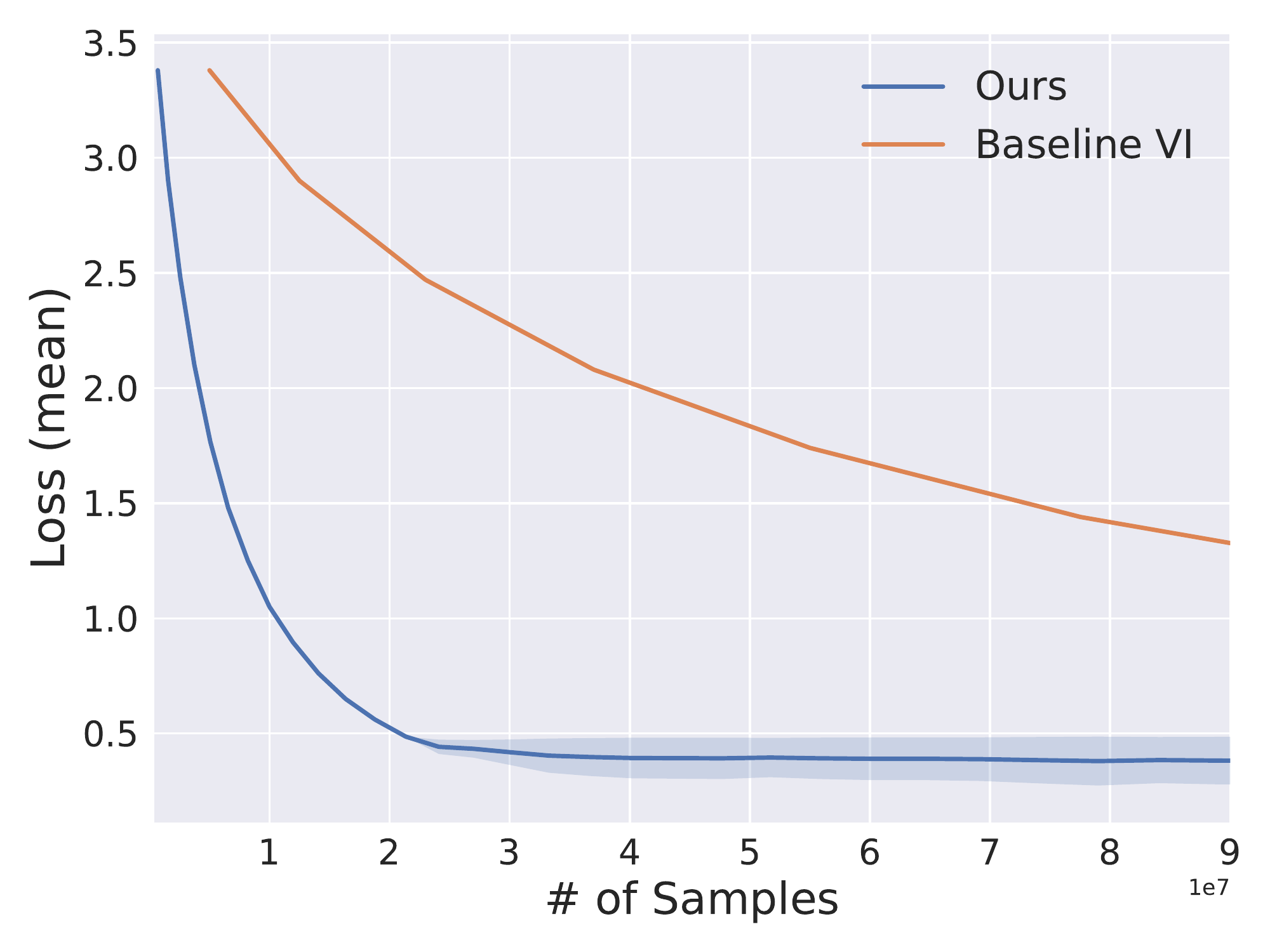}
}
\hspace{-2.3ex}
\subfigure[$\ell_{\infty}$ Errors]{
   % \label{fig:ip_.9_linf_loss}
    \includegraphics[width=0.24\textwidth]{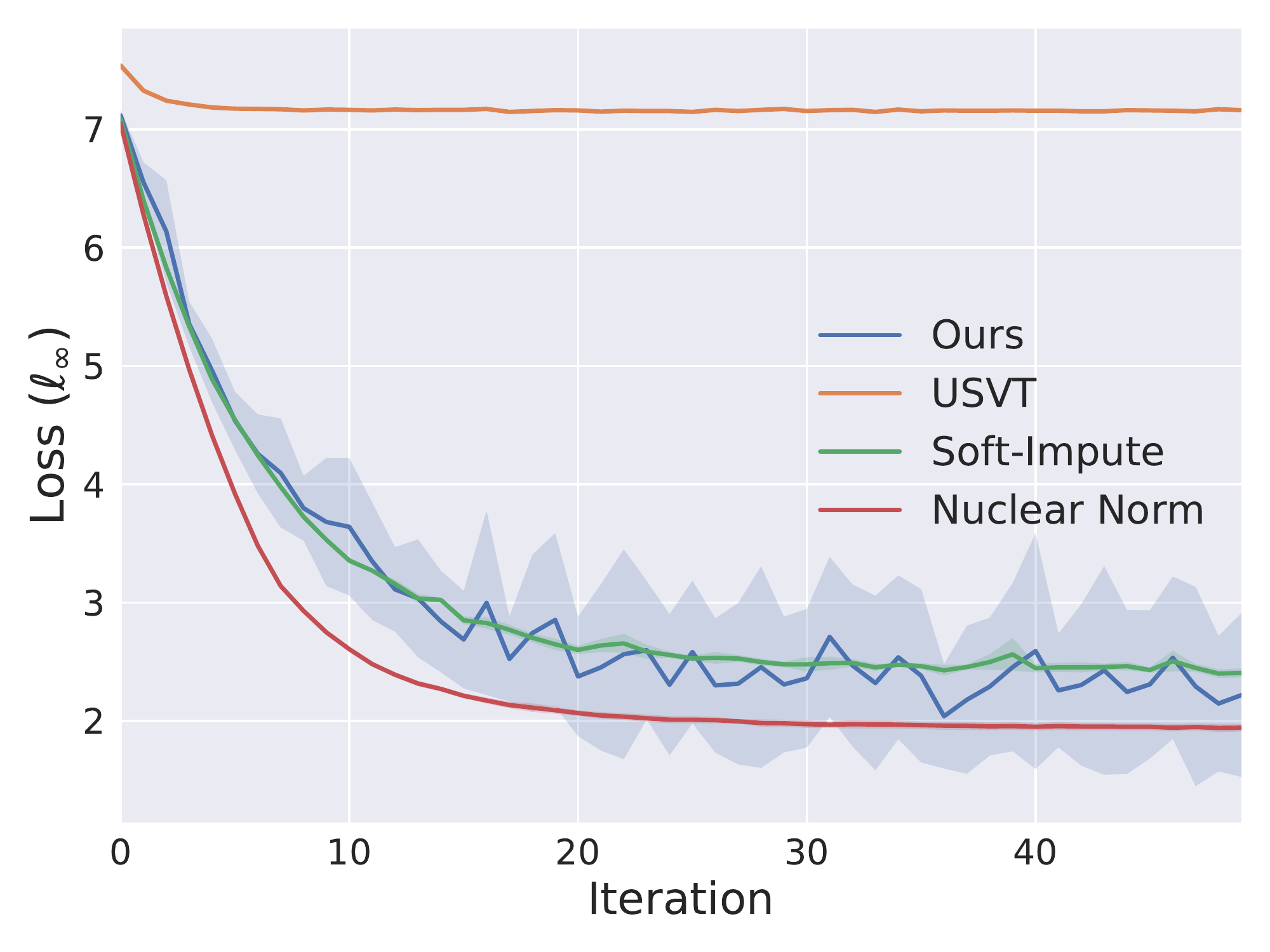}
}
\hspace{-2.3ex}
% \hfill
\subfigure[Mean Errors]{
  %  \label{fig:ip_.9_mean_loss}
    \includegraphics[width=0.24\textwidth]{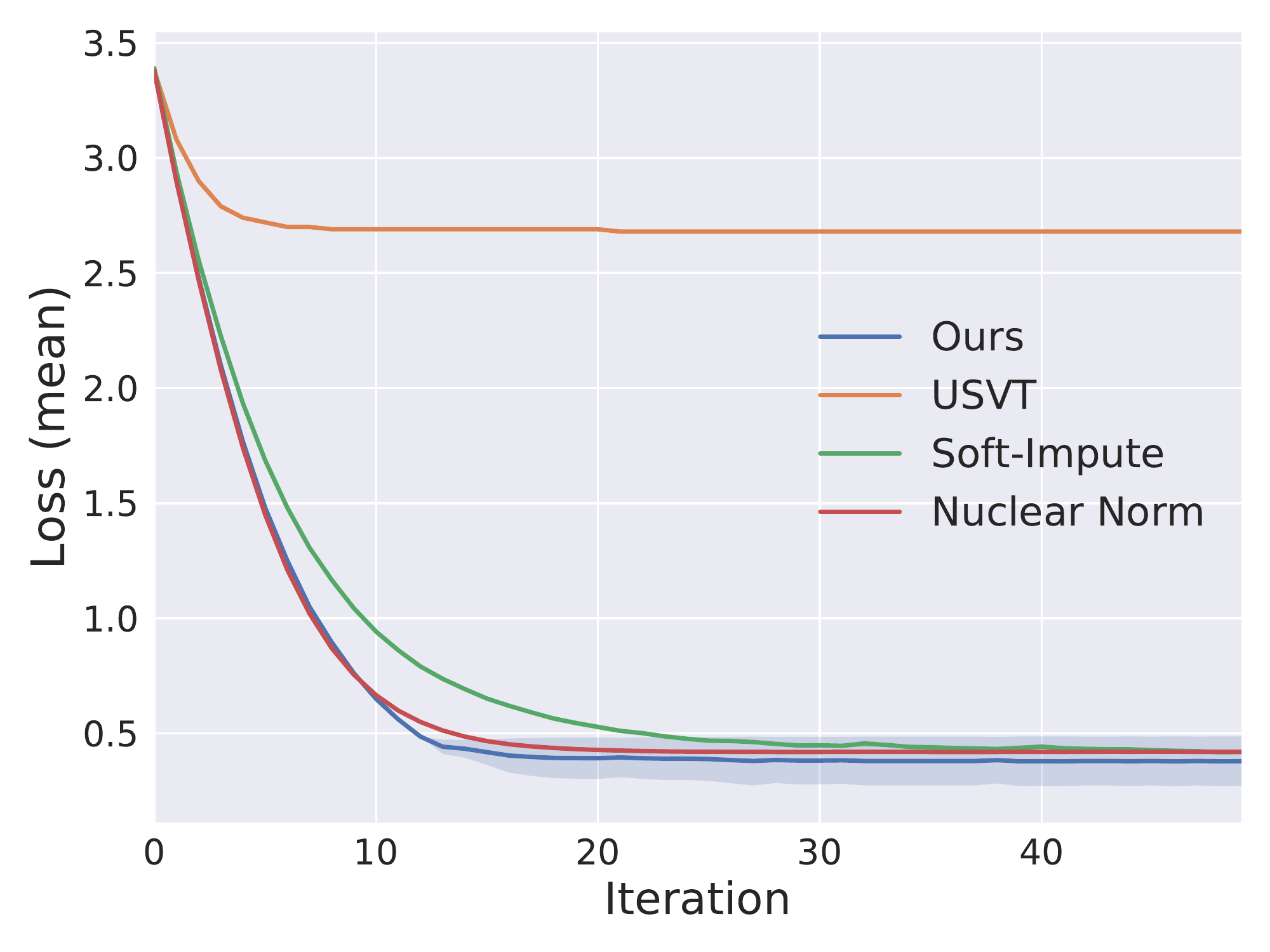}
}
% \hspace{-2.3ex}
% \subfigure[Sample complexity]{
%     \label{fig:ip_.9_mean_complexity}
%     \includegraphics[width=0.2475\textwidth]{figures/ip_gamma.9_mean_complexity.pdf}
% }
\vspace{-0.3cm}
\caption{ Empirical results on the Inverted Pendulum  control task. In (a) and (b), we show the improved sample complexity for achieving different levels of $\ell_\infty$ error and mean error, respectively. In (c) and (d), we compare the $\ell_\infty$ error and the mean error for various ME methods. Results are averaged across 5 runs for each method.
}
\vspace{-0.3cm}
\end{figure}

\medskip\noindent
{\bf Policy Visualization.}
\begin{figure}[H]
\vspace{-0.16in}
\centering
\subfigure[Optimal Policy]{
    \label{fig:policy_ip_optimal}
    \includegraphics[height=0.22\textwidth]{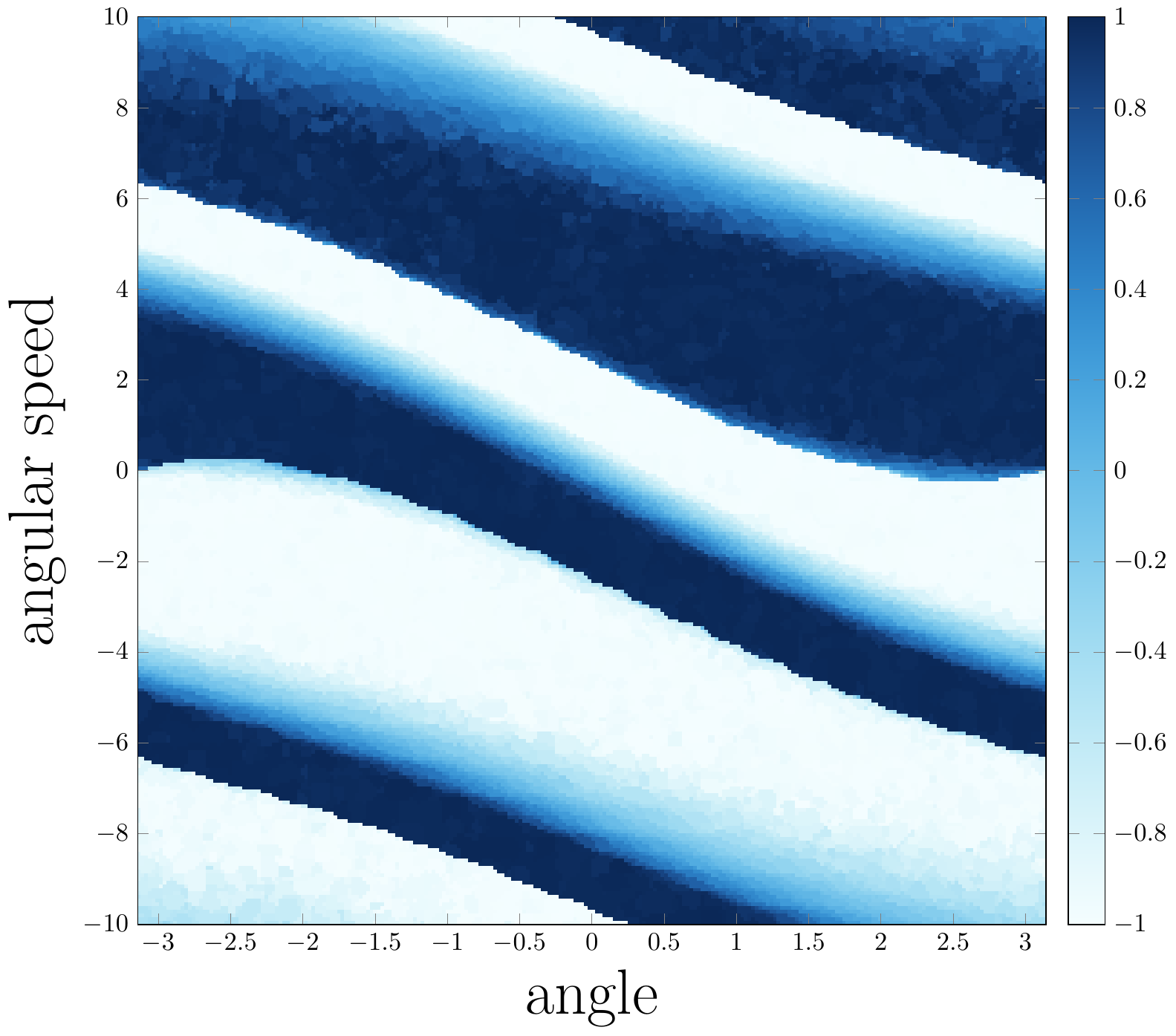}
}
\hspace{-2.48ex}
\subfigure[Soft-Impute]{
    \label{fig:policy_ip_softimp}
    \includegraphics[trim={27 0 0 0},clip,height=0.22\textwidth]{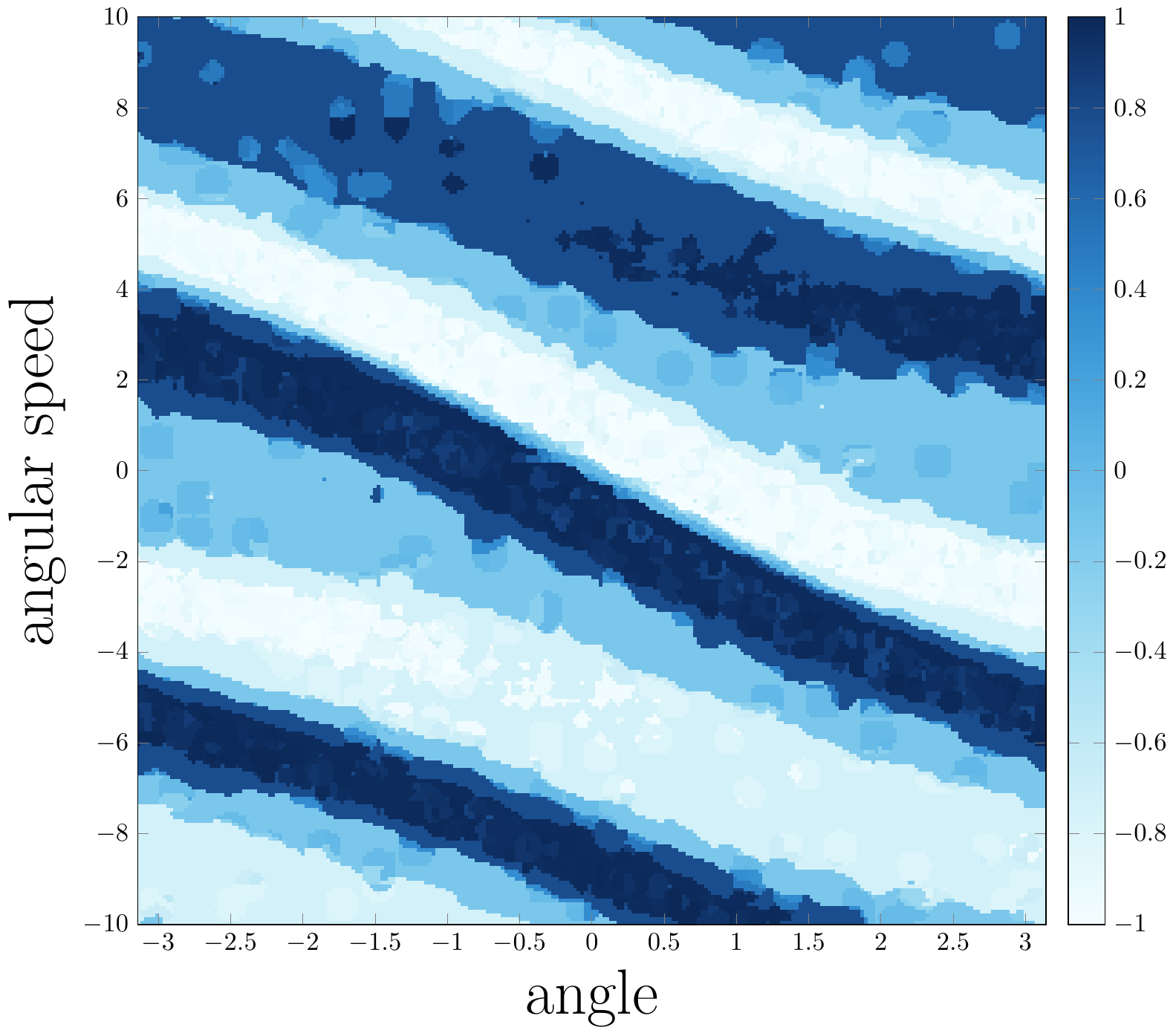}
}
\hspace{-2.48ex}
\subfigure[Nuclear Norm]{
    \label{fig:policy_ip_nucnorm}
    \includegraphics[trim={27 0 0 0},clip,height=0.22\textwidth]{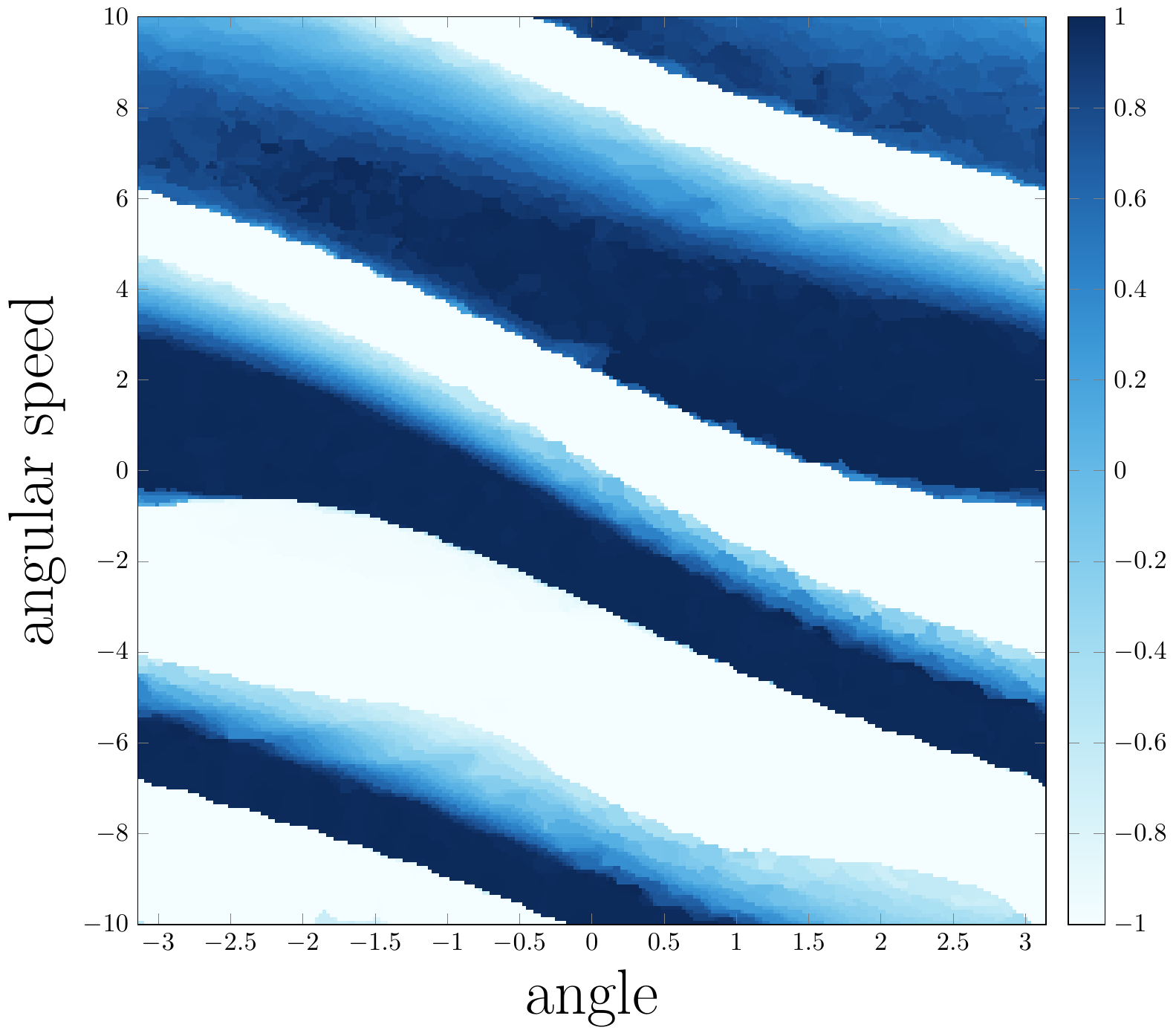}
}
\hspace{-2.48ex}
\subfigure[Ours]{
    \label{fig:policy_ip_ours}
    \includegraphics[trim={27 0 0 0},clip,height=0.22\textwidth]{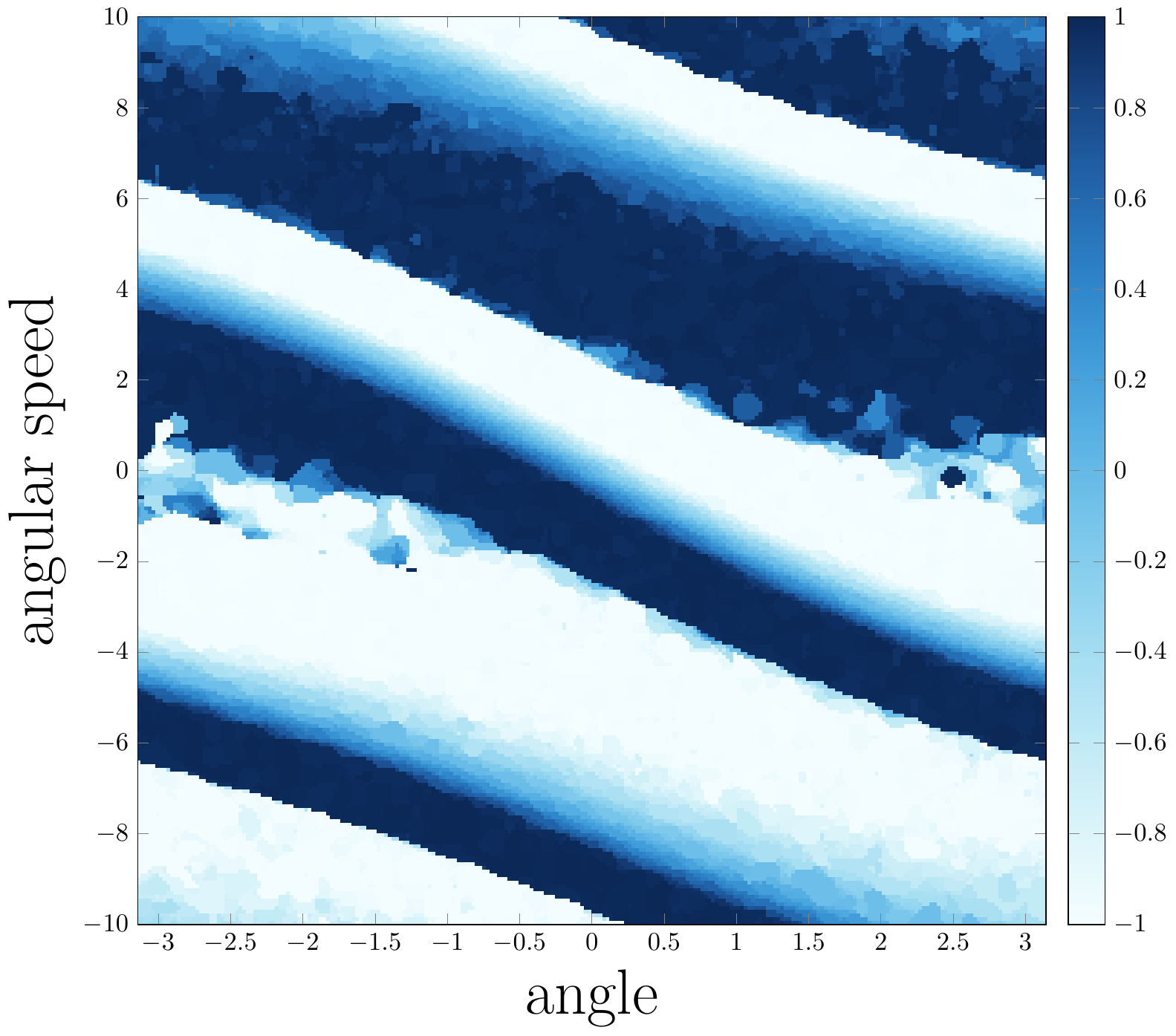}
}
%\vspace{-0.1cm}
\caption{Policy visualization of different methods on the Inverted Pendulum control task. The policy is obtained from the output $Q^{(T)}$  by taking $\arg\max_{a\in\mA}Q^{(T)}(s,a)$ at each state $s$.}
\label{fig:policy-ip}
\vspace{-0.3cm}
\end{figure}

%Recall that $\mS$ is 2-dimensional.

\subsection{Mountain Car}

{\bf Sample Complexity and Error Guarantees.}
\begin{figure}[H]
\vspace{-0.16in}
\centering
\subfigure[Sample Complexity]{
    \label{fig:mc_.9_linf_complexity}
    \includegraphics[width=0.24\textwidth]{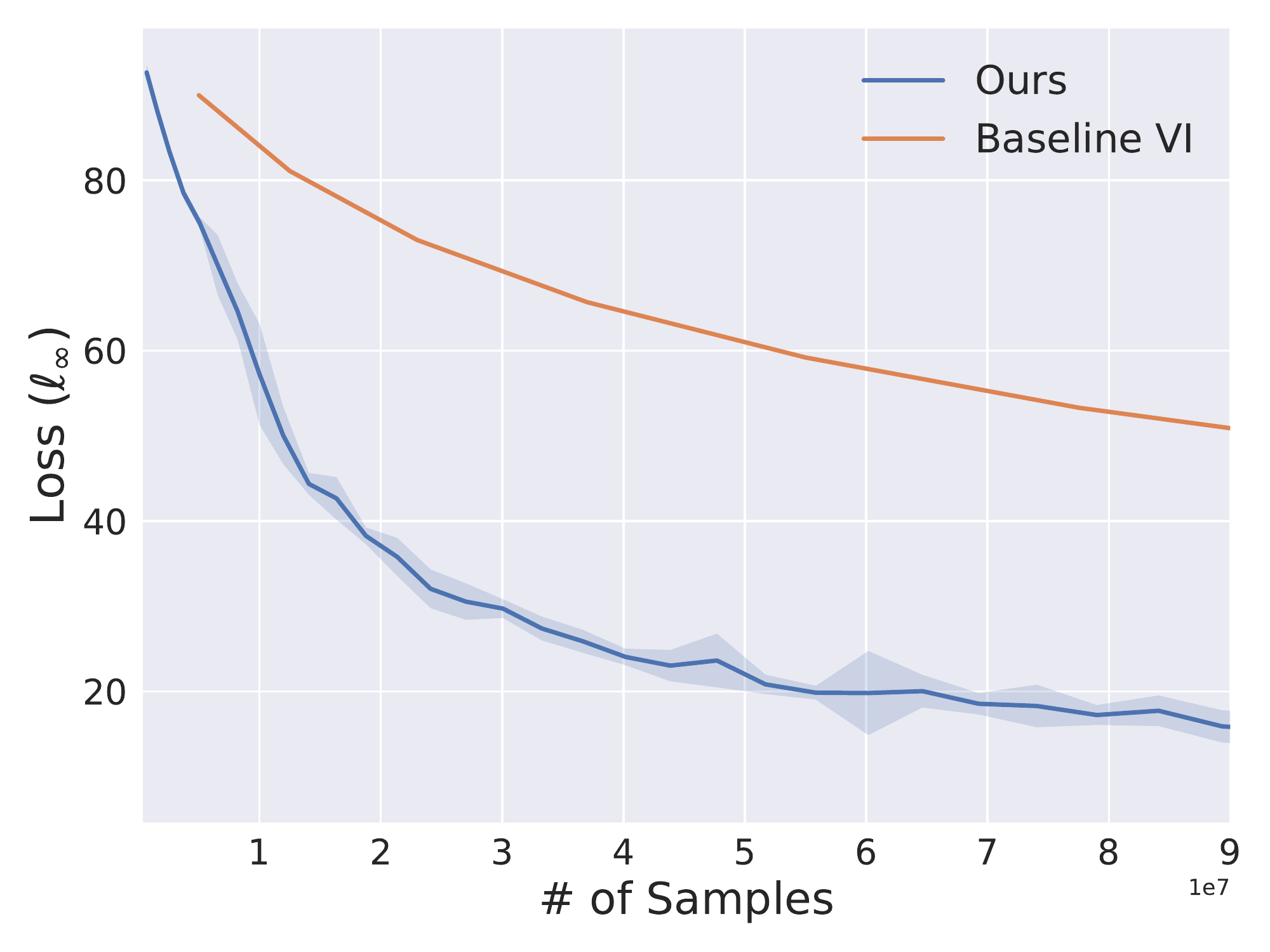}
}
\hspace{-2.3ex}
\subfigure[Sample Complexity]{
    \label{fig:mc_.9_mean_complexity}
    \includegraphics[width=0.24\textwidth]{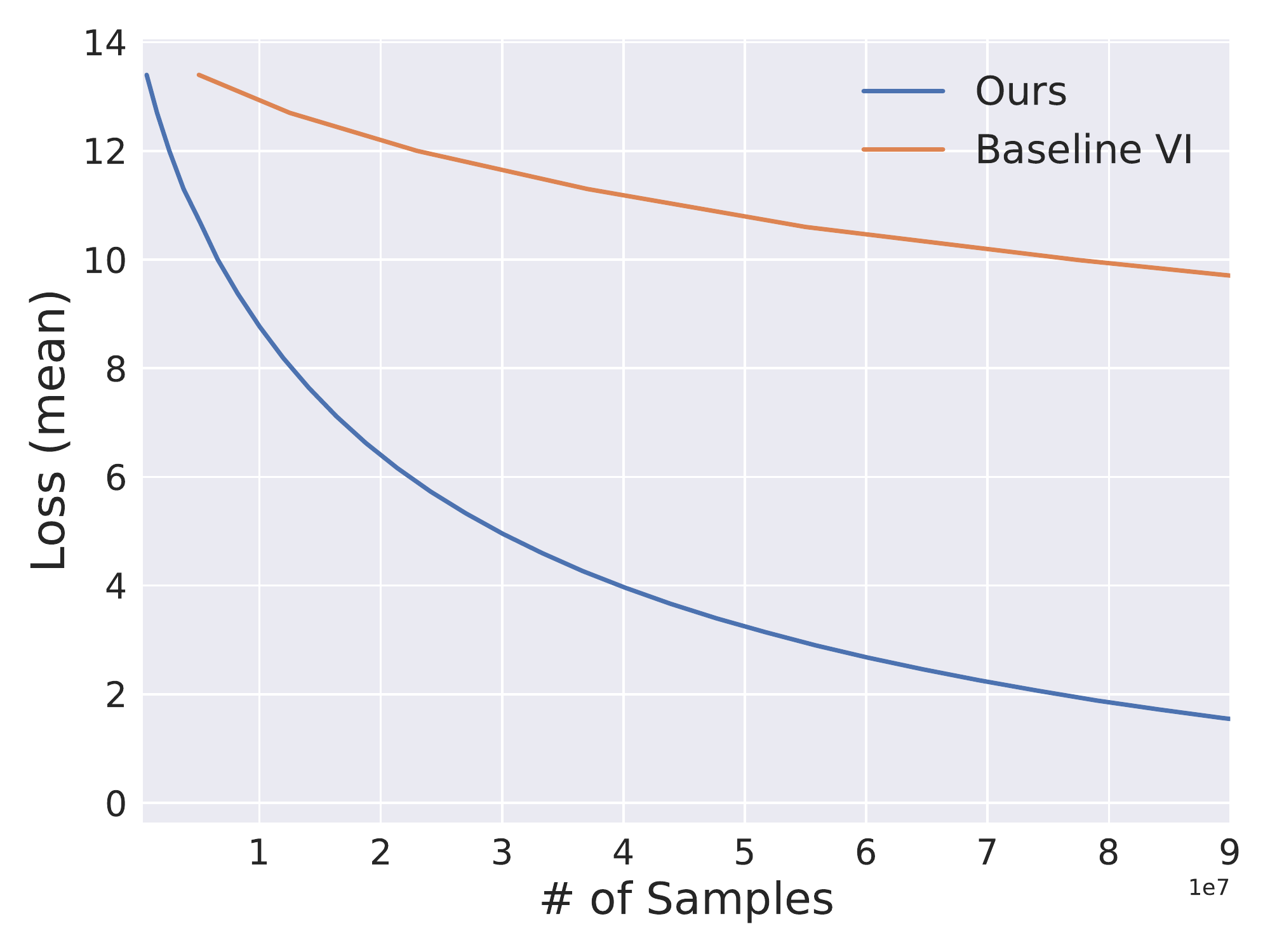}
    }
\hspace{-2.3ex}  
\subfigure[$\ell_{\infty}$ Errors]{
    \label{fig:mc_.9_linf_loss}
    \includegraphics[width=0.24\textwidth]{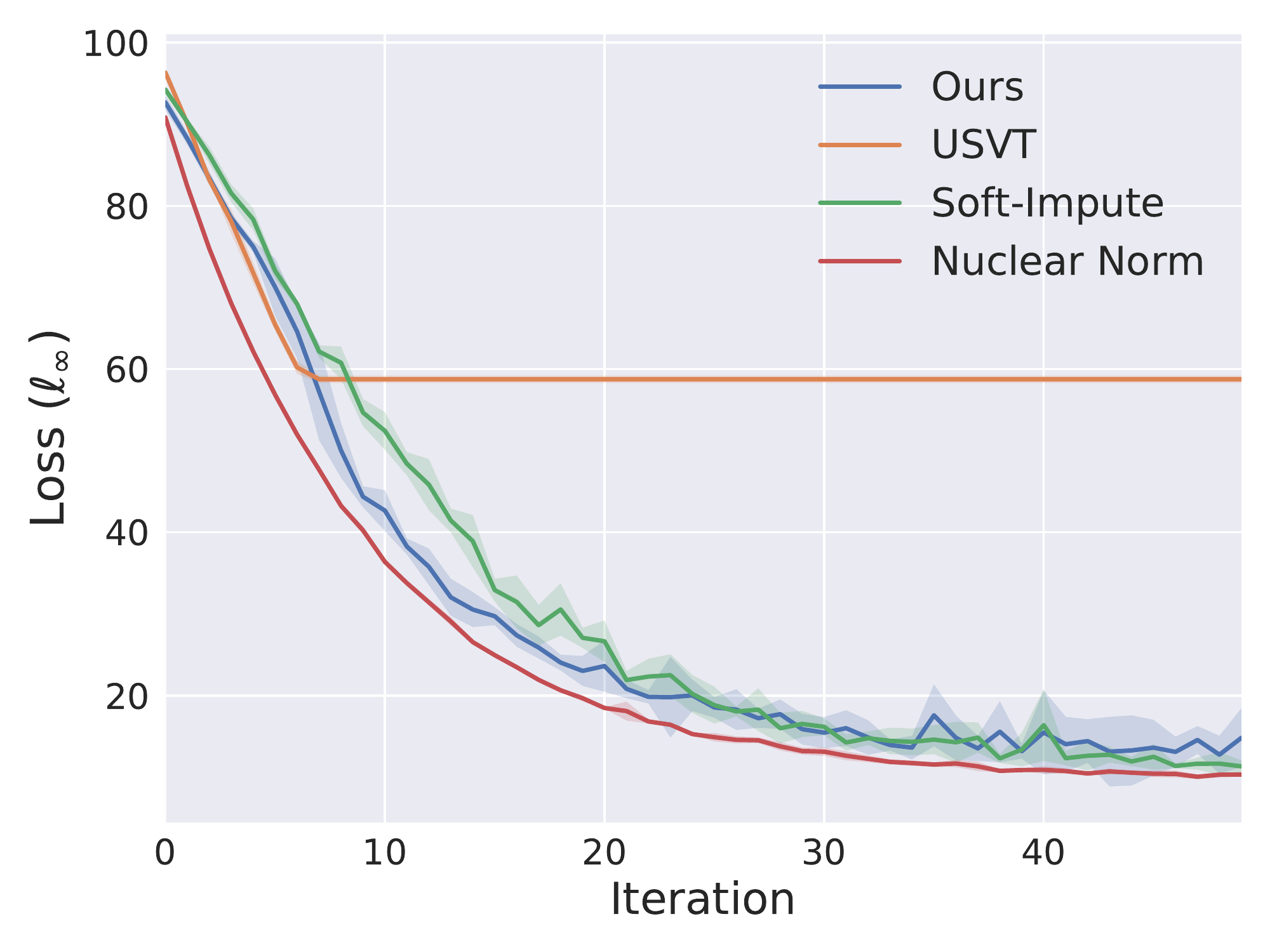}
}
\hspace{-2.3ex}
% \hfill
\subfigure[Mean Errors]{
    \label{fig:mc_.9_mean_loss}
    \includegraphics[width=0.24\textwidth]{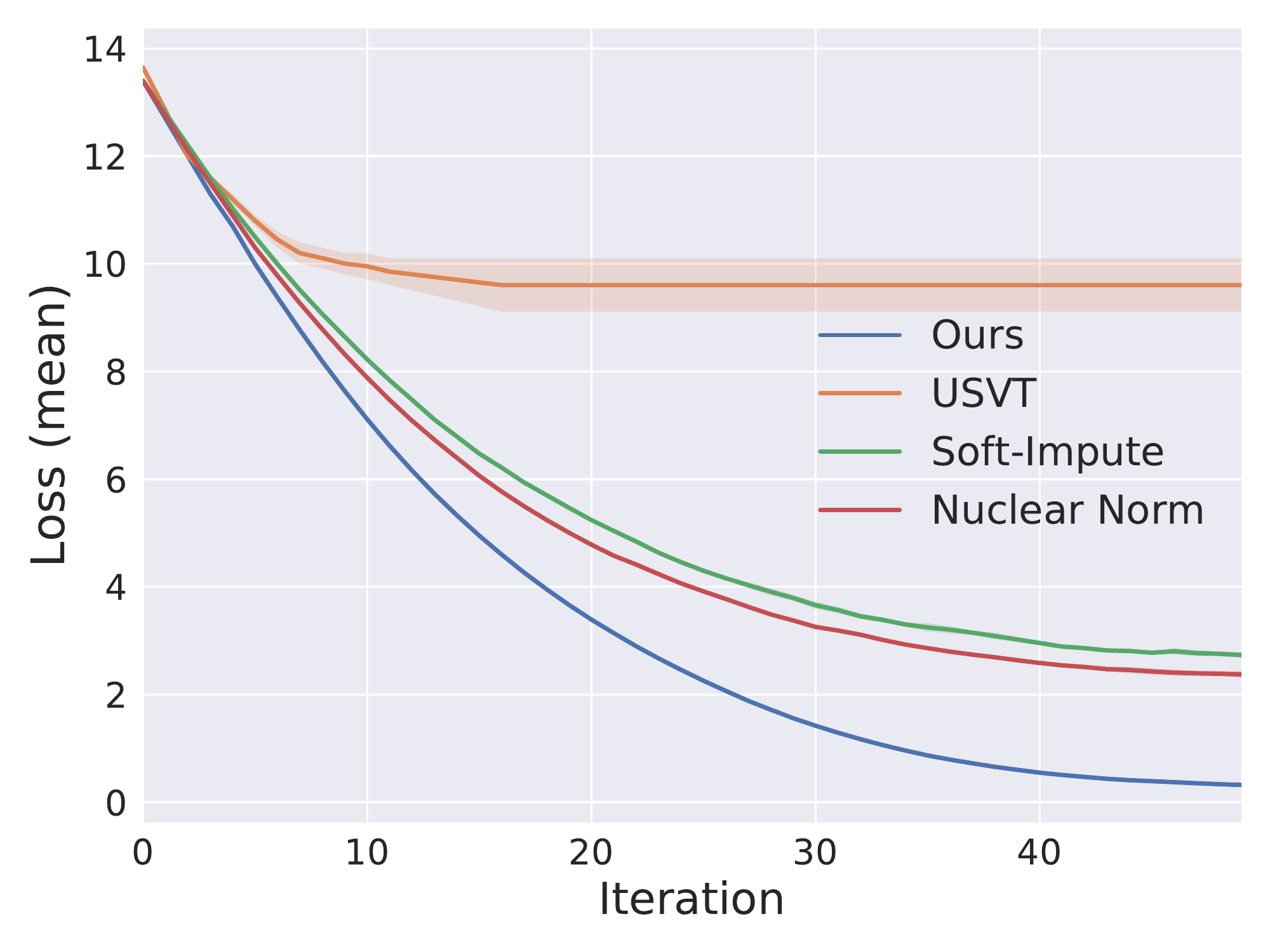}
}
\vspace{-0.2cm}
\caption{Empirical results on the Mountain Car control task. In (a) and (b), we show the improved sample complexity for achieving different levels of $\ell_\infty$ error and mean error, respectively. In (c) and (d), we compare the $\ell_\infty$ error and the mean error for various ME methods. Results are averaged across 5 runs for each method.}
\label{fig:mc_.9_main}
\vspace{-0.3cm}
\end{figure}

\medskip\noindent
{\bf Policy Visualization.}
\begin{figure}[H]
\vspace{-0.16in}
\centering
\subfigure[Optimal Policy]{
    \label{fig:policy_mc_optimal}
    \includegraphics[height=0.22\textwidth]{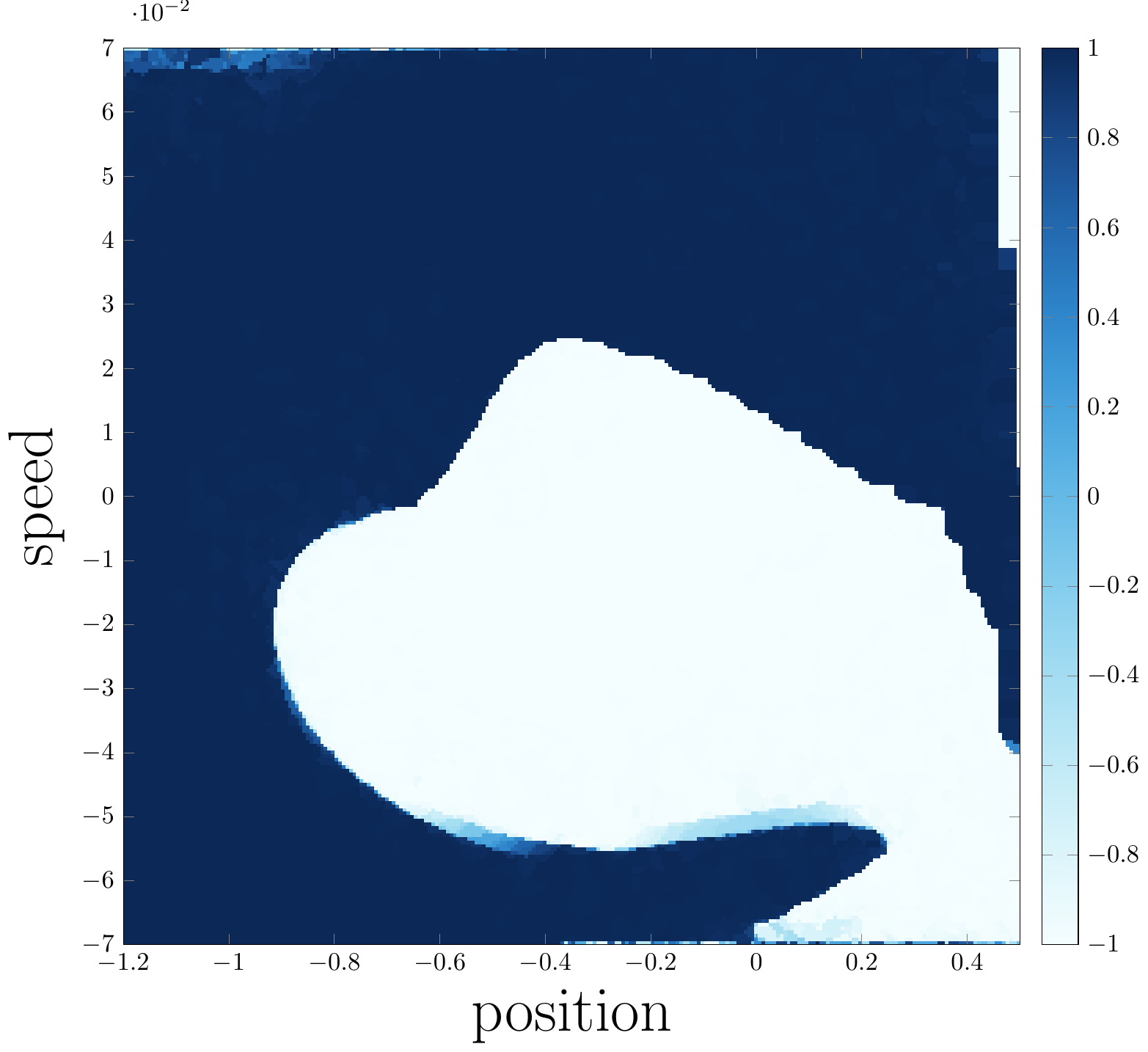}
}
\hspace{-2.48ex}
\subfigure[Soft-Impute]{
    \label{fig:policy_mc_softimp}
    \includegraphics[trim={27 0 0 0},clip,height=0.22\textwidth]{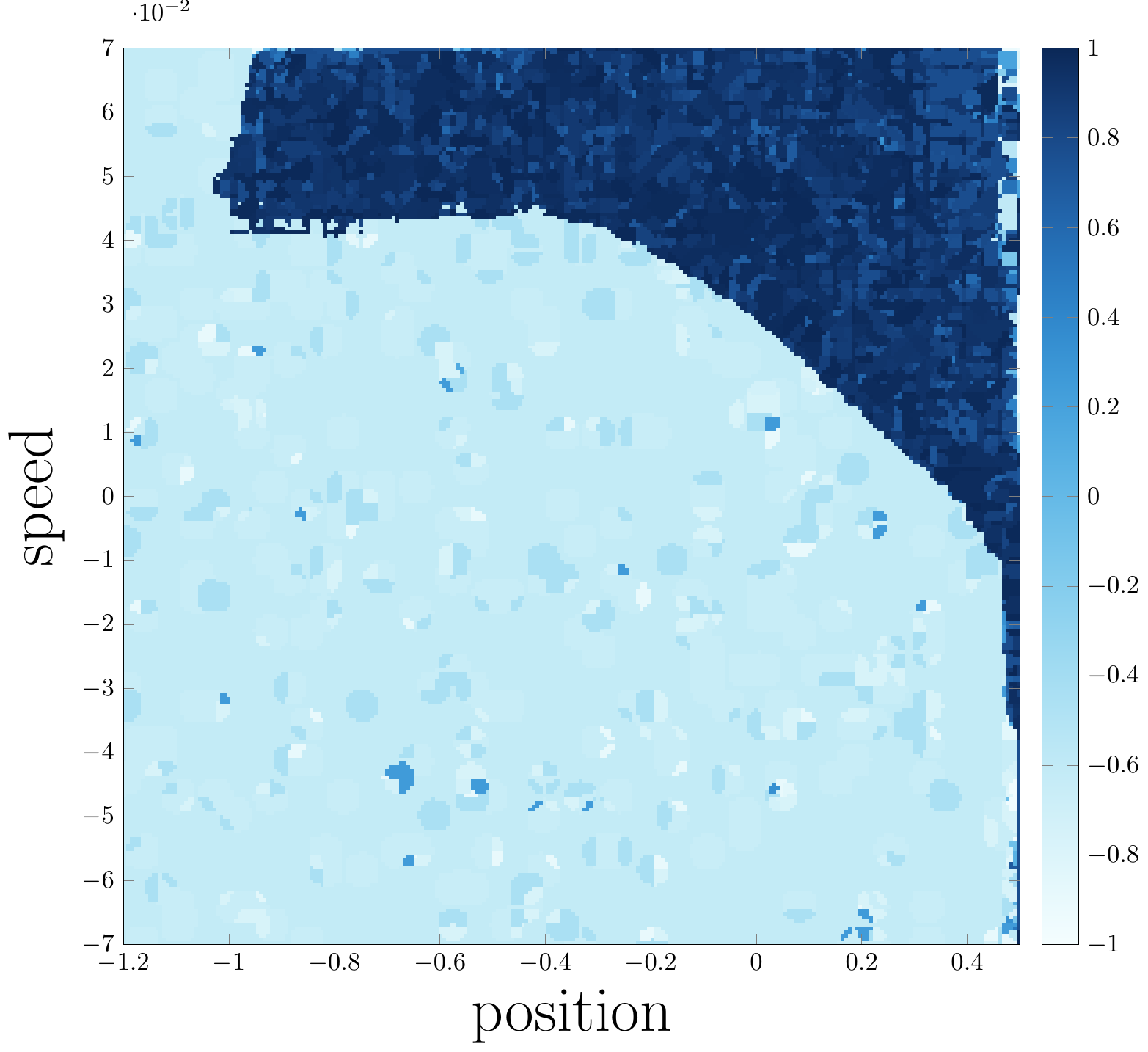}
}
\hspace{-2.48ex}
\subfigure[Nuclear Norm]{
    \label{fig:policy_mc_nucnorm}
    \includegraphics[trim={27 0 0 0},clip,height=0.22\textwidth]{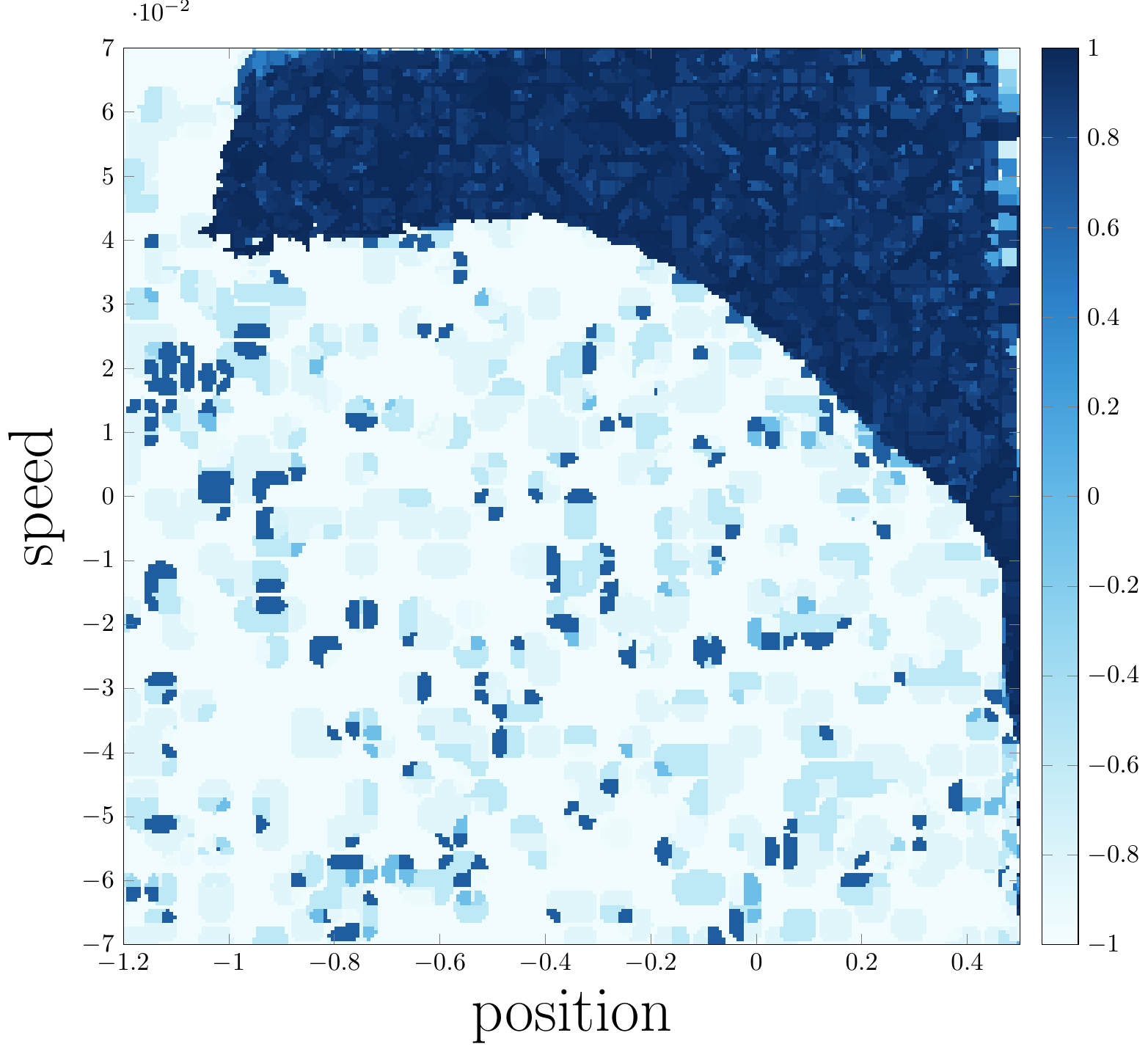}
}
\hspace{-2.48ex}
\subfigure[Ours]{
    \label{fig:policy_mc_ours}
    \includegraphics[trim={27 0 0 0},clip,height=0.22\textwidth]{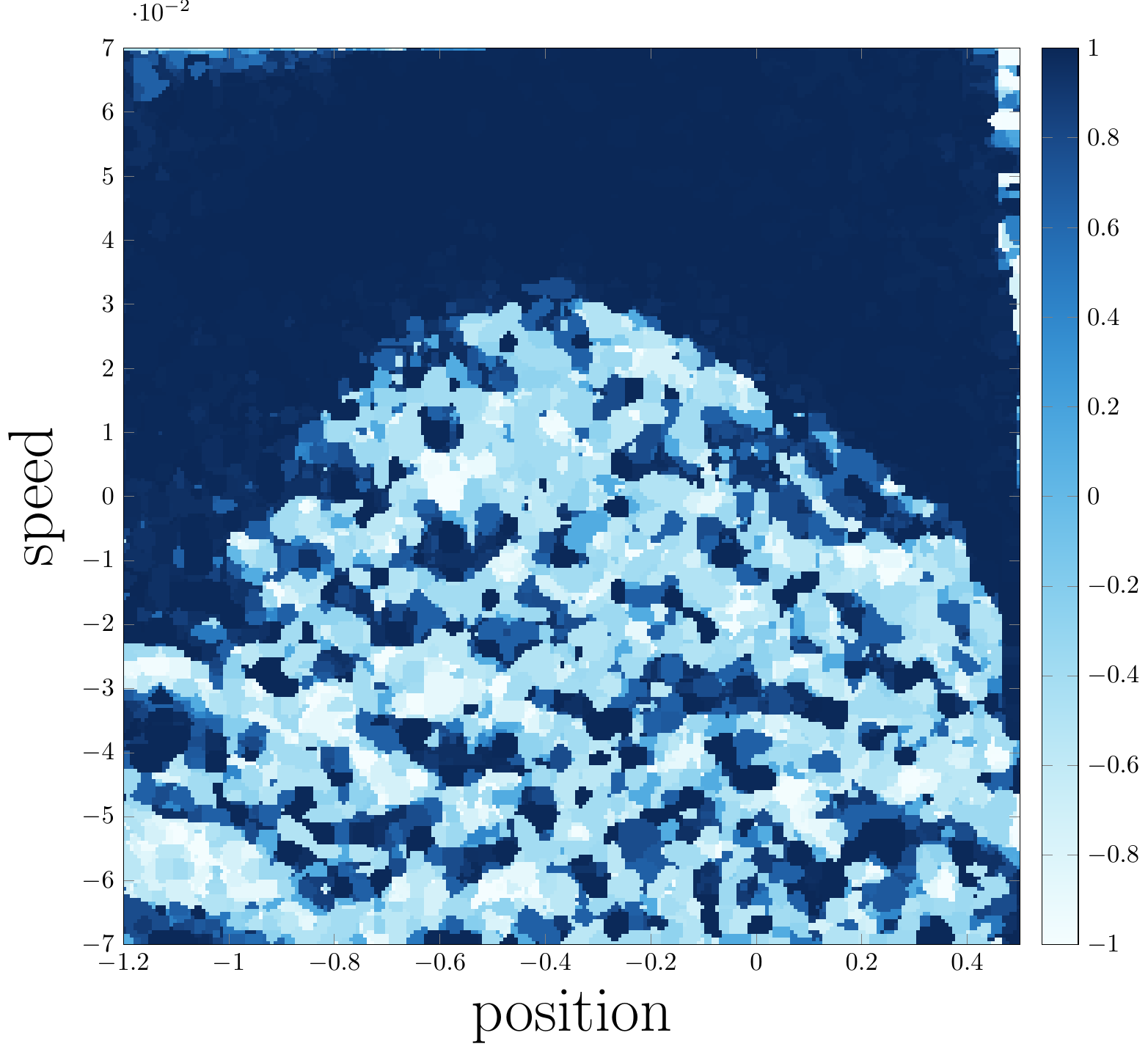}
}
%\vspace{-0.3cm}
\caption{Policy visualization of different methods on the Mountain Car control task. The policy is obtained from the output $Q^{(T)}$  by taking $\arg\max_{a\in\mA}Q^{(T)}(s,a)$ at each state $s$.}
\label{fig:policy-mc}
\vspace{-0.3cm}
\end{figure}

\subsection{Double Integrator}
{\bf Sample Complexity and Error Guarantees.}
\begin{figure}[H]
\vspace{-0.16in}
\centering
\subfigure[Sample Complexity]{
    \label{fig:di_.9_linf_complexity}
    \includegraphics[width=0.241\textwidth]{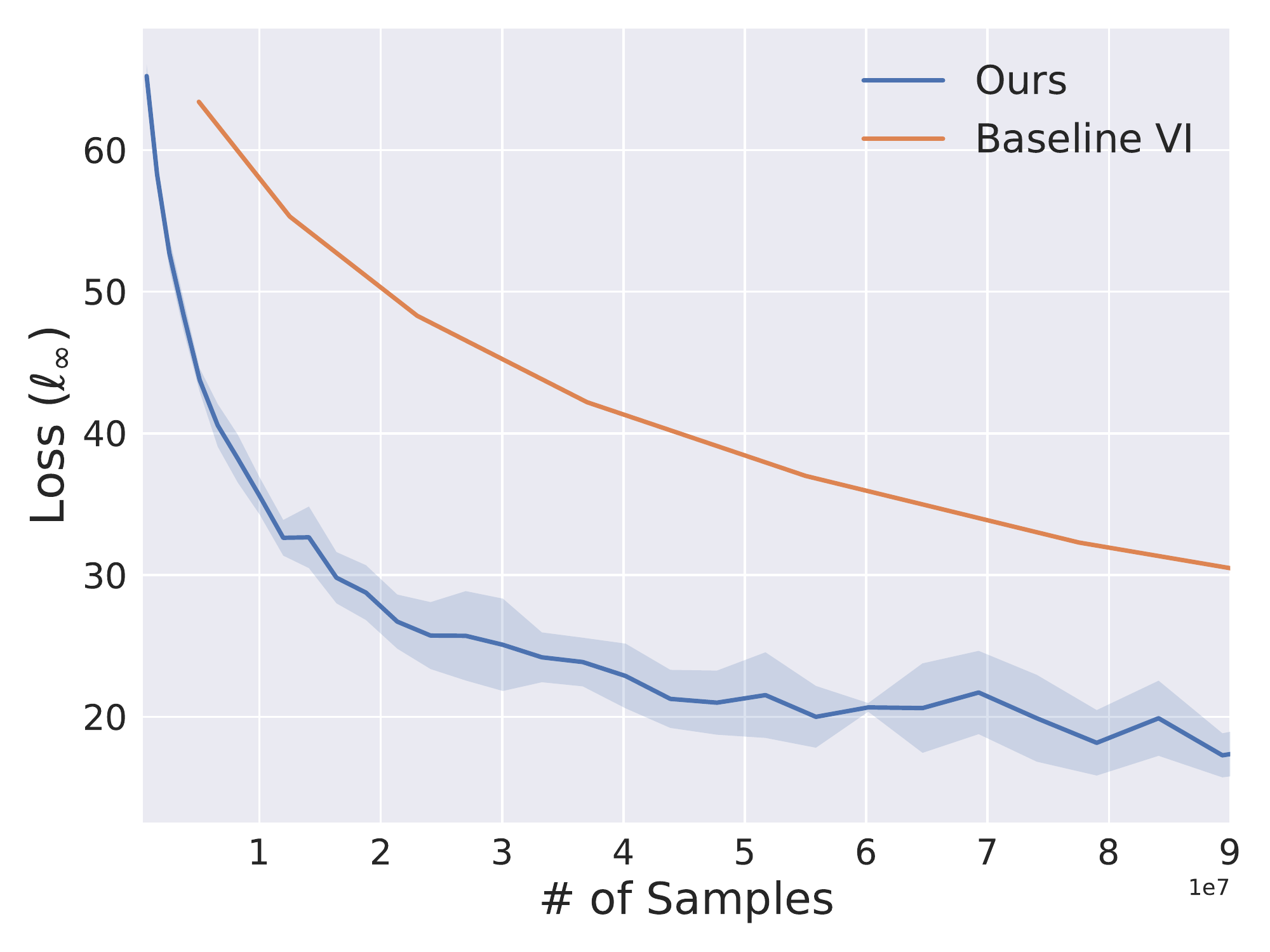}
}
\hspace{-2.3ex}
\subfigure[Sample Complexity]{
    \label{fig:di_.9_mean_complexity}
    \includegraphics[width=0.241\textwidth]{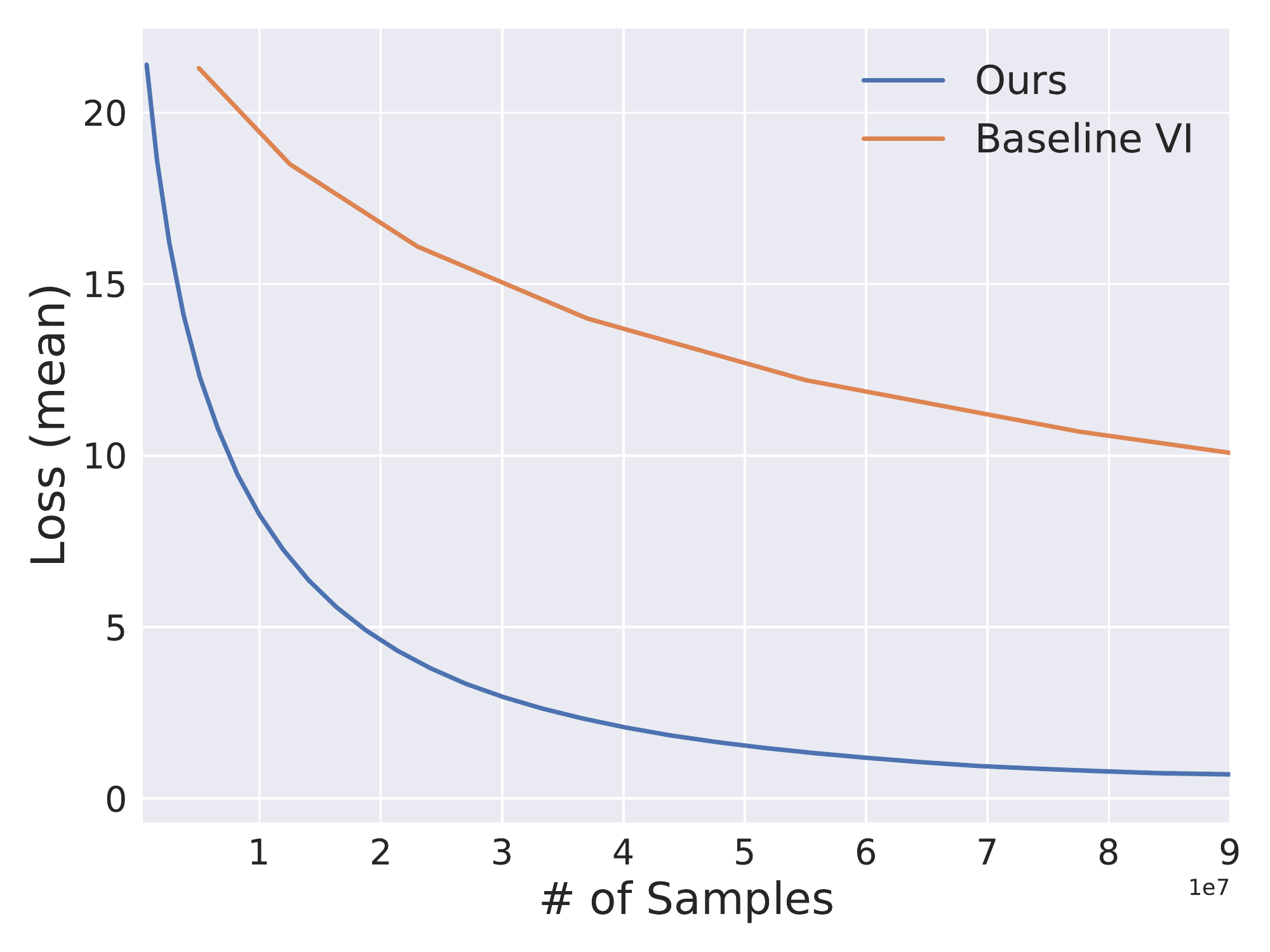}}
\hspace{-2.3ex}
% \hfill
\subfigure[$\ell_{\infty}$ Errors]{
    \label{fig:di_.9_linf_loss}
    \includegraphics[width=0.241\textwidth]{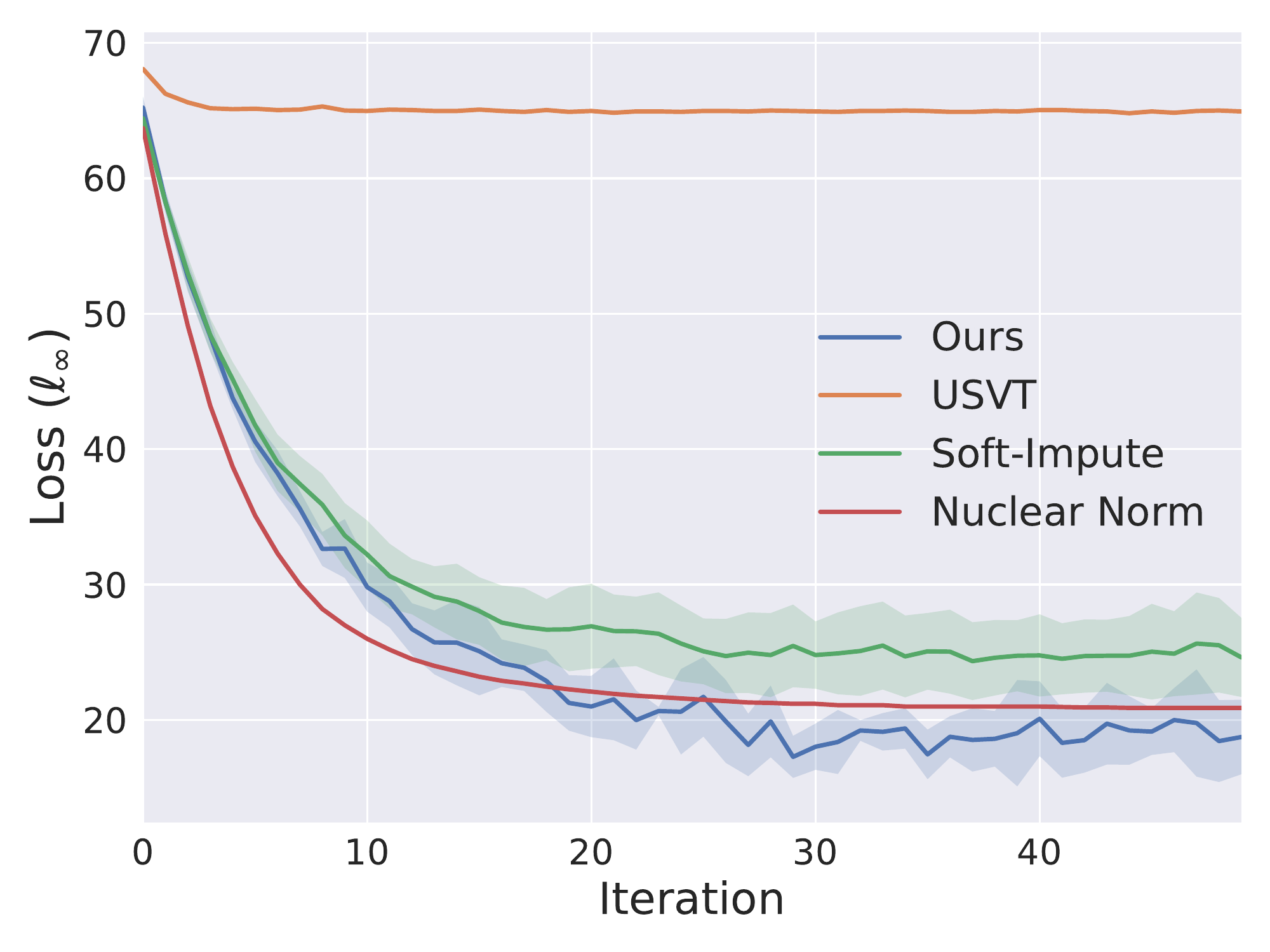}
}
\hspace{-2.3ex}
\subfigure[Mean Errors]{
    \label{fig:di_.9_mean_loss}
    \includegraphics[width=0.241\textwidth]{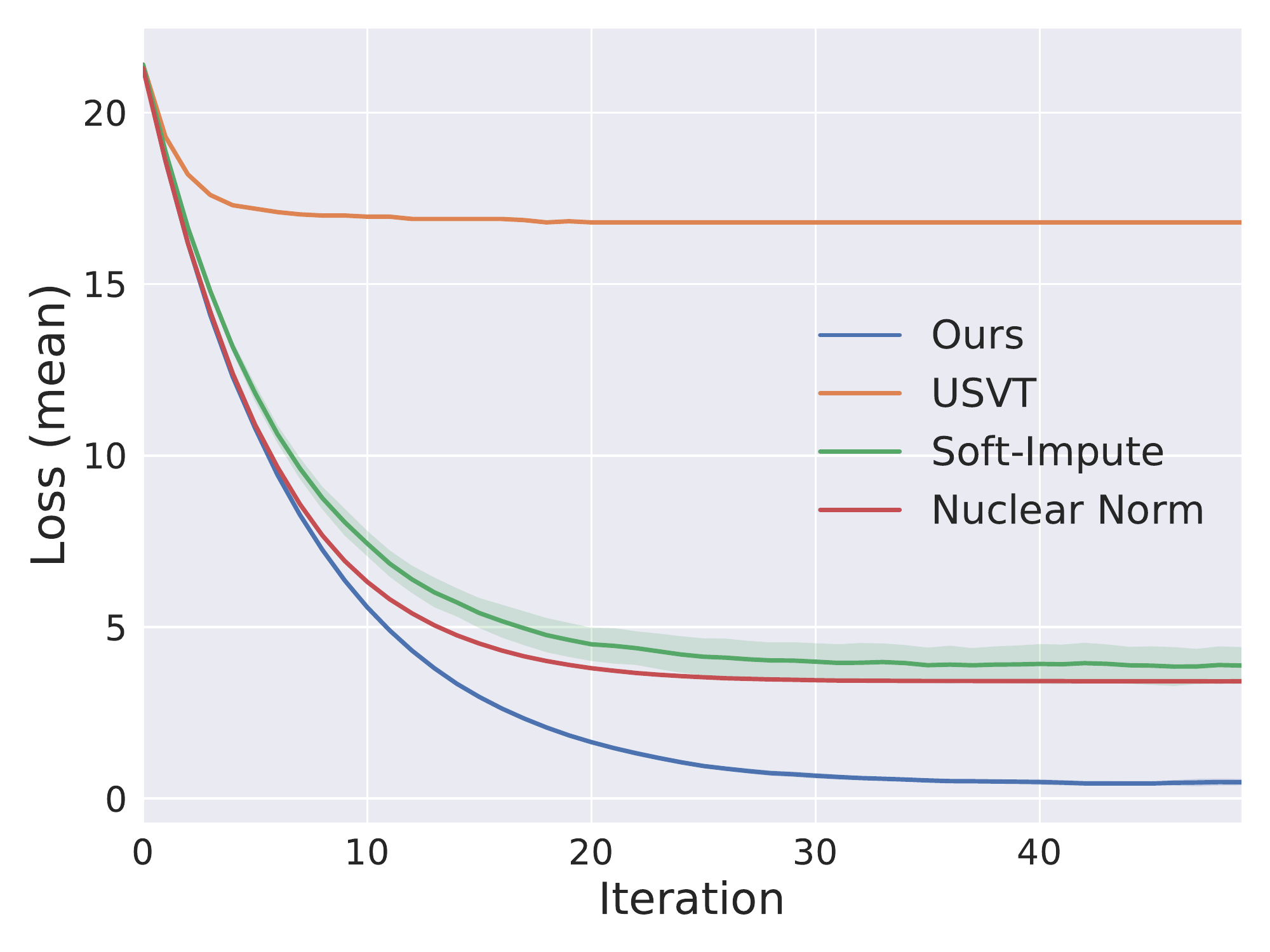}
}

%\vspace{-0.3cm}
\caption{ Empirical results on the Double Integrator control task.  In (a) and (b), we show the improved sample complexity for achieving different levels of $\ell_\infty$ error and mean error, respectively. In (c) and (d), we compare the $\ell_\infty$ error and the mean error for various ME methods. Results are averaged across 5 runs for each method.}
\label{fig:di_.9_main}
\vspace{-0.3cm}
\end{figure}

\medskip\noindent
{\bf Policy Visualization.}
\begin{figure}[H]
\vspace{-0.16in}
\centering
\subfigure[Optimal Policy]{
    \label{fig:policy_di_optimal}
    \includegraphics[height=0.218\textwidth]{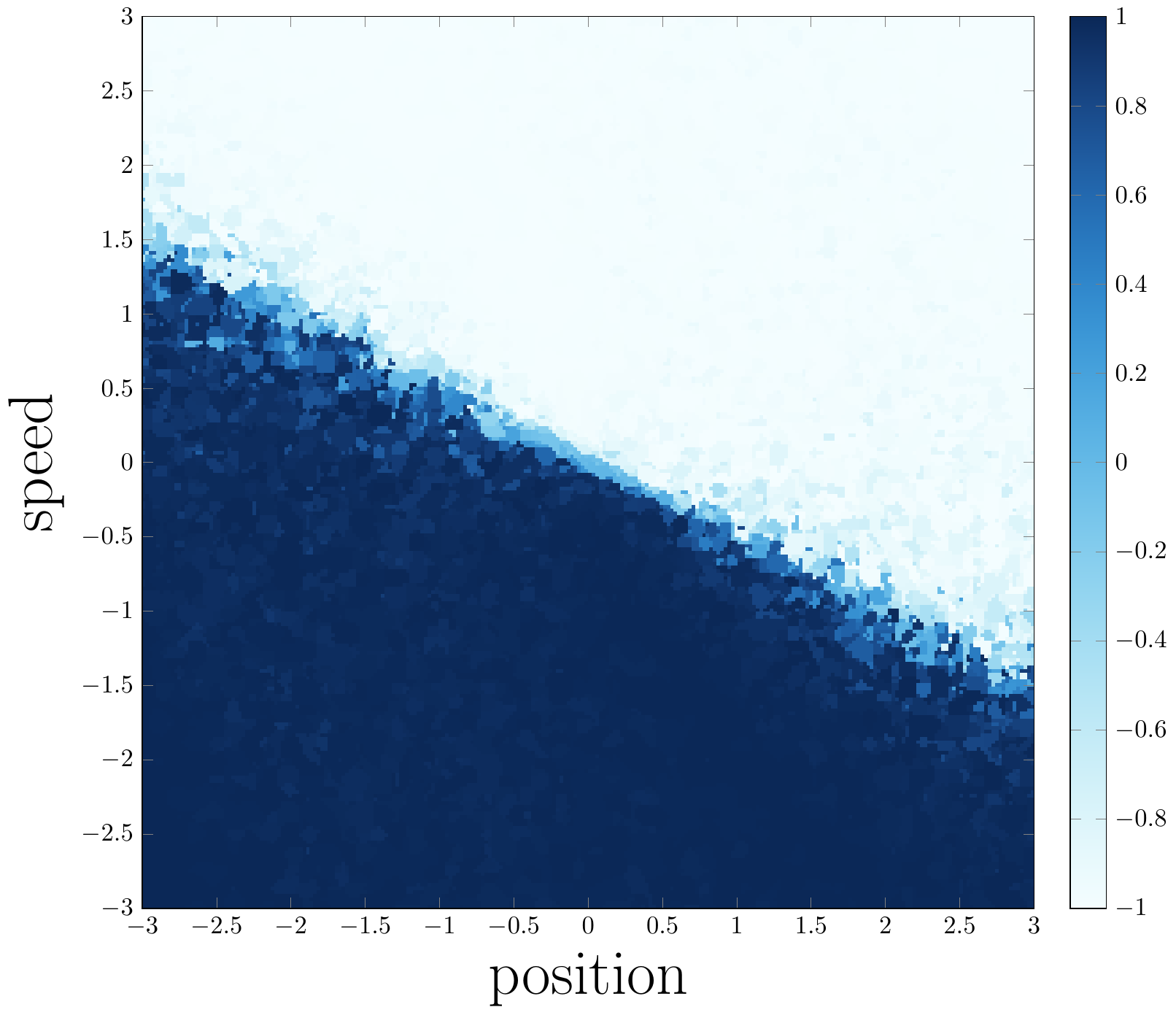}
}
\hspace{-2.48ex}
\subfigure[Soft-Impute]{
    \label{fig:policy_di_softimp}
    \includegraphics[trim={27 0 0 0},clip,height=0.218\textwidth]{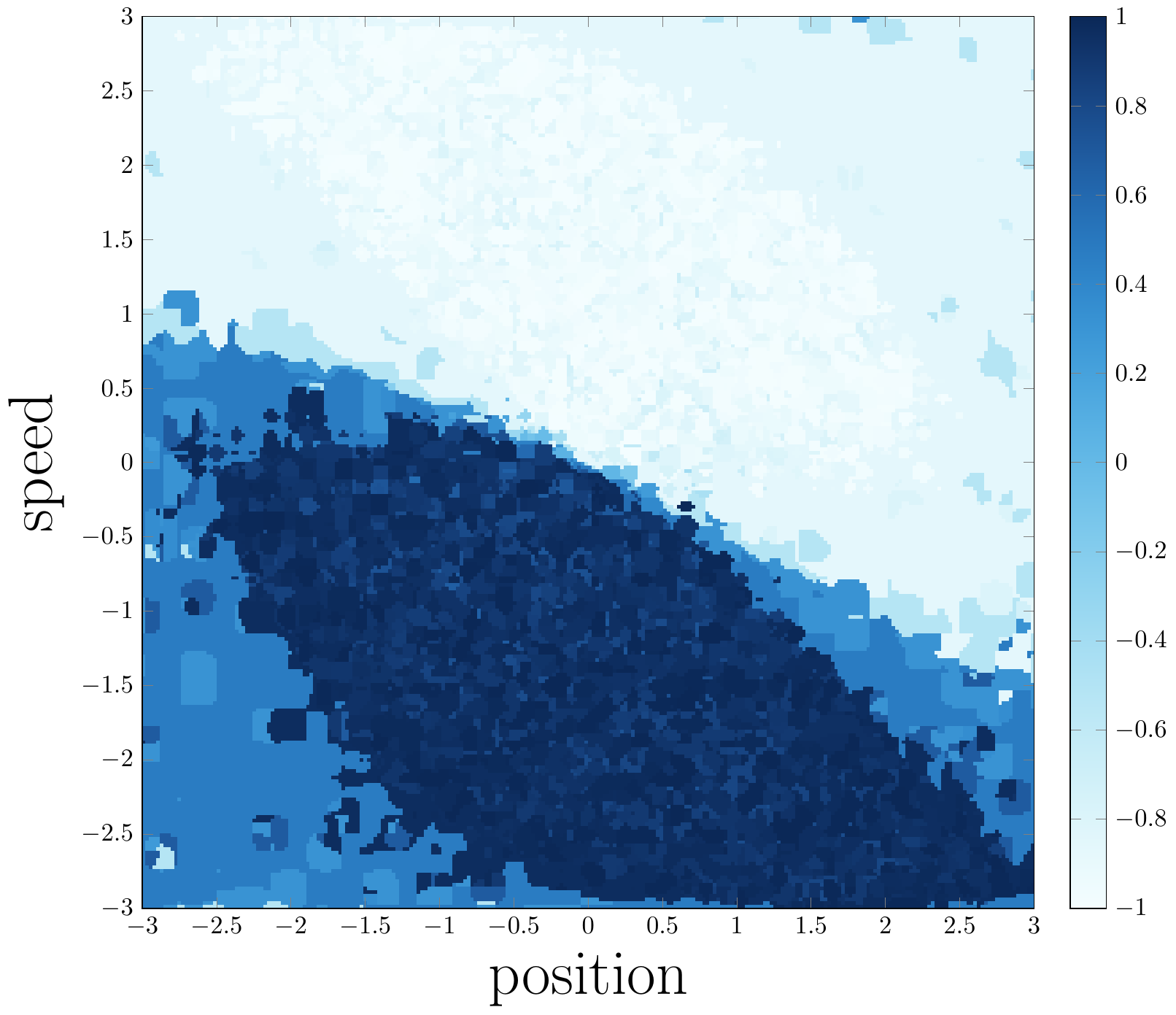}
}
\hspace{-2.48ex}
\subfigure[Nuclear Norm]{
    \label{fig:policy_di_nucnorm}
    \includegraphics[trim={27 0 0 0},clip,height=0.218\textwidth]{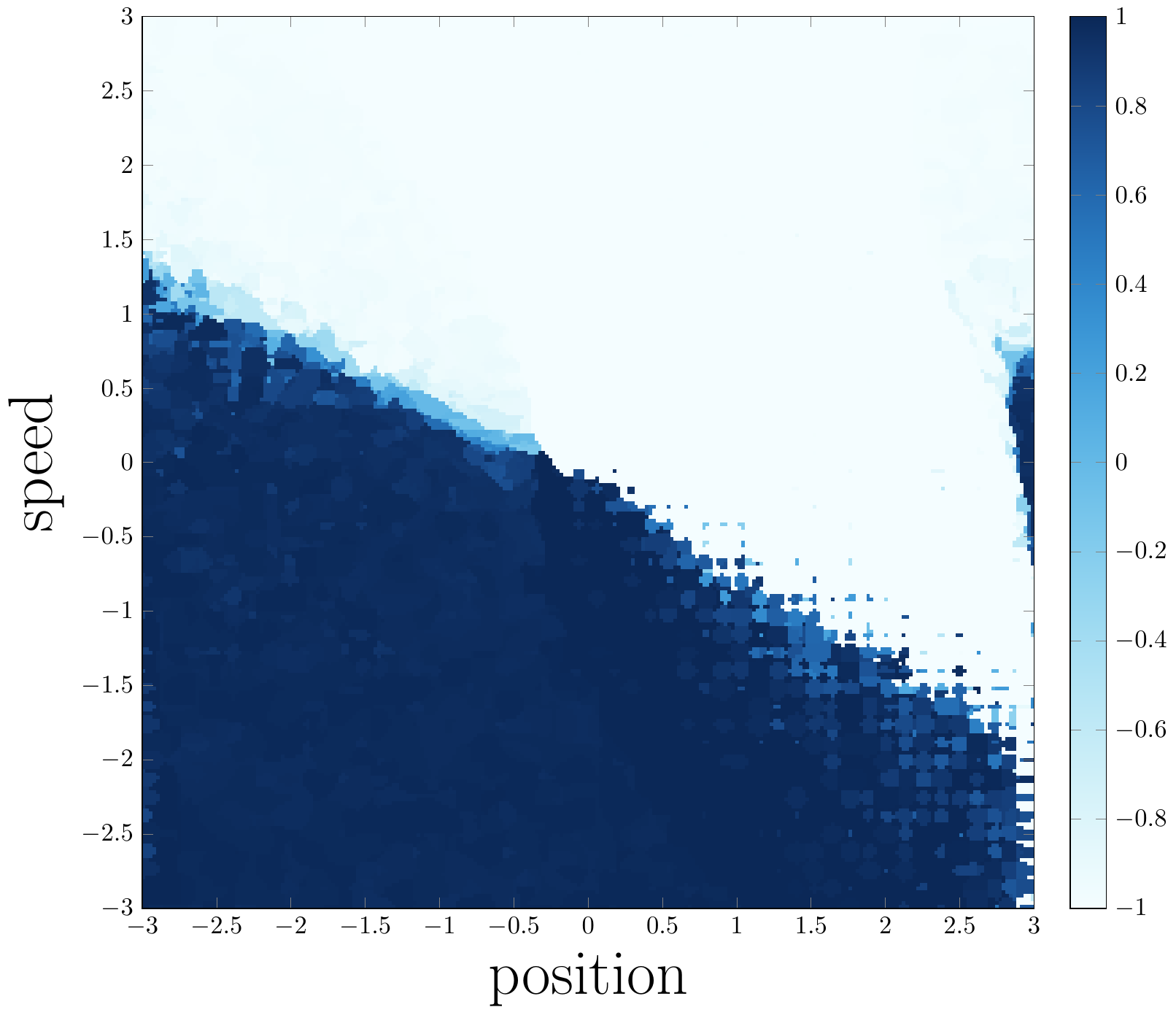}
}
\hspace{-2.48ex}
\subfigure[Ours]{
    \label{fig:policy_di_ours}
    \includegraphics[trim={27 0 0 0},clip,height=0.218\textwidth]{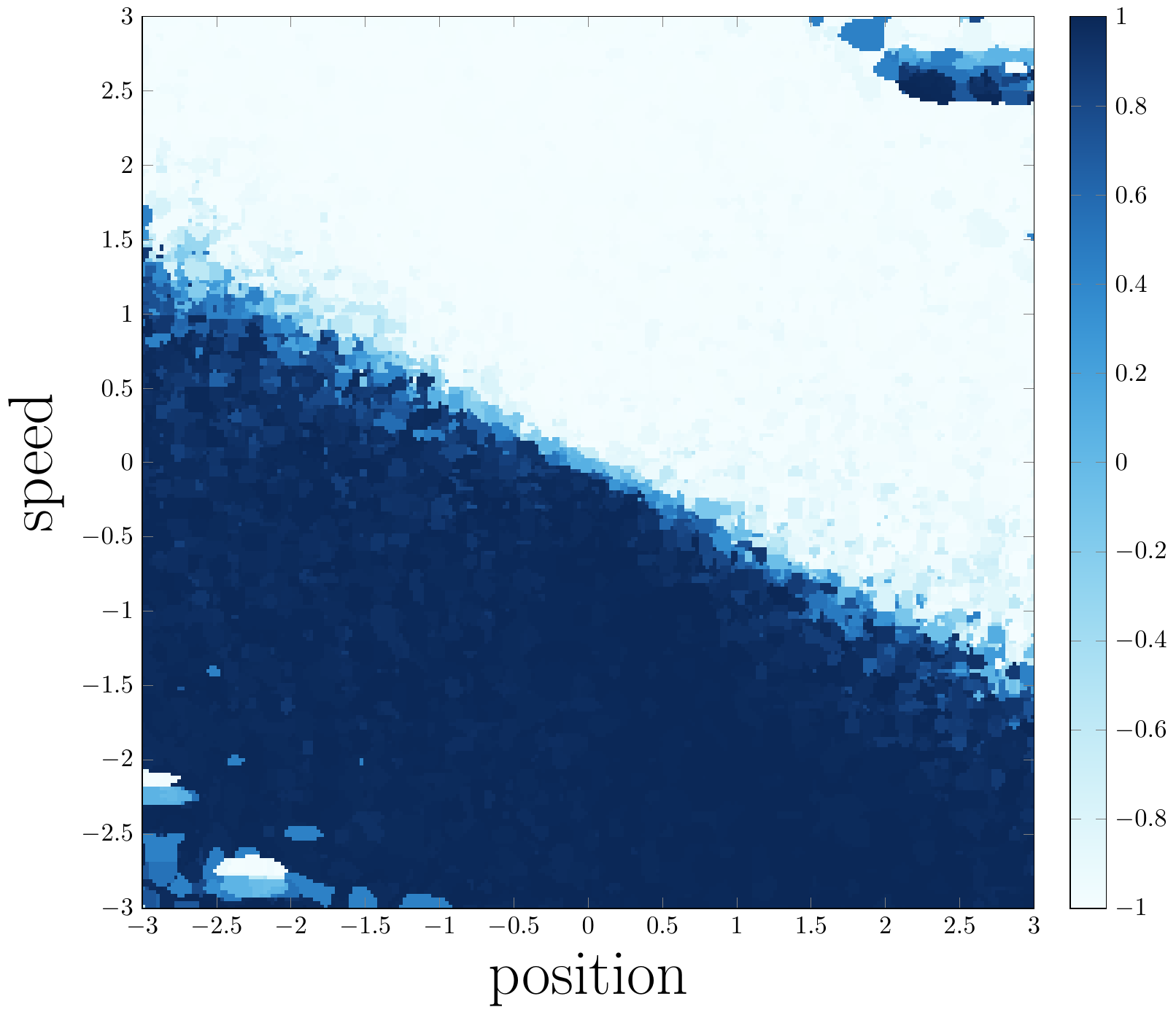}
}
%\vspace{-0.3cm}
\caption{ Policy visualization of different methods on the Double Integrator control task. The policy is obtained from the output $Q^{(T)}$  by taking $\arg\max_{a\in\mA}Q^{(T)}(s,a)$ at each state $s$.}
\label{fig:policy-di}
\vspace{-0.3cm}
\end{figure}

\subsection{Cart-Pole}
{\bf Sample Complexity and Error Guarantees.}
\begin{figure}[H]
\vspace{-0.16in}
\centering
\subfigure[Sample Complexity]{
    \label{fig:cp_.9_linf_complexity}
    \includegraphics[width=0.241\textwidth]{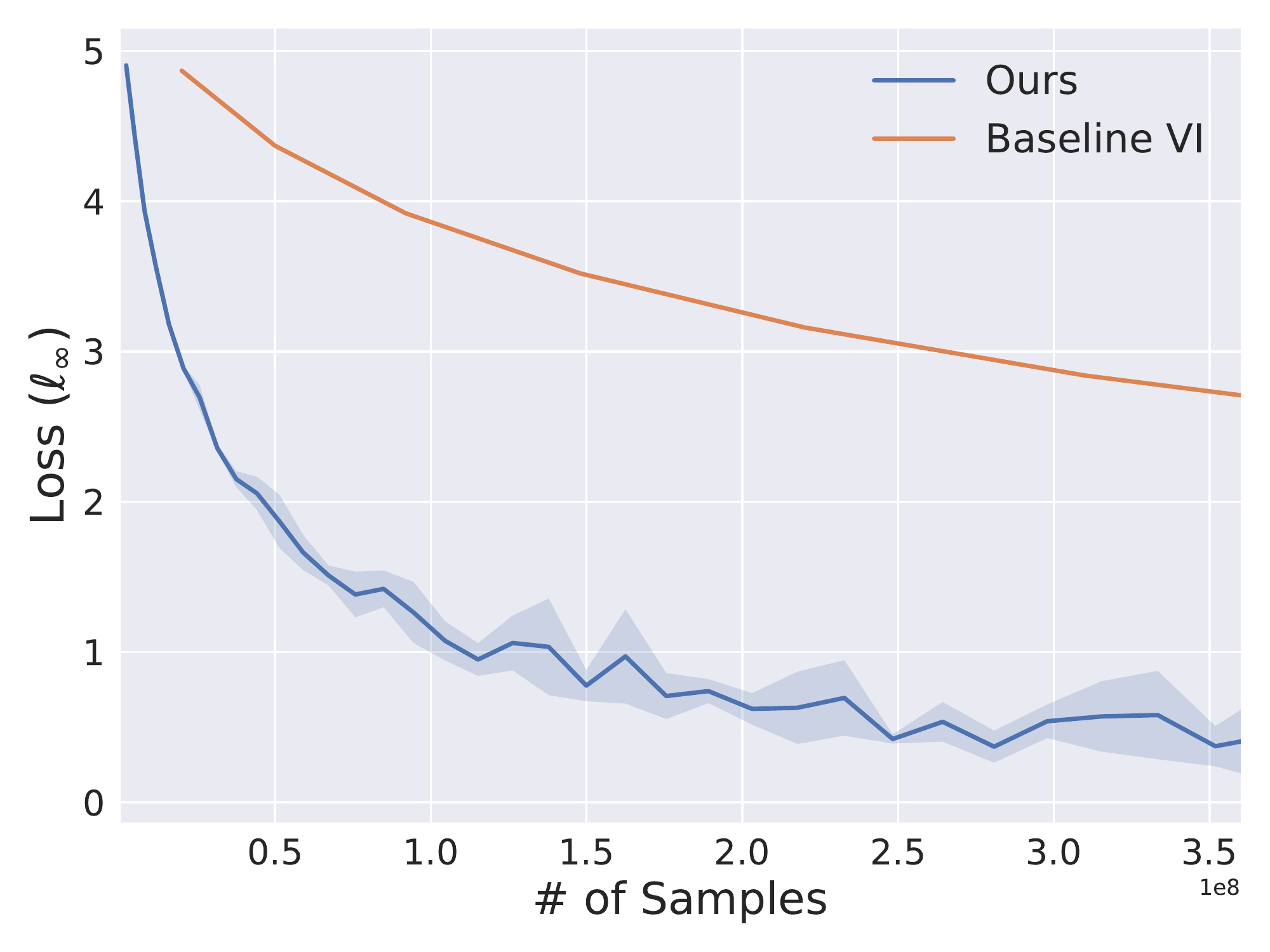}
}
\hspace{-2.3ex}
\subfigure[Sample Complexity]{
    \label{fig:cp_.9_mean_complexity}
    \includegraphics[width=0.241\textwidth]{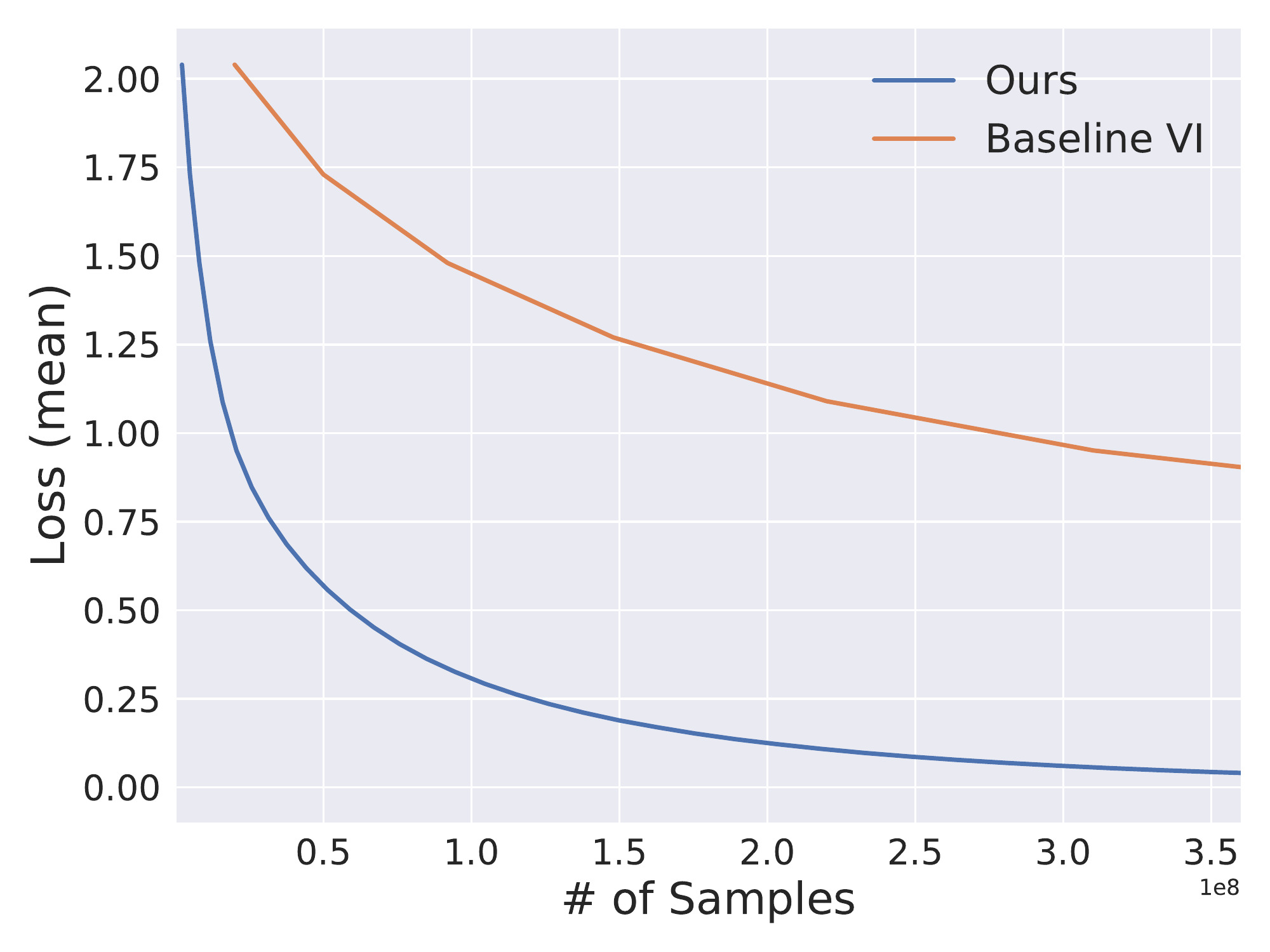}
}
\hspace{-2.3ex}
\subfigure[$\ell_{\infty}$ Errors]{
    \label{fig:cp_.9_linf_loss}
    \includegraphics[width=0.241\textwidth]{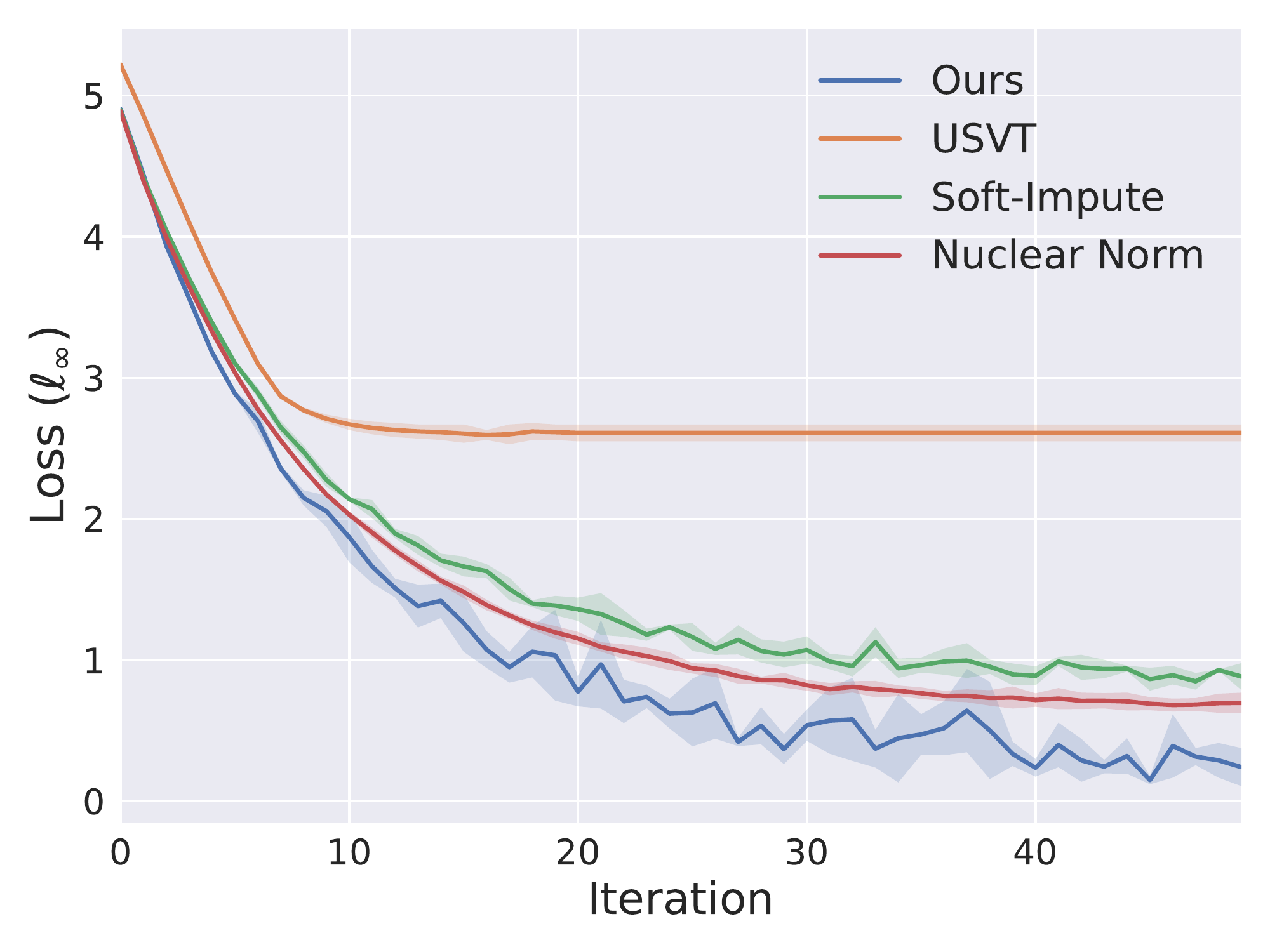}
}
\hspace{-2.3ex}
% \hfill
\subfigure[Mean Errors]{
    \label{fig:cp_.9_mean_loss}
    \includegraphics[width=0.241\textwidth]{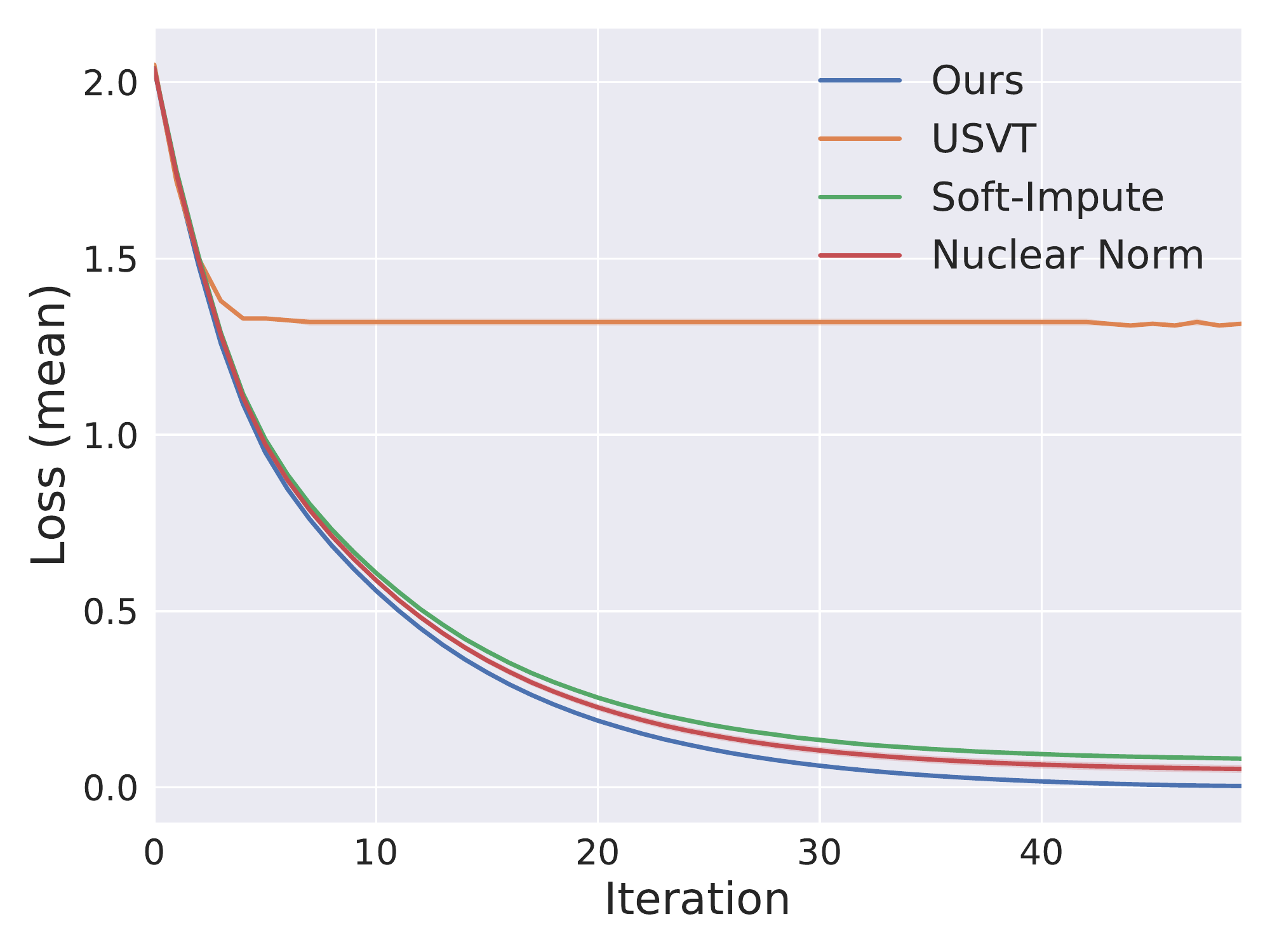}
}

%\vspace{-0.3cm}
\caption{Empirical results on the Cart-Pole control task.  In (a) and (b), we show the improved sample complexity for achieving different levels of $\ell_\infty$ error and mean error, respectively. In (c) and (d), we compare the $\ell_\infty$ error and the mean error for various ME methods. Results are averaged across 5 runs for each method.}
\label{fig:cp_.9_main}
\vspace{-0.3cm}
\end{figure}

\medskip\noindent
{\bf Policy Visualization.}
\begin{figure}[H]
\vspace{-0.16in}
\centering
\subfigure[Optimal Policy]{
    \label{fig:policy_cp_optimal}
    \includegraphics[height=0.22\textwidth]{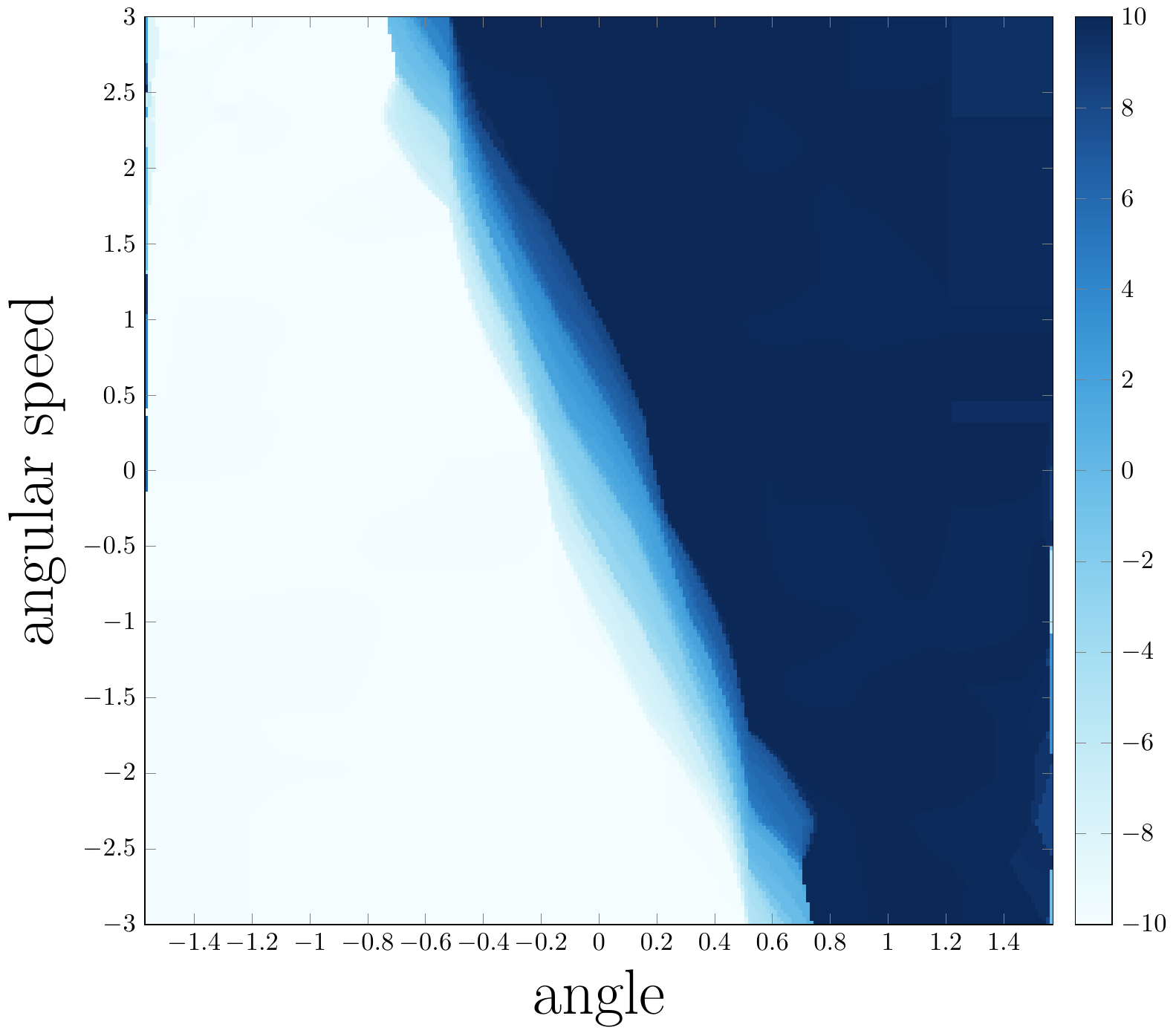}
}
\hspace{-2.48ex}
\subfigure[Soft-Impute]{
    \label{fig:policy_cp_softimp}
    \includegraphics[trim={27 0 0 0},clip,height=0.22\textwidth]{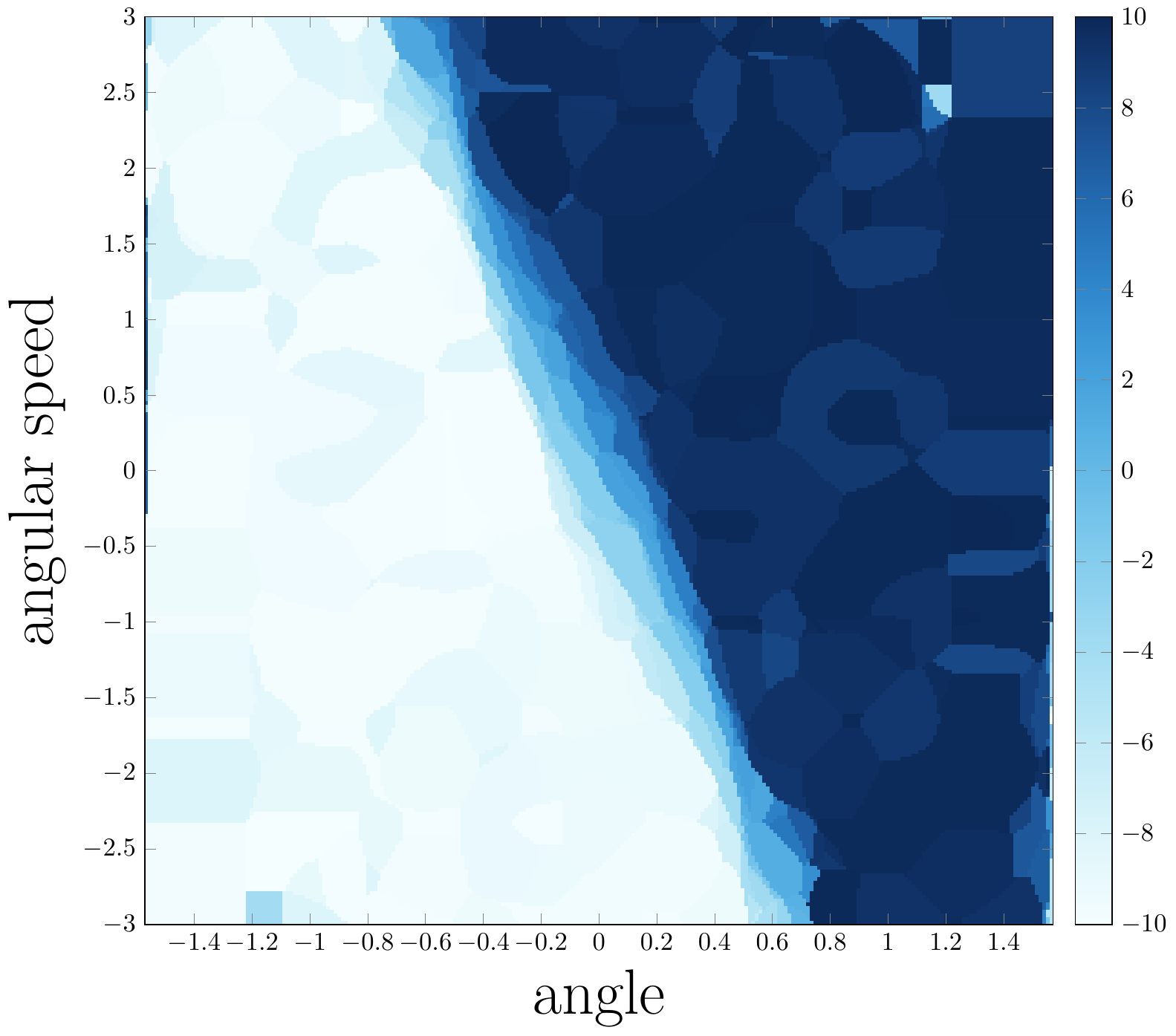}
}
\hspace{-2.48ex}
\subfigure[Nuclear Norm]{
    \label{fig:policy_cp_nucnorm}
    \includegraphics[trim={27 0 0 0},clip,height=0.22\textwidth]{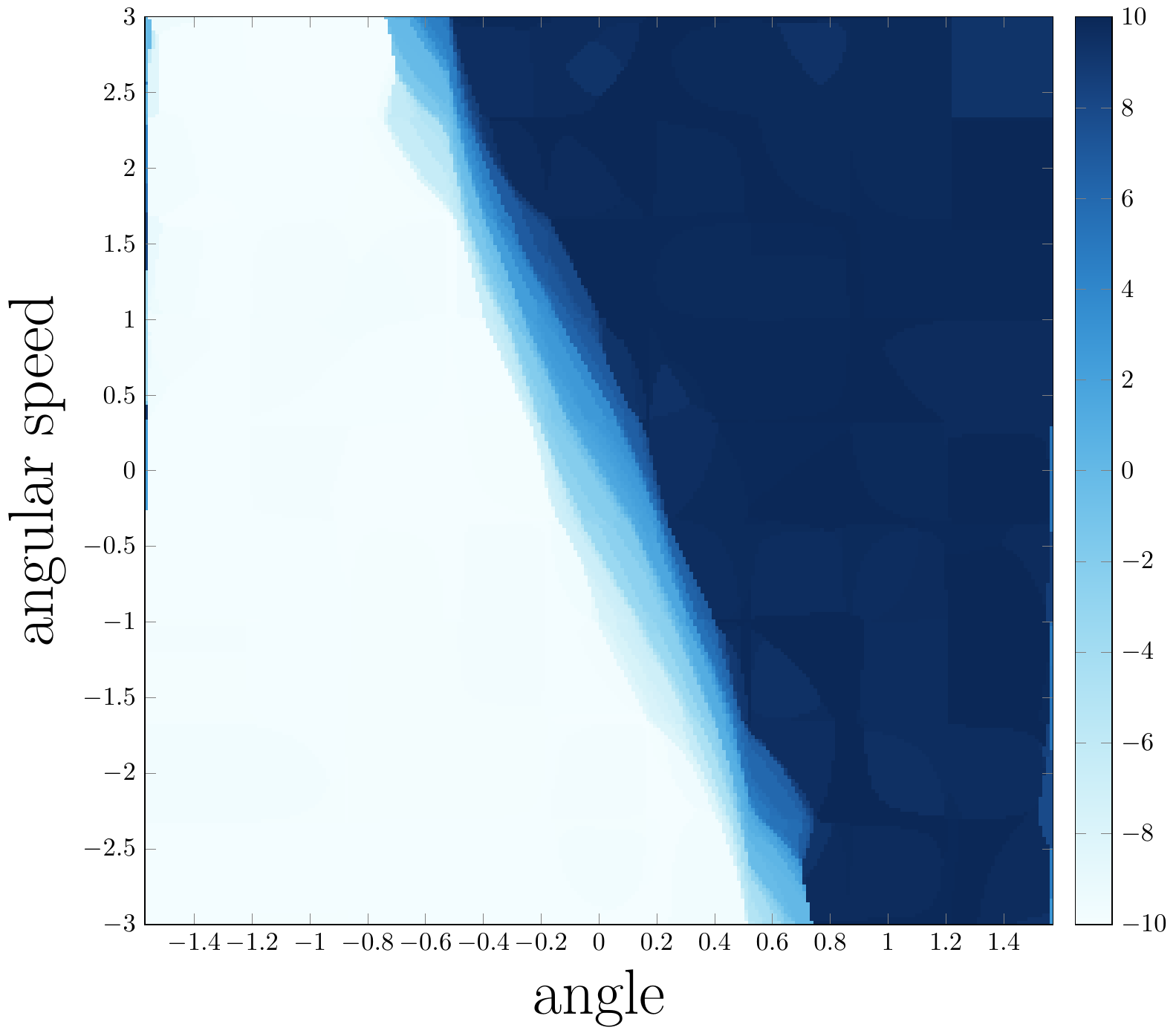}
}
\hspace{-2.48ex}
\subfigure[Ours]{
    \label{fig:policy_cp_ours}
    \includegraphics[trim={27 0 0 0},clip,height=0.22\textwidth]{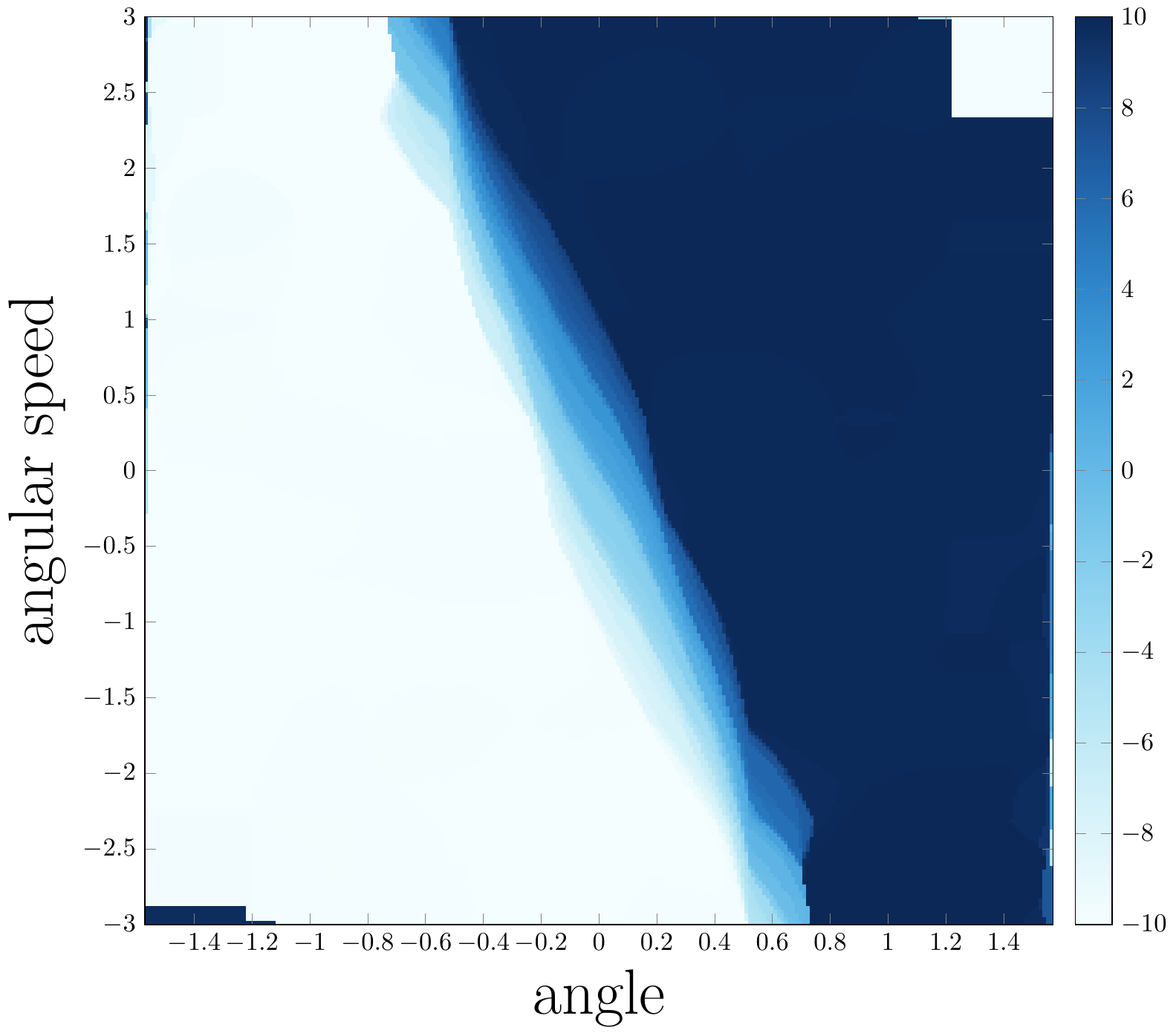}
}
%\vspace{-0.3cm}
\caption{ Policy visualization of different methods on the Cart-Pole control task. The policy is obtained from the output $Q^{(T)}$  by taking $\arg\max_{a\in\mA}Q^{(T)}(s,a)$ at each state $s$. Recall that the state space is 4-dimensional. We hence visualize a 2-dimensional slice in the figure.}
\label{fig:policy-cp}
\vspace{-0.3cm}
\end{figure}

\subsection{Acrobot}
{\bf Sample Complexity and Error Guarantees.}
\begin{figure}[H]
\vspace{-0.16in}
\centering

\subfigure[Sample Complexity]{
    \label{fig:ac_.9_linf_complexity}
    \includegraphics[width=0.241\textwidth]{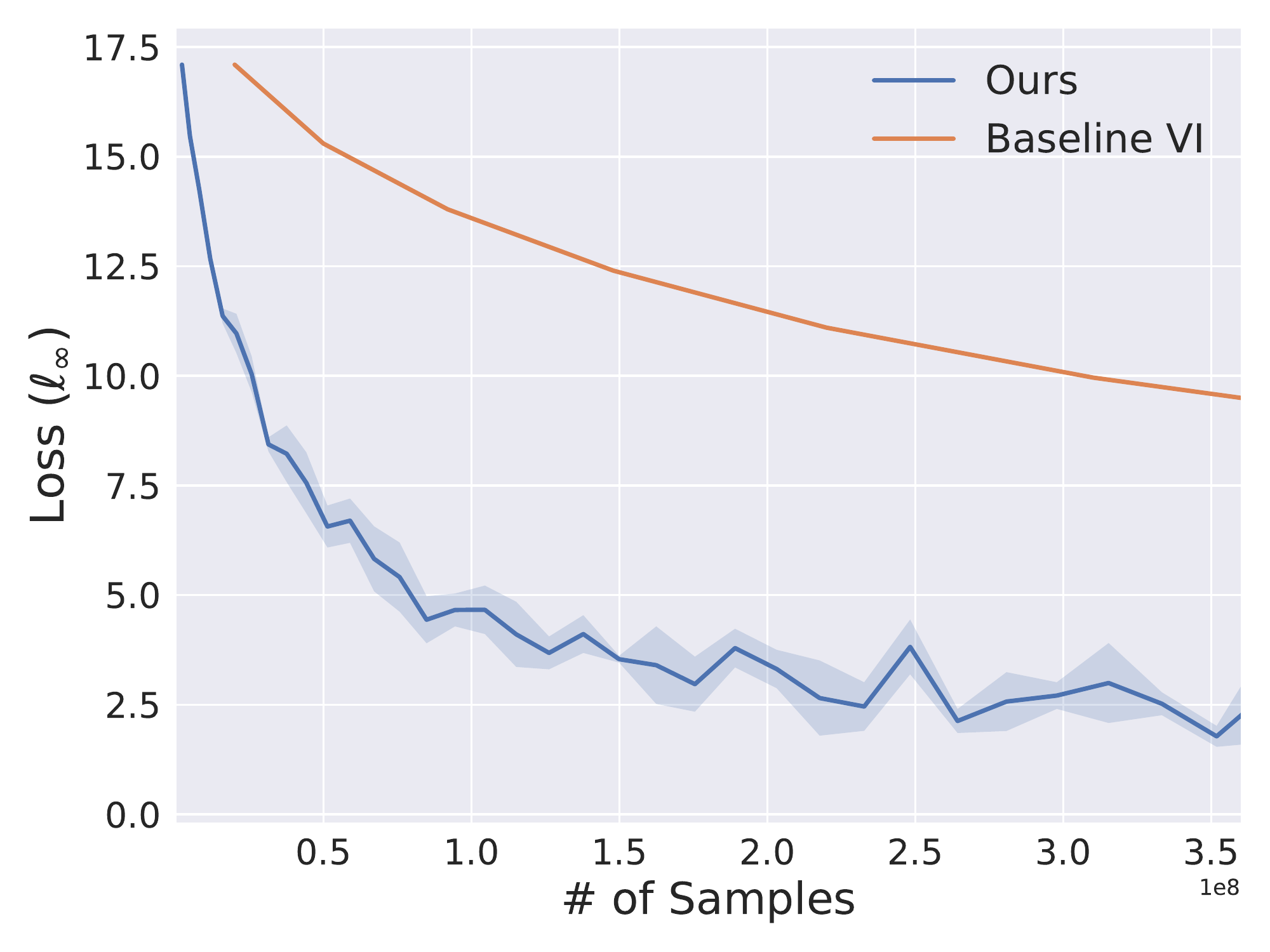}
}
\hspace{-2.3ex}
\subfigure[Sample Complexity]{
    \label{fig:ac_.9_mean_complexity}
    \includegraphics[width=0.241\textwidth]{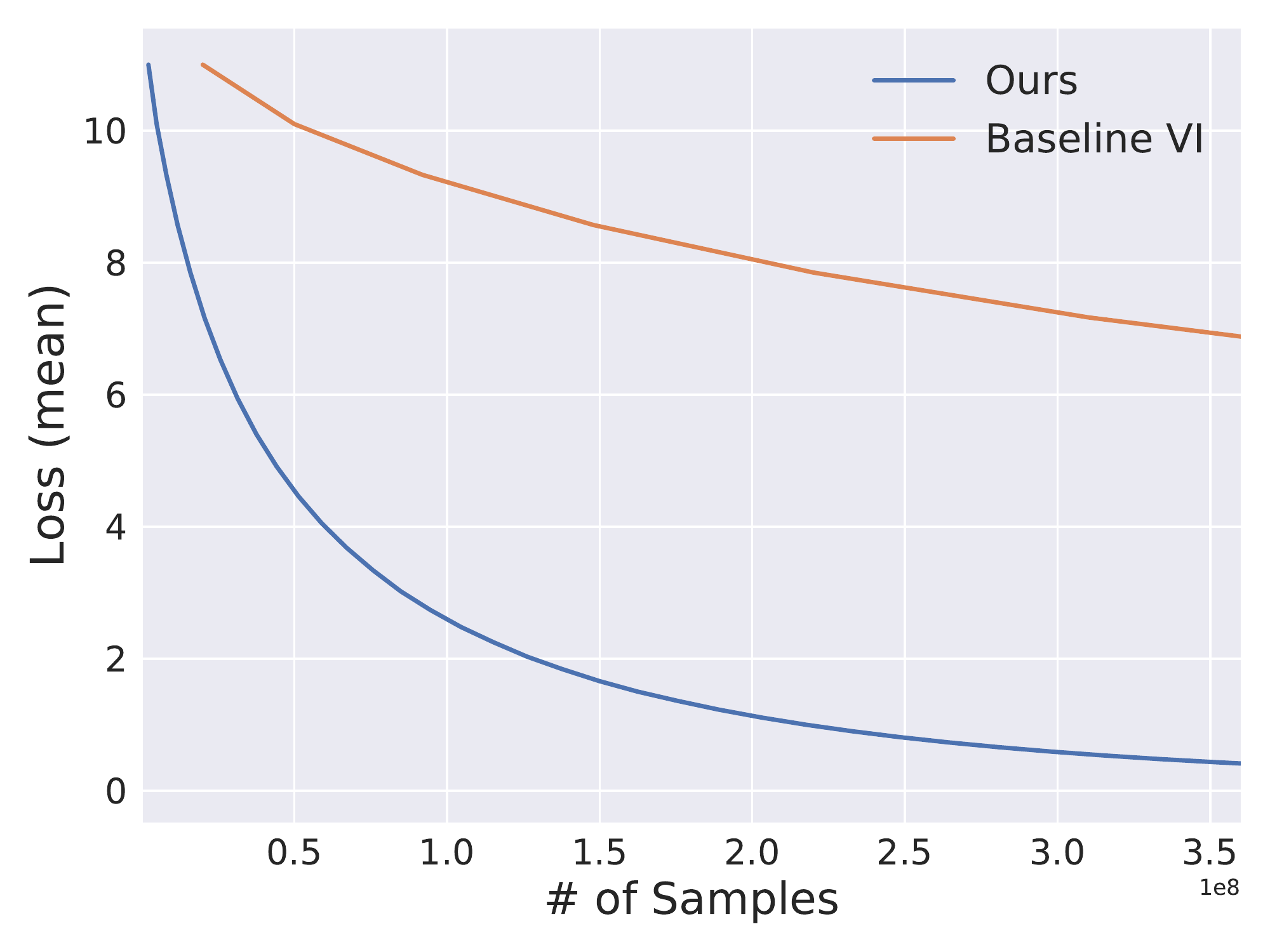}
}
\hspace{-2.3ex}
\subfigure[$\ell_{\infty}$ Errors]{
    \label{fig:ac_.9_linf_loss}
    \includegraphics[width=0.241\textwidth]{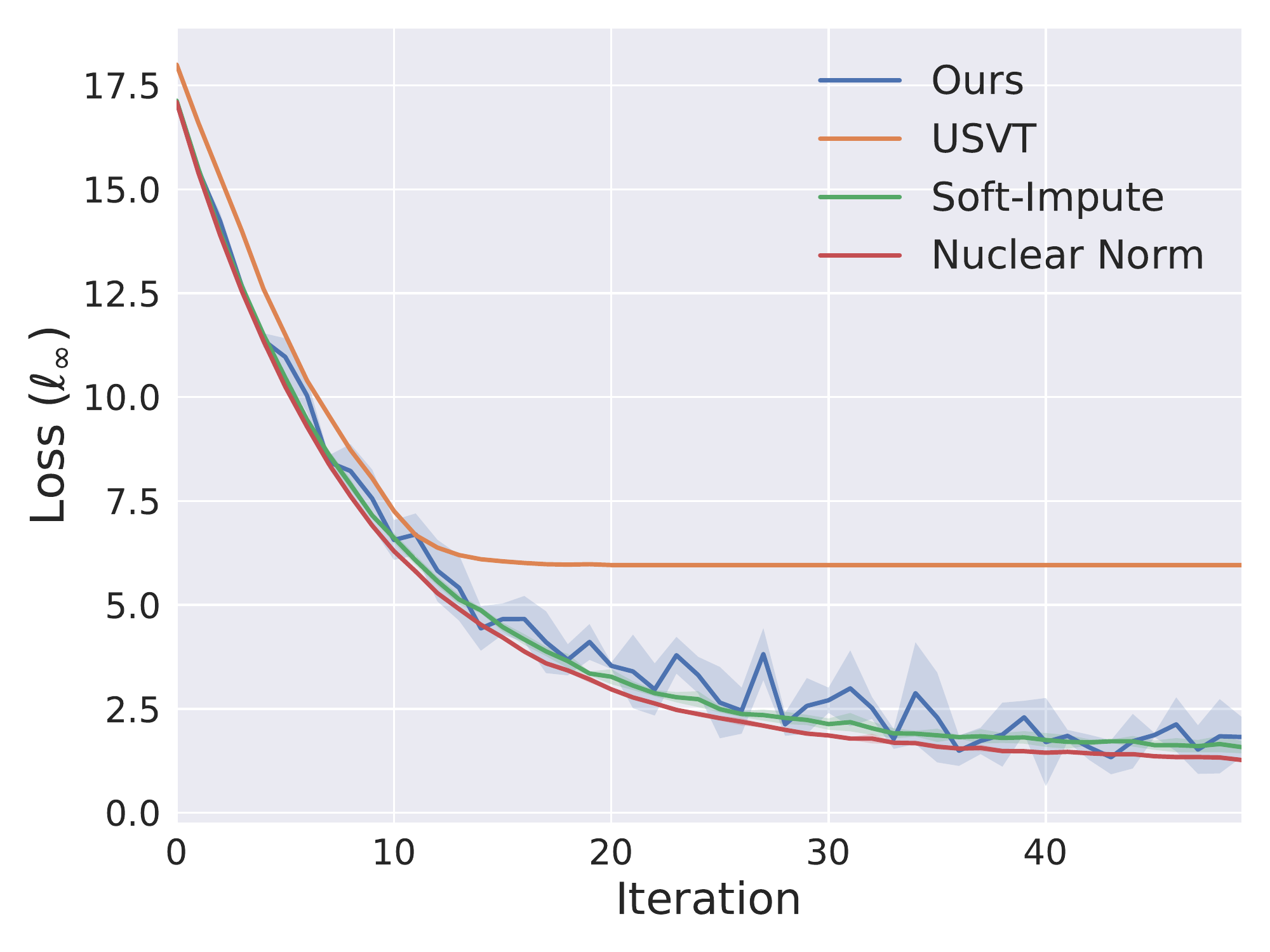}
}
\hspace{-2.3ex}
% \hfill
\subfigure[Mean Errors]{
    \label{fig:ac_.9_mean_loss}
    \includegraphics[width=0.241\textwidth]{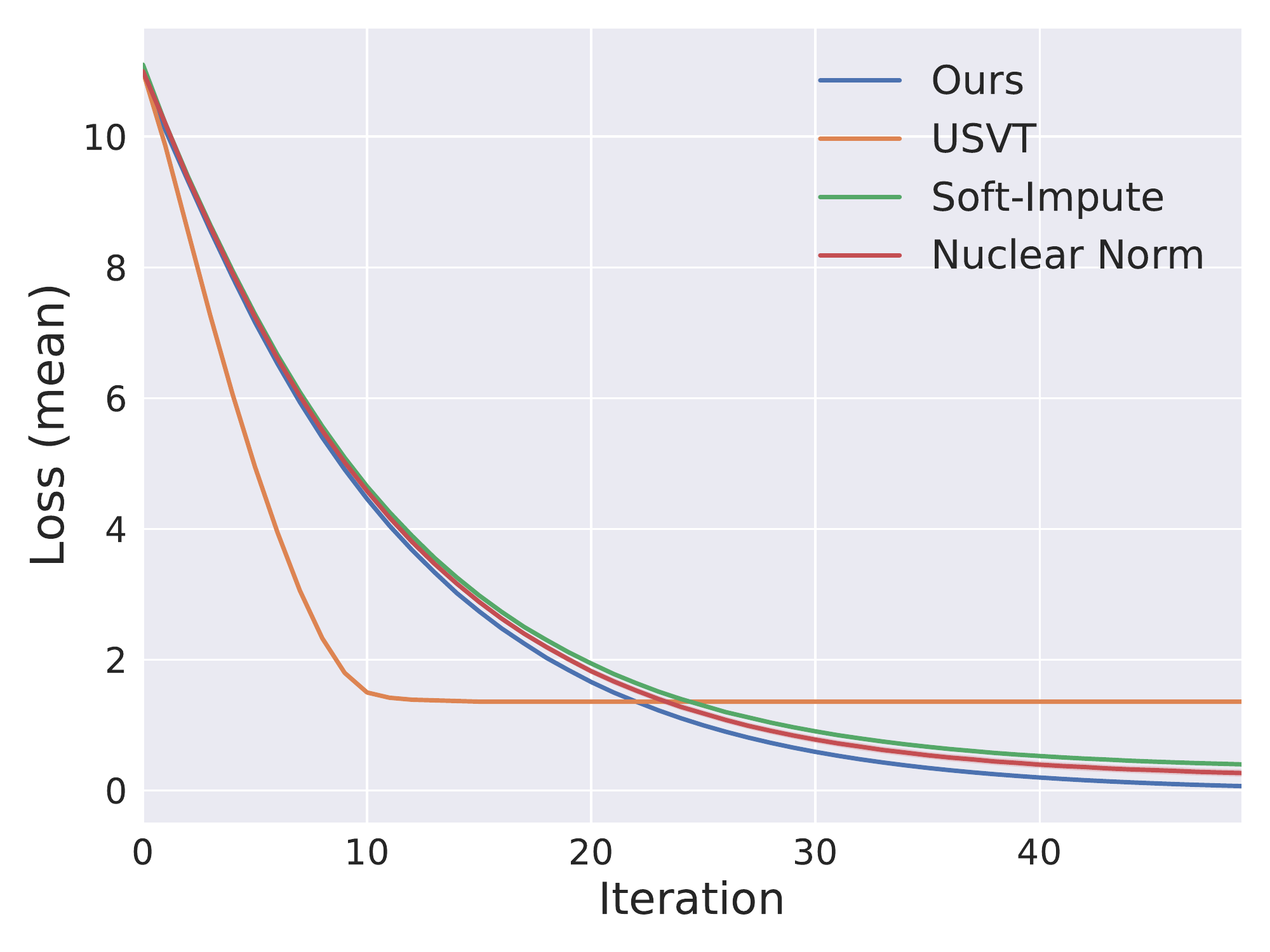}
}

%\vspace{-0.3cm}
\caption{ Empirical results on the Acrobot control task.  In (a) and (b), we show the improved sample complexity for achieving different levels of $\ell_\infty$ error and mean error, respectively. In (c) and (d), we compare the $\ell_\infty$ error and the mean error for various ME methods. Results are averaged across 5 runs for each method.}
\label{fig:ac_.9_main}
\vspace{-0.3cm}
\end{figure}

\medskip\noindent
{\bf Policy Visualization.}
\begin{figure}[H]
\vspace{-0.16in}
\centering
\subfigure[Optimal Policy]{
    \label{fig:policy_ac_optimal}
    \includegraphics[height=0.22\textwidth]{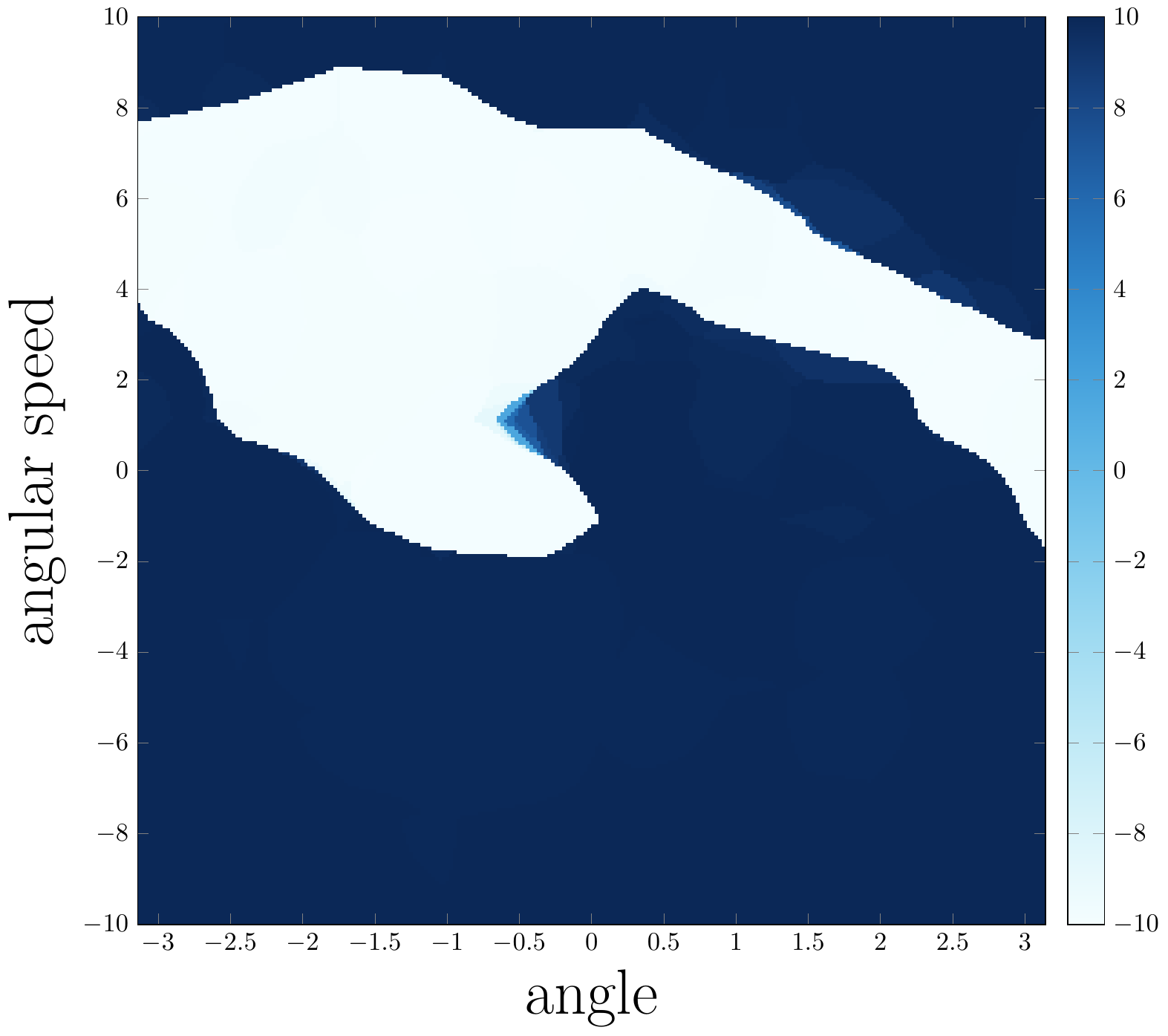}
}
\hspace{-2.3ex}
\subfigure[Soft-Impute]{
    \label{fig:policy_ac_softimp}
    \includegraphics[trim={27 0 0 0},clip,height=0.22\textwidth]{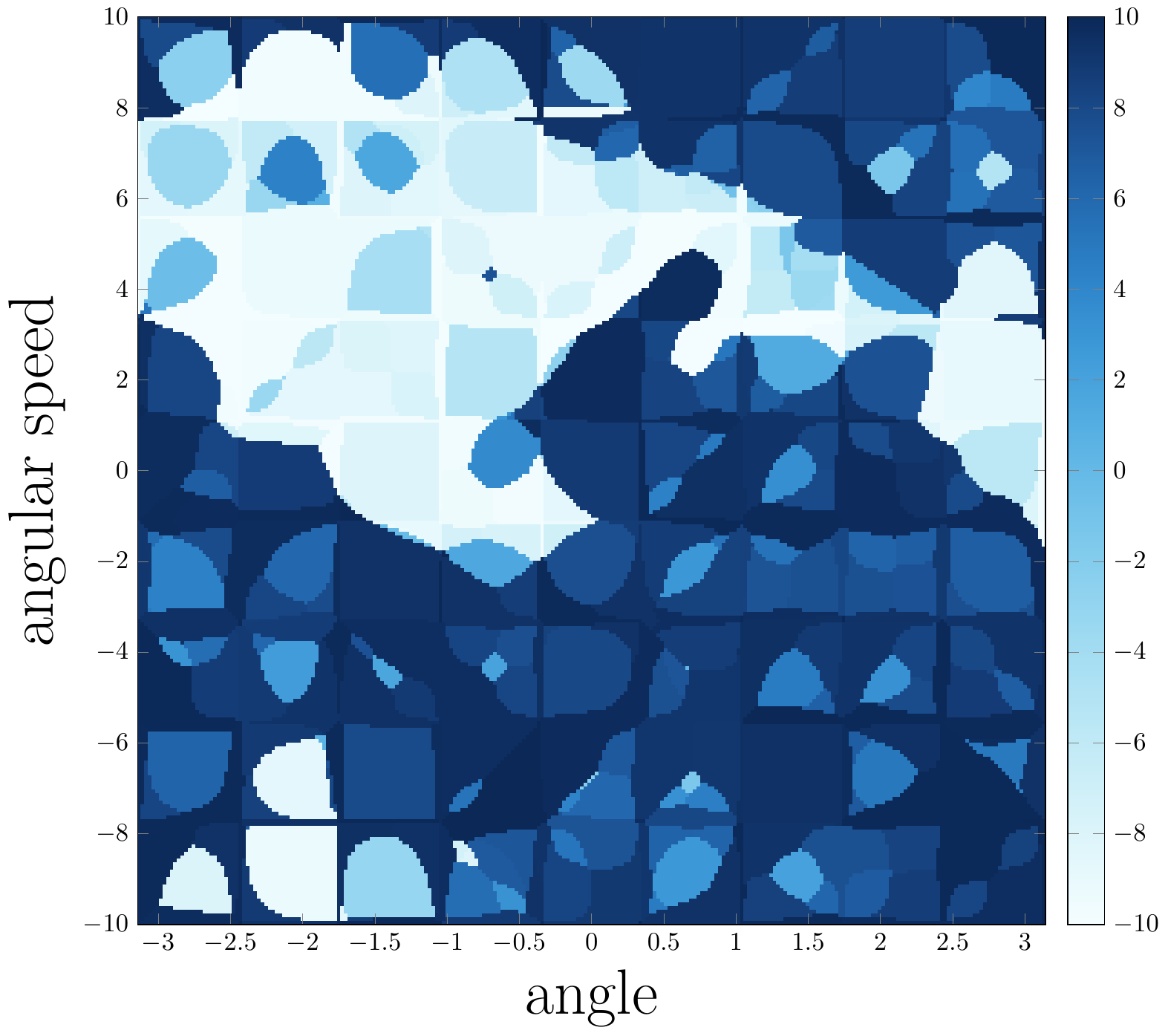}
}
\hspace{-2.3ex}
\subfigure[Nuclear Norm]{
    \label{fig:policy_ac_nucnorm}
    \includegraphics[trim={27 0 0 0},clip,height=0.22\textwidth]{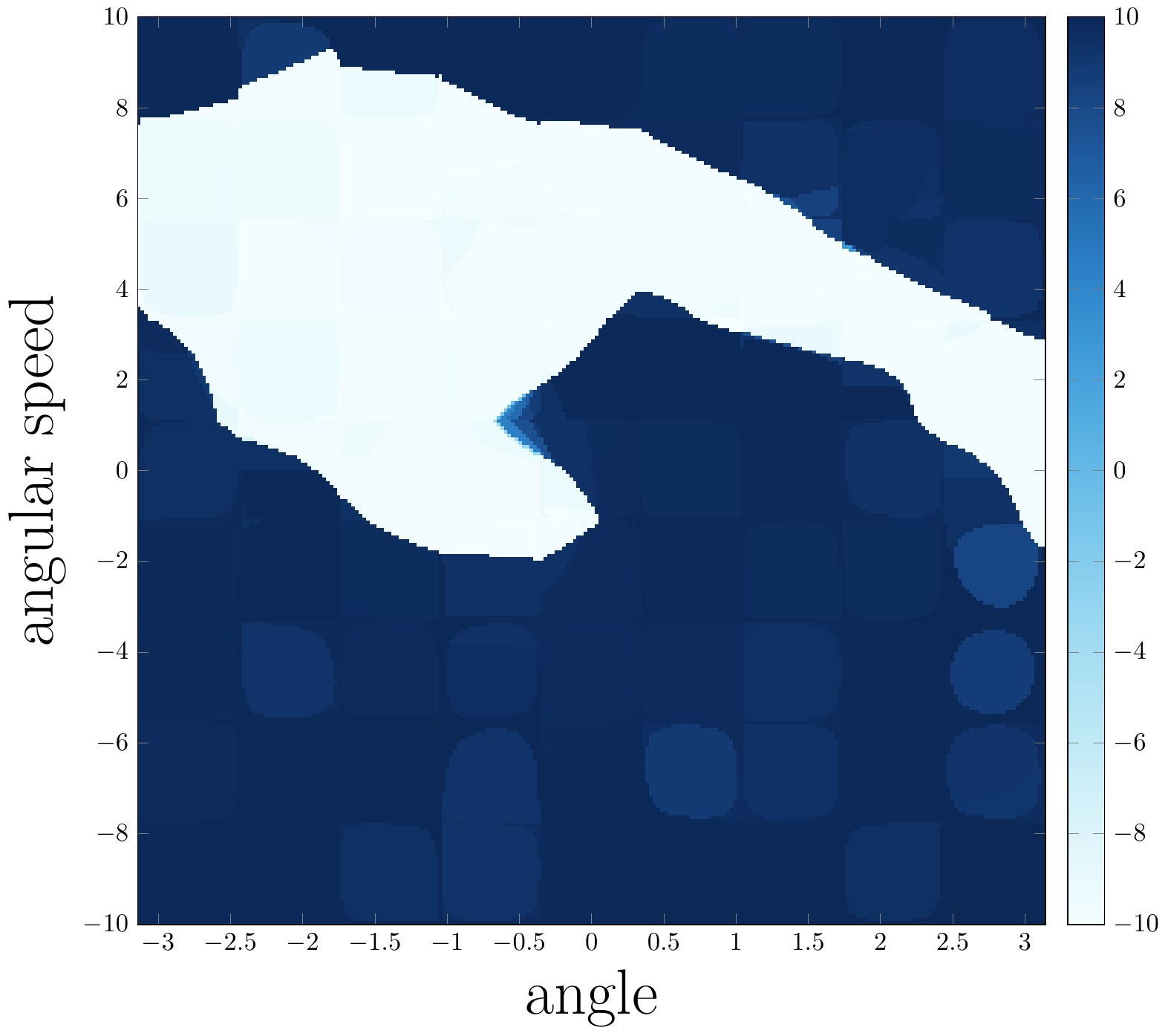}
}
\hspace{-2.3ex}
\subfigure[Ours]{
    \label{fig:policy_ac_ours}
    \includegraphics[trim={27 0 0 0},clip,height=0.22\textwidth]{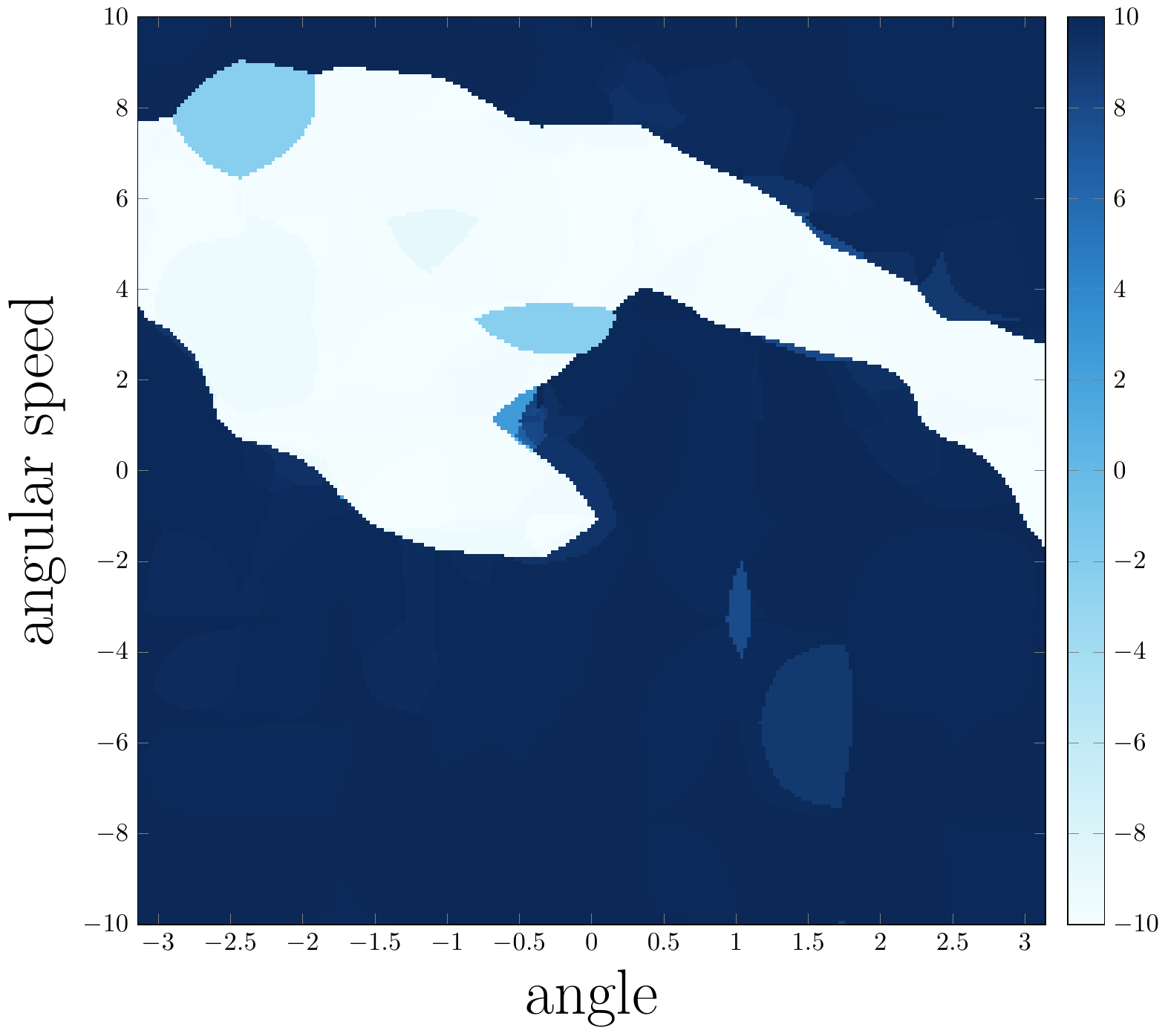}
}
%\vspace{-0.3cm}
\caption{Policy visualization of different methods on the Acrobot control task. The policy is obtained from the output $Q^{(T)}$  by taking $\arg\max_{a\in\mA}Q^{(T)}(s,a)$ at each state $s$. Recall that the state space is 4-dimensional. We hence visualize a 2-dimensional slice in the figure.}
\label{fig:policy-ac}
\vspace{-0.3cm}
\end{figure}

\subsection{Additional Study on the Discounting Factor $\gamma$}
\label{appendix:gamma-ablation}
Throughout the empirical study, we follow the literature~\cite{russ2019,yang2020harnessing} to use a large discounting factor $\gamma$ (i.e., $0.9$) on several real control tasks. We have demonstrated that the proposed low-rank algorithm can perform well on those settings, confirming the efficacy of our method. Just as a final proof of concept for our theoretical guarantees, 
 we provide in this section an ablation study on the $\ell_\infty$  error with smaller value of $\gamma$. We choose $\gamma=0.5$ on the Inverted Pendulum control task.
 note that this affects the reward design and changes the original task. The experiment is only meant to further validate our guarantees.

% \red{Note that since these control tasks can exhibit some internal task structure, which influences the reward design, thus the designed reward function may not be suitable for very small $\gamma$.}
% ----(
% \red{To illustrate this point, maybe can also put the policy visualization here; the visualization is not good even for the optimal one.})

We show the sample complexity as well as the $\ell_\infty$ errors in Fig.~\ref{fig:gamma.5_ip}.
As expected, with a smaller $\gamma$, the convergence is faster. Again, the  overall conclusion is consistent with the previous experiments: 
significant gains on sample complexity are achieved by our efficient algorithm, and the performance of our simple ME method is competitive.
% Overall trend is similar.
% However, with a smaller $\gamma$, the convergence is much faster, as can be observed in Fig.~\ref{fig:gamma.5_ip}, only very few iterations / samples are needed to converge.

\begin{figure}[H]
\vspace{-0.16in}
\centering

% \hfill
\subfigure[Sample Complexity, $\gamma=0.5$]{
    \label{fig:gamma.5_ip_complexity}
    \includegraphics[width=0.4\textwidth]{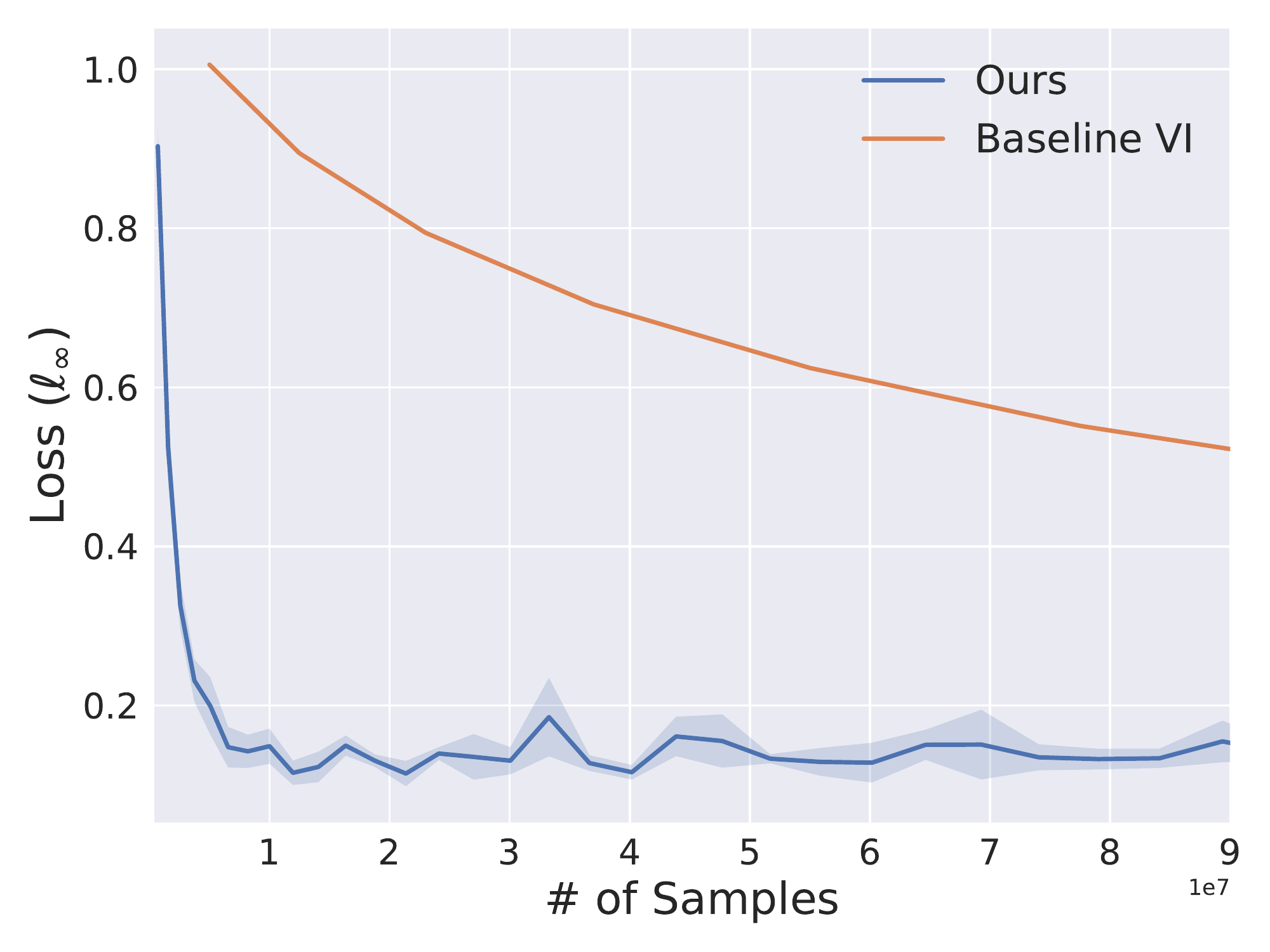}
}
\hspace{3ex}
\subfigure[$\ell_{\infty}$ Errors, $\gamma=0.5$]{
    \label{fig:gamma.5_ip_loss}
    \includegraphics[width=0.4\textwidth]{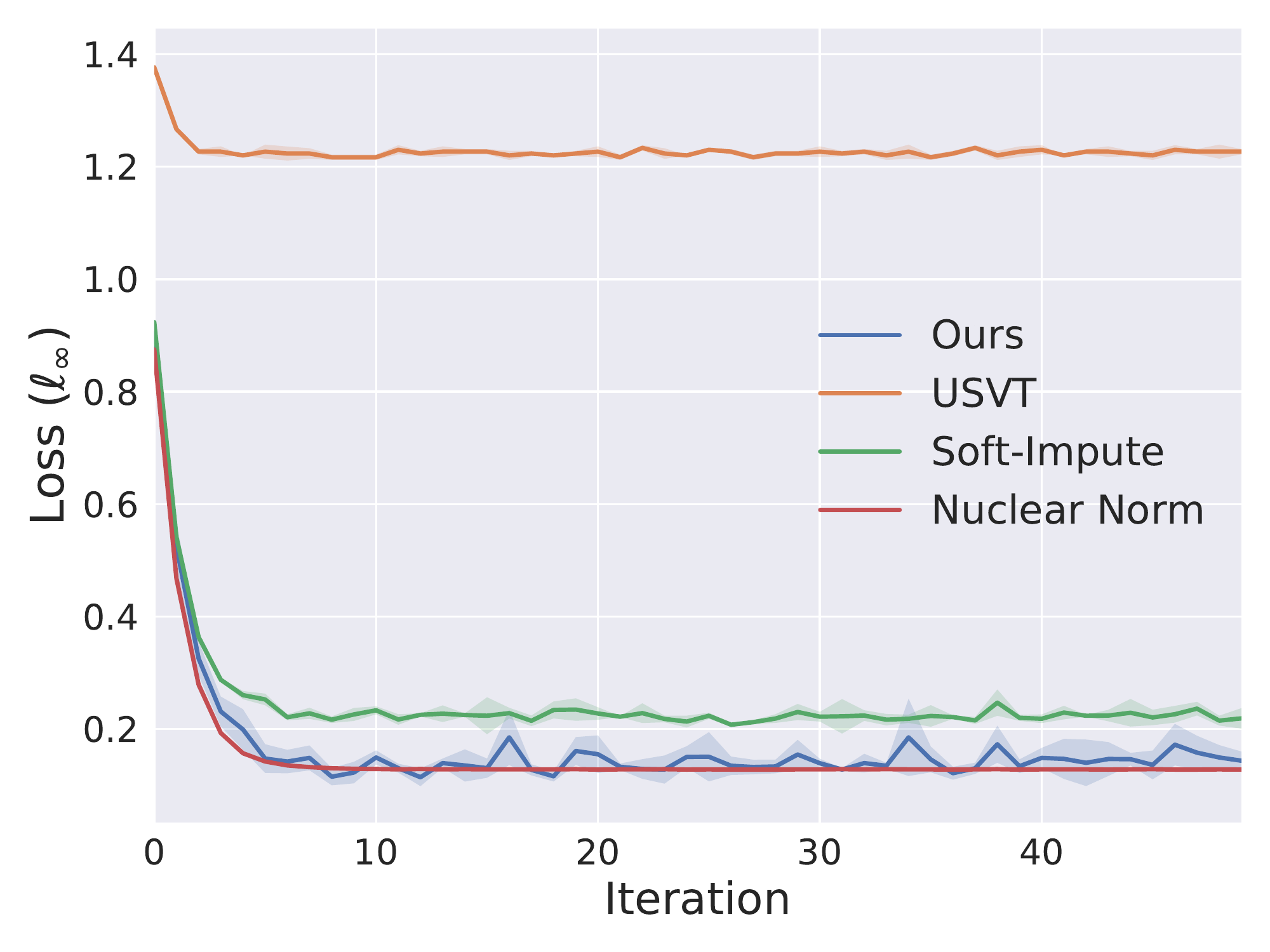}
}
% \hspace{-2.3ex}
% \subfigure[Mean errors]{
%     \label{fig:ac_.9_mean_loss}
%     \includegraphics[width=0.2475\textwidth]{figures/ac_gamma.9_mean.pdf}
% }
% \hspace{-2.3ex}
% \subfigure[Sample complexity]{
%     \label{fig:ac_.9_mean_complexity}
%     \includegraphics[width=0.2475\textwidth]{figures/ac_gamma.9_mean_complexity.pdf}
% }
%\vspace{-0.2cm}
\caption{
Empirical results on the Inverted Pendulum control task, with  $\gamma=0.5$. We show the improved sample complexity in (a) and compare the $\ell_\infty$ error  for various ME methods in (b).}
\label{fig:gamma.5_ip}
\vspace{-0.2cm}
\end{figure}

\end{document}